%% file: FnTmeta.tex
\documentclass[SIG,biber]{nowfnt}   

\usepackage[utf8]{inputenc}
\usepackage{booktabs}
\usepackage{url}
\usepackage{amstext}
\usepackage{amsmath,amssymb}
\usepackage{pifont} 
\usepackage{graphicx}
\usepackage{subcaption}

\usepackage{physics}

\usepackage{mathtools}
\usepackage{wrapfig}
\usepackage{algorithm}
\usepackage{algpseudocode}
\usepackage{makecell}

\makeatletter
\def\blfootnote{\xdef\@thefnmark{}\@footnotetext}
\makeatother

\DeclareMathOperator*{\argmax}{arg\,max}
\DeclareMathOperator*{\argmin}{arg\,min}
\newcommand{\EE}{\mathbb{E}}
\newcommand{\btheta}{\boldsymbol{\theta}}
\newcommand{\bphi}{\boldsymbol{\phi}}

\title{Learning with Limited Samples – Meta-Learning and Applications to Communication Systems}

\maintitleauthorlist{
}

\issuesetup
{%
 copyrightowner={...},
 volume        = xx,
 issue         = xx,
 pubyear       = 2022,
 isbn          = xxx-x-xxxxx-xxx-x,
 eisbn         = xxx-x-xxxxx-xxx-x,
 doi           = 10.1561/XXXXXXXXX,
 firstpage     = 1, 
 lastpage      = \pageref{LastPage}
 }

\usepackage{mwe}

\author[$\star$]{Lisha Chen}
\author[$\dagger$]{Sharu Theresa Jose}
\author[$\dagger$]{Ivana Nikoloska}
\author[$\dagger$]{Sangwoo Park}
\author[$\star$]{Tianyi Chen}
\author[$\dagger$]{Osvaldo Simeone\blfootnote{The first four authors are listed in alphabetical order. Lisha Chen is the main author of Section~\ref{sec:ch2} excluding Section~\ref{sec:ch2:modular}, as well as Sections~\ref{sec:ch3:bilevel-opt}, \ref{sec:ch4:meta_linear_risk}, and~\ref{sec:ch7.2:theory}; Sharu Theresa Jose is the main author of Section~\ref{sec:ch4}; Ivana Nikoloska is the main author of Sections~\ref{sec:ch2:modular},~\ref{sec:ch5.5:power_control} and \ref{sec:ch6.2:quantum_compute}; Sangwoo Park is the main author of Section~\ref{sec:ch5} excluding Section~\ref{sec:ch5.5:power_control}, as well as Sections~\ref{sec:ch7.1:methods} and \ref{sec:ch7.3:apps}; Tianyi Chen is the main author of Section~\ref{sec:ch3:bilevel-opt}; and Osvaldo Simeone is the main author of Section~\ref{sec:ch1} and Section~\ref{sec:ch6.1:neuro_compute}.  This monograph is based on a tutorial delivered by Tianyi Chen and Osvaldo Simeone at IEEE ICASSP 2022. Tianyi Chen and Osvaldo Simeone have supervised the writing process, and Osvaldo Simeone led the  editing of the document.
}
}

\affil[$\dagger$]{King's College London}
\affil[$\star$]{Rensselaer Polytechnic Institute}

\articledatabox{\nowfntstandardcitation}

\begin{document}

\makeabstracttitle

\begin{abstract}
Deep learning has achieved remarkable success in many machine learning tasks such as image classification, speech recognition, and game playing. However, these breakthroughs are often difficult to translate into real-world engineering systems because deep learning models require a massive number of training samples, which are costly to obtain in practice.  To address labeled data scarcity, few-shot meta-learning optimizes  learning algorithms that can efficiently adapt to new tasks quickly. While meta-learning is gaining significant interest in the machine learning literature, its working principles and theoretic fundamentals are not as well understood in the engineering community. 

This review monograph provides an introduction to meta-learning by covering principles, algorithms, theory, and engineering applications. After introducing meta-learning in comparison with conventional and joint learning, we describe the main meta-learning algorithms, as well as a general bilevel optimization framework for the definition of meta-learning techniques. Then, we summarize known results on the generalization capabilities of meta-learning from a statistical learning viewpoint. Applications to communication systems, including decoding and power allocation, are discussed next, followed by an introduction to aspects related to the integration of meta-learning with emerging computing technologies, namely neuromorphic and quantum computing. The monograph is concluded with an overview of open research challenges.

\end{abstract}

\include{Chapter1/Chapter1}

\include{Chapter2/Chapter2_OS}

%
%
\include{Chapter3/Chapter3}
 \include{Chapter4/Chapter4}

\include{Chapter5/Chapter5}

\include{Chapter6/Chapter6}
\include{Chapter7/Chapter7}

\begin{acknowledgements}

The work of Sharu Jose, Ivana Nikoloska, Sangwoo Park, and Osvaldo Simeone was supported by the European Research Council (ERC) under the European Union's
Horizon 2020 Research and Innovation Program (Grant Agreement No. 725731).
The work of Lisha Chen and Tianyi Chen was partially supported by National Science Foundation (NSF) CAREER Award 2047177, NSF MoDL-SCALE Grant 2134168 and the Rensselaer-IBM AI Research Collaboration (\url{http://airc.rpi.edu}), part of the IBM AI Horizons Network. 

\end{acknowledgements}
 
\backmatter  

\printbibliography

\end{document}

%% file: Chapter1/Chapter1.tex
\chapter{Introduction and Background}
\label{sec:ch1}
\section{Introduction}

One of the main principles underlying the design of data-efficient machine learning is \textbf{knowledge sharing} across learning tasks.  As an example, consider the problem of \textbf{few-shot classification}. In it, one is interested in designing a classifier based on few examples for each class. The limited availability of data is typically an insurmountable problem for conventional machine learning solutions, unless one has detailed information about the structure of the problem that can be used to handcraft a well-performing classifier. When such domain knowledge is not available, it may be, however, possible to collect data sets from distinct 
classification tasks that are deemed to be related to the task of interest. Transferring knowledge from such auxiliary tasks to the target task may compensate for the lack of sufficient data or domain knowledge. 
%
%

The specific way in which knowledge sharing can be realized depends on the setting of interest and on the availability of data. Central to these distinctions is the notion of a \textbf{learning task}. A learning task generally refers to a specific supervised, unsupervised, or reinforcement learning instance characterized by an underlying data-generation distribution and loss or reward function. For instance, a learning task may amount to the problem of classifying images in a number of categories based on labelled examples. With this definition, at a high level, we can distinguish the following methodologies (see, e.g., \cite{simeoneCUP}).

\begin{itemize}
	\item \textbf{Transfer learning}: In transfer learning, one is concerned with two learning tasks -- a source task and a target task. Data are typically available for both tasks, although data for the target task may be limited. The goal is to address the target task by utilizing also data from the source task with the aim of reducing data requirements for the target task. In the image classification example, transfer learning would facilitate the optimization of a classifier for a target classification task, e.g., distinguishing images of cats and dogs, using data for another classification task, e.g., distinguishing images of teapots and mugs.
	\item \textbf{Multi-task learning and joint learning}: In multi-task learning, there are $K>1$ learning tasks, and one is interested in learning a machine learning model that is able to address \emph{all} the tasks based on data pooled from all the tasks. Generally, the machine learning model has some shared components, e.g., layers of a neural network, and also separate parts pertaining each task, e.g., ``heads'' of a classifier. When the model is fully shared across tasks, multi-task learning is also known as \textbf{joint learning}.  In the image classification example, multi-task learning would optimize a classifier producing decisions for a set of classification tasks.
	\item \textbf{Meta-learning}: In meta-learning, we have access to data for a number of tasks, but we are not interested in training a machine learning model for them as in multi-task learning. Rather, we would like to use data from multiple tasks in order to design a \textbf{training procedure}, and not to produce a single machine learning model. Specifically, the goal is ensure that the \emph{meta-learned} training procedure can  efficiently optimize a machine learning model for \emph{any}, a priori unknown, learning task. Accordingly, in a meta-learning setting, one does not know a priori what the target task will be, although one expects it to be similar to those for which data are available. By optimizing the learning process, meta-learning implements a form of \textbf{learning to learn}.  In the image classification example, meta-learning would produce a procedure able to optimize a classifier for \textit{any new classification task} by using data from a pool of other similar classification tasks.
\end{itemize}

This review monograph provides an introduction to meta-learning by covering principles, algorithms, theory, and engineering applications. In this section, we start by providing a first exposition to meta-learning by contrasting it with conventional machine learning and multi-task learning. The chapter concludes with a description of the organization of the rest of the monograph.

\section{Meta-Learning}

In meta-learning, we target an entire \textbf{class of tasks}, also known as the \textbf{task environment}, and we wish to ``prepare'' for any new task that may be encountered from this class. As we will review in this subsection, conventional learning aims at optimizing model parameters, such as the weights of a neural network, by applying a given training algorithm, which is defined by a set of \textbf{hyperparameters}. Training algorithms typically involve local search procedures, e.g., based on gradient information, and hyperparameters include  the learning rate -- i.e., the size of the updates at each iteration -- and the initialization. In contrast,  the goal of meta-learning is to 
optimize \textbf{hyperparameters} with the goal of identifying a training algorithm that may perform well on new tasks. 


\subsection{Meta-Training and Meta-Testing}

The working assumption underlying meta-learning is that, prior to
observing the -- typically small -- training data set for a new
task, one has access to a larger data set of examples
from related tasks. This is known as the \textbf{meta-training data set}. Meta-learning consists of two distinct phases: \vspace{-0.2cm} \begin{itemize}  \item \textbf{Meta-training}: Given the meta-training data set, a set of hyperparameters is optimized;
	\item \textbf{Meta-testing}: After the meta-learning phase is completed, data for
	a target task, known as \textbf{meta-test task}, is revealed, and model parameters
	are optimized using the meta-trained hyperparameters.\end{itemize}\vspace{-0.2cm}
As such, the meta-training phase aims at optimizing hyperparameters
that enable efficient training on a new, a priori unknown, target
task in the meta-testing phase.

\subsection{Reviewing Conventional Learning}

In order to introduce the notation necessary to describe
meta-learning, let us briefly review the operation of conventional
machine learning.

\noindent \textbf{Training and testing.}	 In conventional machine learning, the starting point is the selection of  a model class $\mathcal{H}$ and of a training algorithm. The choice of model class and training algorithm determines the \textbf{inductive bias} applied by the learning procedure to generalize from training to test data.  The model class $\mathcal{H}$ contains models parameterized
by a vector $\phi$, such as neural networks. Model class and training algorithm are ideally tailored to information available about the problem of interest. 

Furthermore, both model class and training algorithm generally depend on a \emph{fixed} vector of
hyperparameters, denoted as $\theta$. Thereafter, hyperparameters may specify, for instance, a mapping defining the vector of features
to be used in a linear model, or the initialization and learning rate
of an iterative optimizer.

The training algorithm is applied to a training set $\mathcal{D}^{\text{tr}}$,
which may include also a separate validation set. The training algorithm
produces a model parameter vector $\phi$ by minimizing the \textbf{training loss} 
\begin{equation}\label{eq:ch1:trloss}
	L_{\mathcal{D}^{\text{tr}}}(\phi),
\end{equation}  which  is obtained by evaluating an empirical average of the loss accrued over the data points in the training set $\mathcal{D}^{\text{tr}}$. Note that regularized versions of the training loss can also be used. Finally, the trained model is tested on a separate test
data set $\mathcal{D}^{\text{va}}$ by evaluating the \textbf{validation loss} 	$L_{\mathcal{D}^{\text{va}}}(\phi)$, in which the loss is averaged over the test data in data set $\mathcal{D}^{\text{va}}$. 
The overall process is summarized in Fig. \ref{fig1ch1}.

\begin{figure}
	\begin{center}
		\includegraphics[width=0.4\textwidth]{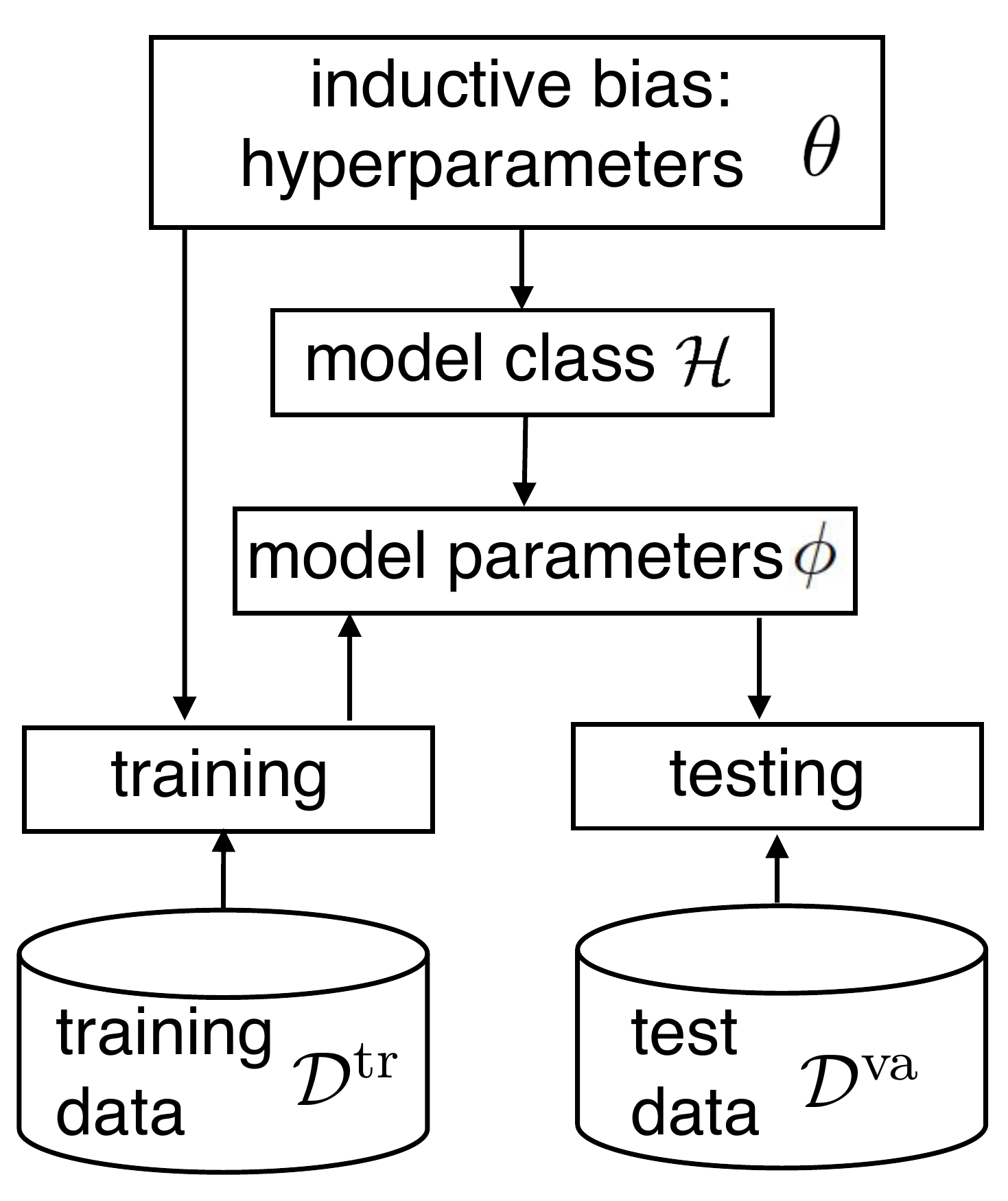}
	\end{center}
	\caption{Illustration of conventional machine learning.  \label{fig1ch1} }
\end{figure}

%
%

\noindent \textbf{Drawbacks of conventional learning.} As anticipated, conventional machine learning suffers from two main potential shortcomings
that meta-learning can help address, namely: \begin{itemize}
	\item  Large \textbf{sample complexity}: By training a model ``from scratch'', conventional learning generally requires a large number
	of training samples, $N$, to obtain a suitable
	test performance. The number of samples needed to obtain some level of accuracy is known as sample complexity.
\item Large \textbf{iteration complexity}: By relying on a generic optimization procedure, conventional learning may  require a large number of iterations to converge to a well-performing model.

\end{itemize}

Both issues can be potentially mitigated if the inductive
bias -- i.e., the selection of model class and training algorithm -- is tailored to the problem under study  based on domain knowledge. 
For instance, as part of the inductive bias, we may choose an architecture for a neural network model that satisfies known symmetries in the data; or  select an initialization point for the model parameters that
$\phi$ is suitably adapted to the learning task at hand. With such informed inductive biases, one we can generally reduce both sample and iteration complexities.

When one does not have access to sufficient information about the problem to identify a tailored inductive bias, it may become useful to transfer knowledge from data pertaining related tasks.

\subsection{Joint Learning}

Suppose that we have access to training data
sets $\mathcal{D}_{k}^{\text{tr}}$ for a number of distinct learning tasks in the same task environment that are indexed by the integer $k=1,...,K.$ Each data set 
$\mathcal{D}_{k}^{\text{tr}}$ contains $N$ training examples. We now review the idea of joint learning, which is a special case of multi-task learning in which a common model is trained for all $K$ learning tasks.

\noindent \textbf{Training and testing.}	 \textbf{Joint learning} pools together
all the training sets $\{\mathcal{D}_{k}^{\text{tr}}\}_{k=1}^{K}$,
and uses the resulting aggregate training loss 
\begin{equation}\label{eq:ch14:joint2}
L_{\{\mathcal{D}_{k}^{\text{tr}}\}_{k=1}^{K}}(\phi)=\frac{1}{K}\sum_{k=1}^{K}L_{\mathcal{D}_{k}^{\text{tr}}}(\phi)
\end{equation} 
as the learning criterion to train a shared model parameter $\phi$.

As illustrated in Fig. \ref{fig2ch1}, joint learning inherently caters only to the $K$ tasks in the original pool, and is hence generally unable to provide desirable performance for new, as of yet unknown, tasks.

\begin{figure}
\begin{center}
	\includegraphics[width=0.4\textwidth]{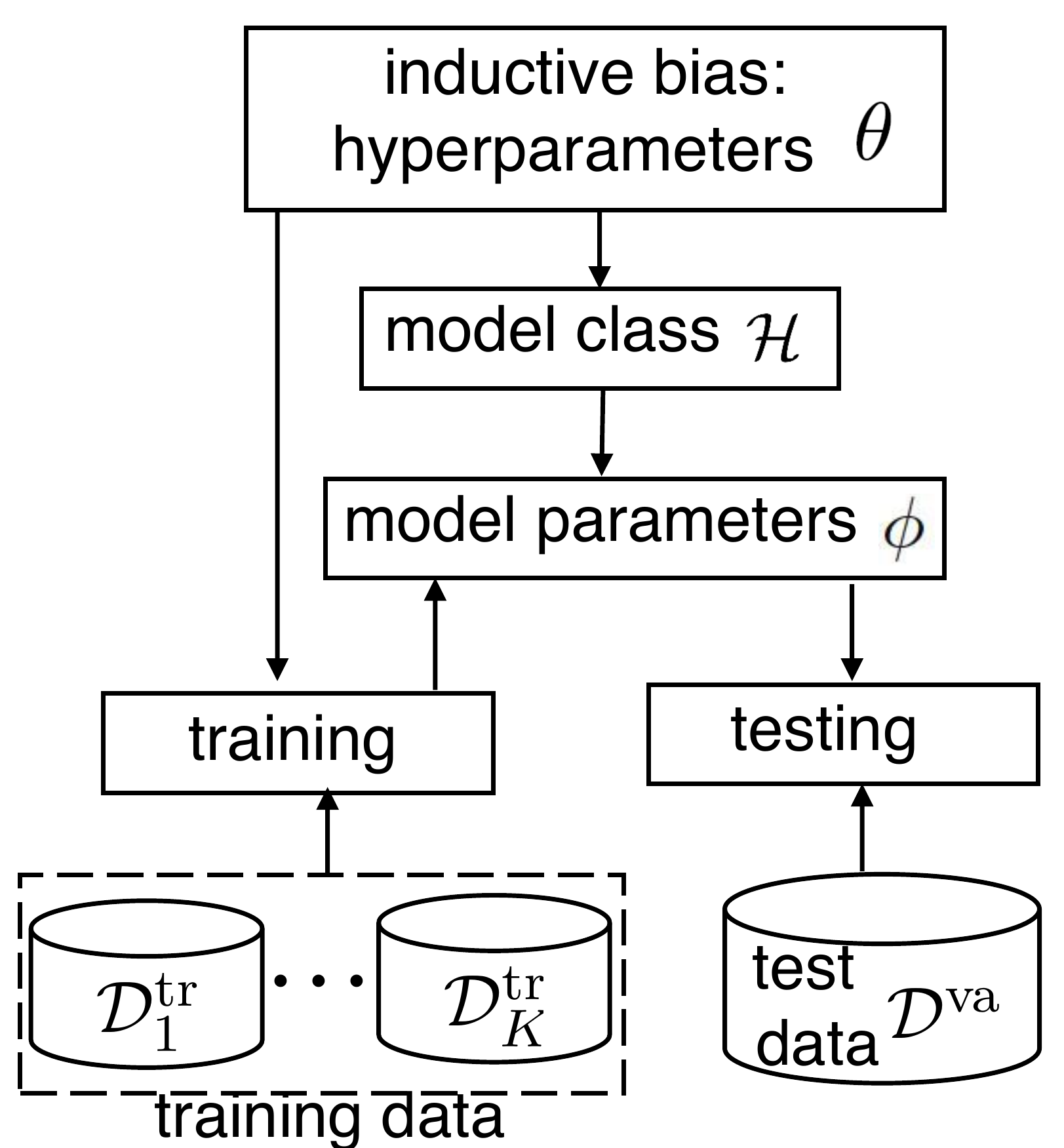}
\end{center}
\caption{Illustration of joint learning. \label{fig2ch1} }
\end{figure}

Joint learning is a natural first attempt to transfer knowledge across tasks with the aim of improving sample and iteration complexities.
First, by pooling together data from $K$ tasks, the overall size
of the training set is $K\cdot N,$ which may be large even when the
available data per task is limited, i.e., when $N$ is small.
Second, training only once for $K$ tasks amortizes the iteration
complexity across the tasks, yielding a potential reduction of the
number of iterations by a factor equal to $K$.

\noindent \textbf{Drawbacks of joint learning.}  Joint learning has two potentially critical shortcomings. \begin{itemize}
	\item \textbf{Bias}: The jointly trained model may improve the performance of conventional learning only if there is a single model
	parameter $\phi$ that ``works well'' for all tasks. This may not be the case if the tasks are sufficiently distinct.
	\item \textbf{Lack of adaptation}: Even if there is a single model parameter $\phi$ that yields
	desirable test results on all $K$ tasks, this does not guarantee
	that the same is true for a new task. In fact, by focusing on training a common model for all tasks, joint learning is not designed
	to enable adaptation to a new task.
\end{itemize}

 As a remedy for the second shortcoming just highlighted, one could use the jointly trained model parameter $\phi$ to initialize
the training process on a new task -- a process known as \textbf{fine-tuning}.
However, there is generally no guarantee that this would yield a desirable outcome, since the training process used by joint learning does not account for the subsequent step 
of adaptation on a new task. This is a key distinction between joint learning and meta-learning, which will be introduced next.

\subsection{Introducing Meta-Learning}\label{sec:ch1:intro}

As for joint learning, in meta-learning one assumes the availability
of data from $K$ related tasks from the same task environment, which are referred to as \textbf{meta-training
tasks}.
However, unlike joint learning, data from these tasks are kept separate,
and a distinct model parameter $\phi_{k}$ is trained for each $k$
task. 
As illustrated in Fig. \ref{fig3ch1}, meta-learning tasks only share a  \textbf{common hyperparameter vector}
$\theta$ that is optimized based on meta-training data. As a result, meta-training data is not used to optimize a common model, but only a  \textbf{shared inductive bias}. In other words, the optimization carried out by meta-learning operates at a higher level of abstraction, leaving the model parameters free to adapt to each individual task.

We now introduce meta-learning by emphasizing the differences with respect to joint learning and by detailing the meta-training and meta-testing phases.

\begin{figure}
\begin{center}
	\includegraphics[width=0.8\textwidth]{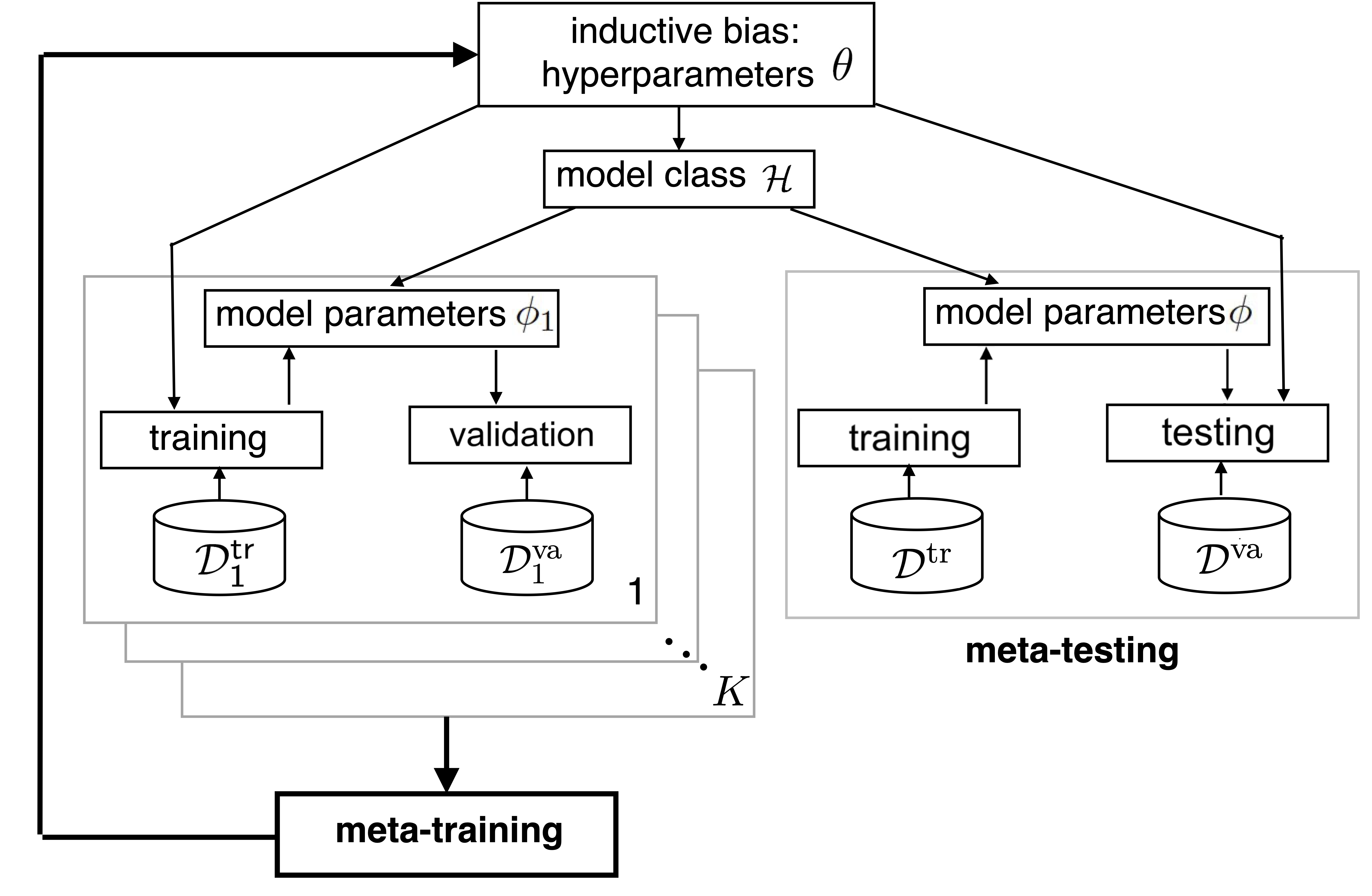}
\end{center}
\caption{Illustration of meta-learning. \label{fig3ch1} }
\end{figure}

\noindent \textbf{Inductive bias and hyperparameters.} As discussed, the goal of meta-learning is optimizing the hyperparameter vector $\theta$ and, through it, the inductive bias that is applied for the training of each task. To simplify the discussion and focus on the most common setting, let us assume that the model class $\mathcal{H}$ is fixed, while the training algorithm is a mapping $\phi^{\textrm{tr}}(\mathcal{D}|\theta)$ between
a training set $\mathcal{D}$ and a model parameter vector $\phi$
that depends on the hyperparameter vector $\theta$, i.e.,
\begin{equation}\label{eq:ch14:trainalgo}
\phi=\phi^{\textrm{tr}}(\mathcal{D}|\theta).
\end{equation}As an example, the training algorithm $\phi^{\textrm{tr}}(\mathcal{D}|\theta)$ could
output the last iterate of an optimizer. 

The hyperparameter $\theta$ can affect the output $\phi^{\textrm{tr}}(\mathcal{D}|\theta)$ of the training procedure in different ways. For instance, it can determine the regularization constant; the learning rate and/or the initialization of an iterative training procedure; the mini-batch size; a subset of the parameters in vector $\phi$, e.g., used to define a shared feature extractor;  the parameters of a prior distribution; and so on.


The output $\phi^{\textrm{tr}}(\mathcal{D}|\theta)$
of a training algorithm is generally random. This is the case, for
instance, if the algorithm relies on stochastic gradient descent (SGD). In the following discussion, we will assume for simplicity a deterministic training algorithm, but the approach carries over directly to the more general case of a random training procedure
by adding an average over the randomized of the trained model $\phi^{\textrm{tr}}(\mathcal{D}|\theta)$.

\noindent \textbf{Meta-training.}	 To formulate meta-training, a natural idea is to use as the optimization criterion the aggregate
training loss 
\begin{equation}\label{eq:ch14:metatrloss}
\mathcal{L}_{\{\mathcal{D}_{k}^{\text{tr}}\}_{k=1}^{K}}(\theta)=\frac{1}{K}\sum_{k=1}^{K}L_{\mathcal{D}_{k}^{\text{tr}}}(\phi^{\textrm{tr}}(\mathcal{D}_{k}^{\text{tr}}|\theta)),
\end{equation}
which is a function of the hyperparameter $\theta$. This quantity is known as the \textbf{meta-training loss}. The resulting problem \begin{equation}\label{eq:ch14:minmtloss}\min_{\theta}\mathcal{L}_{\{\mathcal{D}_{k}^{\text{tr}}\}_{k=1}^{K}}(\theta)\end{equation} of minimizing the meta-training loss over the hyperparameter $\theta$
is different from the ERM problem $\min_{\phi}L_{\{\mathcal{D}_{k}^{\text{tr}}\}_{k=1}^{K}}(\phi)$
tackled in joint learning for the following reasons:\vspace{-0.2cm}\begin{itemize} \item First, optimization is over the \textbf{hyperparameter} vector $\theta$ and not
over a shared model parameter $\phi$. \item Second, the model parameter $\phi$ is trained \textbf{separately} for each
task $k$ through the parallel applications of the training function
$\phi^{\textrm{tr}}(\cdot|\theta)$ to the training set $\mathcal{D}_{k}^{\text{tr}}$
of each task $k=1,...,K$.\end{itemize}\vspace{-0.2cm}

As a result of these two key differences with respect to joint training, the minimization of the meta-training loss (\ref{eq:ch14:metatrloss}) inherently caters for \textbf{adaptation}:
The hyperparameter vector $\theta$ is optimized in such a way that the
trained model parameter vectors $\phi_{k}=\phi^{\textrm{tr}}(\mathcal{D}_{k}^{\text{tr}}|\theta)$,
adapted separately to the data of each task $k$, minimize the aggregate
loss across all meta-training tasks $k=1,...,K$.

\noindent \textbf{Advantages of meta-training over joint training.}  While retaining the advantages of joint learning in terms of sample
and iteration complexity, meta-learning  addresses the two shortcomings
of joint learning: \vspace{-0.2cm}\begin{itemize}\item \textbf{Knowledge sharing via hyperparameters}: Meta-learning does not assume that there is a single model
parameter $\phi$ that ``works well'' for all tasks. It only assumes
that there exists a common model class and a common training algorithm,
as specified by \textbf{hyperparameters} $\theta$, that can be effectively applied
across the class of tasks of interest. 
\item  \textbf{Optimization for adaptation}: Meta-learning prepares the training algorithm $\phi^{\textrm{tr}}(\mathcal{D}|\theta)$
to \textbf{adapt} to potentially new tasks through the selection of the hyperparameters
$\theta$. This is because the model parameter vector $\phi$ is left free by design 
to be adapted to the training data $\mathcal{D}^{\textrm{tr}}_{k}$ of each task $k$.\end{itemize}\vspace{-0.2cm}

\noindent \textbf{Meta-testing.} 
As mentioned, the goal of meta-learning is
ensuring generalization to any new task that is drawn at random from the same task environment. 	For any new task, during the meta-testing phase, we have access to training set $\mathcal{D}^{\text{tr}}$
and validation set $\mathcal{D}^{\text{va}}.$ The new task is referred to as the \textbf{meta-test task}, and is illustrated in Fig. \ref{fig3ch1} along with the meta-training tasks.

The training data $\mathcal{D}^{\text{tr}}$ of the meta-test task is used to adapt the model parameter vector to the meta-test task, obtaining $\phi^{\textrm{tr}}(\mathcal{D}^{\text{tr}}|\theta)$. Importantly, the training algorithm depends on the hyperparameter $\theta$. 
The performance metric of interest for a given hyperparameter $\theta$
is the test loss for the meta-test task, or \textbf{meta-test loss}, given by \begin{equation}\label{eq:ch14:metatloss}L_{\mathcal{D}^{\text{va}}}(\phi^{\textrm{tr}}(\mathcal{D}^{\text{tr}}|\theta)).\end{equation} In (\ref{eq:ch14:metatloss}), the population loss of the trained model is estimated via the test loss evaluated with the test set  $\mathcal{D}^{\text{va}}.$

We have just seen that  meta-testing requires a split of the data for the new task into a training part, used for adaptation, and a validation part, used to estimate the population loss (\ref{eq:ch14:metatloss}). We now discuss how the idea of splitting per-task data sets into training and validation parts can be useful also during the meta-training phase.

As explained in Section \ref{sec:ch1:intro}, the training algorithm $\phi(\mathcal{D}^{\text{tr}}|\theta)$ 
is defined by an optimization procedure for the problem
of minimizing the training loss on the training set $\mathcal{D}^{\text{tr}}$. We can write the learning procedure informally as \begin{equation}\label{eq:ch14:dum}
\phi^{\textrm{tr}}(\mathcal{D}^{\text{tr}}|\theta)\underset{\theta}{\leftarrow}\min_{\phi}L_{\mathcal{D}^{\text{tr}}}(\phi),
\end{equation}  highlighting the dependence of the training algorithm on the training loss $L_{\mathcal{D}^{\text{tr}}}(\phi)$ and on the hyperparameter $\theta$.

 Because of (\ref{eq:ch14:dum}), in problem (\ref{eq:ch14:minmtloss})
one is effectively optimizing the training losses $L_{\mathcal{D}_{k}^{\text{tr}}}(\phi)$ for the meta-training tasks $k=1,...,K$ twice, first over
the model parameters in the inner optimization (\ref{eq:ch14:dum}) and then over the hyperparameters $\theta$ in the outer optimization (\ref{eq:ch14:minmtloss}). 
This reuse of the meta-training data for both adaptation and meta-learning may cause overfitting to the meta-training
data, and thus result in a training algorithm $\phi^{\textrm{tr}}(\cdot|\theta)$
that fails to generalize to new tasks. 

 The problem highlighted above is caused by the fact that the meta-training loss (\ref{eq:ch14:metatrloss}) does not provide an unbiased estimate of the sum of the population losses across the meta-training tasks. The bias is a consequence of the reuse of the same data for both adaptation and hyperparameter optimization. To address this problem, for each meta-training task $k$, we can partition the
available data into two data sets, a training data set $\mathcal{D}_{k}^{\text{tr}}$ and a validation data set $\mathcal{D}_{k}^{\text{va}}$.
Therefore, the overall meta-training data set is given as $\mathcal{D}^{\textrm{mtr}}=\{(\mathcal{D}_{k}^{\text{tr}},\mathcal{D}_{k}^{\text{va}})_{k=1}^{K}\}$.

The key idea is that the training data set $\mathcal{D}_{k}^{\text{tr}}$ is used for adaptation using the training algorithm (\ref{eq:ch14:dum}), while the test data set $\mathcal{D}_{k}^{\text{va}}$ is kept aside to estimate the population distribution of task $k$ for the trained model. The hyperparameter $\theta$ is not optimized to minimize the sum of the training losses as in (\ref{eq:ch14:minmtloss}). Rather, they target the sum of the test losses, which provides an unbiased estimate of the corresponding sum of population losses. 

\noindent \textbf{Meta-learning as nested optimization.}	To summarize, the general procedure followed by many meta-learning algorithms consists of a nested optimization of the following form:\vspace{-0.2cm}\begin{itemize}
\item \textbf{Inner loop}: For a fixed hyperparameter vector $\theta$, training
on each task $k$ is done separately, producing per-task model parameters
\begin{equation}\label{eq:ch14:inner}\phi_{k}=\phi^{\textrm{tr}}(\mathcal{D}_{k}^{\text{tr}}|\theta)\underset{\theta}{\leftarrow}\min_{\phi}L_{\mathcal{D}_{k}^{\text{tr}}}(\phi)\end{equation}
for $k=1,...,K;$
\item \textbf{Outer loop}: The hyperparameter vector $\theta$ is optimized as \begin{equation}\label{eq:ch14:outer}\theta_{\mathcal{D}^{\textrm{mtr}}}=\arg\min_{\theta}\mathcal{L}_{\mathcal{D}^{\textrm{mtr}}}(\theta),\end{equation}
where the  \textbf{meta-training loss} is (re-)defined as
\begin{equation}
	\mathcal{L}_{\mathcal{D}^{\textrm{mtr}}}(\theta)=\frac{1}{K}\sum_{k=1}^{K}L_{\mathcal{D}_{k}^{\text{va}}}(\phi^{\textrm{tr}}(\mathcal{D}_{k}^{\text{tr}}|\theta)).
\end{equation}\end{itemize}\vspace{-0.2cm} 	
As we will detail in Section 2, the specific implementation of a meta-learning algorithm depends on the selection of the training
algorithm $\phi^{\textrm{tr}}(\mathcal{D}|\theta)$ and on the method used to solve the outer optimization.

\subsection{Meta-Inductive Bias}

While the inductive bias underlying the training algorithm used in the inner loop is
optimized by means of meta-learning, the meta-learning process itself
assumes a \textbf{meta-inductive bias}. The meta-inductive bias encompasses
the choices of the hyperparameters to optimize in the outer loop -- e.g., the initialization
of an SGD training algorithm
-- as well as the optimization algorithm used in the outer loop. 
There is of course no end to this nesting of inductive biases: any
new learning level brings its own assumptions and biases. Meta-learning
moves the potential cause of bias at the outer level of the meta-learning
loop, which may improve the efficiency of training.

It is important, however, to note that the selection of a meta-inductive
bias may cause \textbf{meta-overfitting}\index{meta-overfitting} in a similar way as the choice of an inductive bias can cause overfitting in conventional
learning. In a nutshell, if the meta-inductive bias is too broad and the number of tasks insufficient, the meta-trained inductive bias may overfit the meta-training data and fail to prepare for adaptation to new tasks.

\section{Organization of the Monograph}

The rest of the monograph is organized as follows. 

\textbf{Section 2. Meta-learning algorithms}: This section provides a taxonomy and an introduction to the most common meta-learning algorithms, including model agnostic meta-learning (MAML).

\textbf{Section 3. Bilevel optimization for meta learning}: Section 3 presents a general optimization-based perspective on meta-learning, which views meta-learning as a form of stochastic bilevel optimization.

\textbf{Section 4. Statistical learning theory for meta-learning}: This section revisits meta-learning through the different perspective of generalization. Specifically, it investigates from a theoretical viewpoint the performance of meta-learning algorithms in terms of their capacity to generalize outside the meta-training data set to new tasks.

\textbf{Section 5. Meta-learning applications to communications}: The section turns to several examples of applications of meta-learning to the engineering problem of designing communication systems. Examples of reviewed applications include demodulation and power control.

\textbf{Section 6. Integration with emerging computing technologies}: This section highlights the potential synergies between meta-learning and two emerging computing technologies, namely neuromorphic and quantum computing. 

\textbf{Section 7. Outlook}: The last section presents an outlook on the area of meta-learning by offering a brief review of open problems and further directions for reading and research.

%% file: Chapter2/Chapter2_OS.tex
\chapter{Meta-Learning Algorithms}
\label{sec:ch2}

 In this section, we review the main classes of meta-learning  algorithms by focusing on selected notable representatives from each class.

\section{Overview of Meta-Learning Algorithms} 
\label{sec:intro_meta_algorithms}
Existing meta-learning algorithms can be roughly grouped into three categories according to the principle underlying the transfer of information among tasks \cite{hospedalesmeta}. We specifically distinguish among: 
(i) \textbf{metric-based} methods, in which  information shared across tasks is encoded in a distance measure used to instantiate  non-parametric predictors;  (ii) \textbf{model-based} methods, whereby data from multiple tasks is used to determine a ``hyper-model'' that maps data from a new task to a model; 
and (iii) \textbf{optimization-based} methods, which target the design of the hyperparameters  of  an optimization procedure for training on new tasks.
We now briefly review each class in turn.


\subsection{Metric-Based Meta-Learning}
\textbf{Metric-based} methods assume that the training and testing tasks in the given environment share a common feature representation mapping that can be used to gauge the similarity between data points. A similarity metric meta-learned based on data from multiple tasks can be leveraged to implement \textbf{non-parametric} predictive models without the need for training on a new task. 
Modern metric-based meta-learning methods include the Matching Network~\cite{vinyals2016matching}, the Prototypical Network~\cite{snell2017prototypical}, and the Relation Network~\cite{sung2018_relationnet}. 
The approach is aligned with empirical Bayes methods that are routinely used in models such as Gaussian Processes, with the caveat that  data is collected here from distinct tasks. 
In this monograph, we will concentrate on parametric models, which have been more commonly adopted for engineering problems, and hence we will not elaborate further on metric-based meta-learning.

\subsection{Optimization-Based Meta-Learning}

Owing to their performance and relative ease of implementation, \\
\textbf{optimization-based} methods constitute the dominant class of meta-learning solutions for \emph{parametric} models. Recently, the most common approach within this class optimizes the \emph{initialization} of the model parameters used by the training procedure. The rationale underlying such optimization-based methods is that a good initialization  can help the training procedure quickly adapt the model parameters to new tasks with  few optimization steps. Notable examples of initialization-based schemes are model agnostic meta-learning (MAML) algorithm and its variants (see e.g., \cite{Finn2017_maml,rajeswaran2019_imaml}). 
More broadly, optimization-based methods may design other hyperparameters of the
training algorithm such as the learning rate~\cite{maclaurin15}.

Existing optimization-based methods that address model initialization can be further divided into two main categories, depending on the type of optimization used for training, namely second-order algorithms and first-order algorithms. 
Second-order algorithms, to be presented in {Section~\ref{sec:second_order_algorithms}}, require second-order derivatives of the per-task loss functions during meta-learning; while first-order algorithms, described in {Section~\ref{sec:first_order_algorithms}}, only need first-order gradient information of the per-task loss functions to be available.

As a distinct example of optimization-based methods,
we will also study \textbf{modular meta-learning}. 
Modular meta-learning relies on the assumption that suitable models for the given environment share a common repository of modules that can be recombined to address each individual task.
Accordingly, modular meta-learning optimizes  the hyperparameters as a set of modules that can be assembled in different ways to yield models for new tasks using combinatorial optimization. 
Modules may consist of instance of layers of a neural network.
We refer to Section~\ref{sec:ch2:modular} for details.

\subsection{Model-Based Meta-Learning}
\textbf{Model-based} methods optimize a hyper-model that directly maps the training set from a task to a model. This mapping can be realized using recurrent neural networks~\cite{schmidhuber1993_recurrent,hochreiter2001_l2l},  convolutional neural networks~\cite{mishra2018simple}, or hypernetworks~\cite{qiao2018few,gidaris2018dynamic}. In Section~\ref{sec:ch2:cavia}, we will elaborate on a simple representative of model-based meta-learning, whereby the training set for the new task is used to optimize a \textbf{context} vector that determines the operation of a model shared across tasks.

\section{Second-Order Optimization-Based Meta-Learning} 
\label{sec:second_order_algorithms}

In this subsection, we introduce second-order optimization-based meta-learning methods by covering the key representatives, MAML~\cite{Finn2017_maml}, implicit MAML (iMAML)~\cite{rajeswaran2019_imaml}, and Bayesian MAML~\cite{grant2018recasting,yoon2018_BMAML,nguyen2020_VAMPIRE}.

\subsection{MAML}
\begin{figure}[ht]
  \centering
  \includegraphics[width=.25\linewidth]{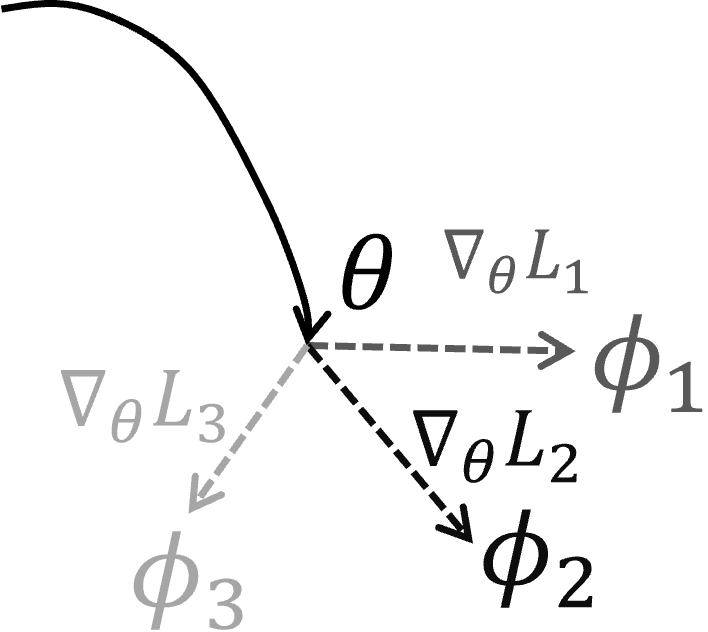}
  \caption{Illustration of MAML: MAML aims at finding an initial parameter vector $\theta$ that allows quick adaptation to new tasks via gradient descent of loss function for the $k$-th task, $L_k = L_{\mathcal{D}_k^{\rm tr}}(\theta)$. The adapted parameter for task $k$ is denoted as ${\phi}_k = \phi(\mathcal{D}_k^{\rm tr} | \theta)$, and is obtained as shown in the figure via a single gradient step.}
  \label{fig:maml}
\end{figure}

As illustrated in~Figure~\ref{fig:maml},
MAML aims at finding an initial parameter vector $\theta$ that allows quick adaptation to new tasks via gradient descent~\cite{Finn2017_maml}. 
In the simplest form of MAML, as seen in Figure~\ref{fig:maml},
starting from the initial parameter vector $\theta$,
the per-task parameter $\phi$ is  adapted using a one-step gradient update for the task-specific loss function $L_{\mathcal{D}_k^{\rm tr}}(\phi )$
for each $k$-th task. 
We recall from~\eqref{eq:ch1:trloss} that
we write as $L_{\mathcal{D}_{k}^{\rm tr}}(\phi )$ the empirical loss evaluated on a training set $\mathcal{D}_{k}^{\rm tr}$ when model parameter $\phi$ is used. Data for the $k$ task comprises the train set $\mathcal{D}_k^{\rm tr}$, which is used for training, as well as the validation set $\mathcal{D}_k^{\rm va}$
that is used to estimate the population loss via the validation loss $L_{\mathcal{D}_k^{\rm va}}(\phi)$.
Let $\mathcal{D}^{\rm mtr} = \{\mathcal{D}_{k}^{\text{mtr}}\}_{k=1}^{K}= \{\mathcal{D}_{k}^{\text{tr}}, \mathcal{D}_{k}^{\text{va}}\}_{k=1}^{K}$ denote the overall meta training dataset.
With these definitions,
the \textbf{meta-training loss function} $\mathcal{L}^{\mathrm{ma}}_{\mathcal{D}^{\rm mtr}}\left(\theta\right)$ for MAML is the average of the validation loss across all meta-training tasks.
Following~\eqref{eq:ch14:trainalgo}, we also write
as ${\phi}^{\mathrm{ma}}(\mathcal{D}_{k}^{\mathrm{tr}} |  \theta)$  the updated model parameter vector based on training data $\mathcal{D}_{k}^{\mathrm{tr}}$ for task $k$ with initialization $\theta$, and aim to optimize
\begin{subequations}
	  \begin{align}
  \min_{\theta}\ 
     \label{eqn.maml}
  &\mathcal{L}^{\mathrm{ma}}_{\mathcal{D}^{\rm mtr}}\left(\theta\right)=\frac{1}{K} \sum_{k=1}^{K} L_{\mathcal{D}_{k}^{\mathrm{va}}}\left({\phi}^{\mathrm{ma}}(\mathcal{D}_{k}^{\mathrm{tr}} |  \theta) \right) \\
  \label{eq:pertask_para_ma}
  &\mathrm{s.t.}~~ {\phi}^{\mathrm{ma}}(\mathcal{D}_{k}^{\mathrm{tr}} |  \theta)=\theta-\alpha \nabla_{\theta} L_{ \mathcal{D}_{k}^{\mathrm{tr}}}
  \left( \theta \right).
  \end{align}
\end{subequations}
  Where $\alpha$ is predefined stepsize. Note that the updated model from~\eqref{eq:pertask_para_ma} corresponds to the one-step gradient update illustrated in Figure~\ref{fig:maml}.
  %

  %
 The MAML algorithm is summarized in Algorithm \ref{alg:maml_train_0}. 
  \begin{algorithm}[H]
    \caption{MAML}
    \label{alg:maml_train_0}
  \begin{algorithmic}[1]
    \State{\textbf{Input:}~
    Initial iterate $\theta$;
    meta-training data $\mathcal{D}^{\rm mtr}$;
    loss function $\ell(z  |  \phi)$;
    stepsizes $\alpha$ and $\beta$}
    \While{not converged}
    \State{Sample batch of tasks $\tilde{\mathcal{K}} \subseteq \mathcal{K} = \{1,\ldots, K\}$}
      \For{all $k \in \tilde{\mathcal{K}}$}
        \State{{Compute per-task parameter ${\phi}^{\rm ma} (\mathcal{D}_k^{\rm tr}  |  \theta)$ using \eqref{eq:pertask_para_ma}}}
      \EndFor
    \State{{Update hyperparameter vector $\theta$ as\\
    \hspace{0.5cm}$\theta \leftarrow \theta - \beta \nabla_{\theta} \mathcal{L}^{\rm ma}_{\mathcal{D}^{\rm mtr}}(\theta)$ using  in~\eqref{eqn.maml}}}
    \EndWhile
    \end{algorithmic}
  \end{algorithm}
  In order to apply MAML, in line~8 of Algorithm~\ref{alg:maml_train_0}, we need to compute the gradient $\nabla_{\theta} \mathcal{L}^{\rm ma}_{\mathcal{D}^{\rm mtr}} (\theta )$ of the meta-training loss in~\eqref{eqn.maml}.
  Using the chain rule of differentiation, with $I$ denoting the identity matrix, the gradient $\nabla_{\theta} \mathcal{L}^{\rm ma}_{\mathcal{D}^{\rm mtr}} (\theta)$  is computed as
  \begin{align}
    \label{eq:maml_chainrule}
    &\nabla_{\theta} 
    \mathcal{L}^{\rm ma}_{\mathcal{D}^{\rm mtr}} (\theta )
    = \frac{1}{K}\sum_{k=1}^K
    \nabla_{\theta} {\phi}^{\mathrm{ma}}(\mathcal{D}_{k}^{\mathrm{tr}} |  \theta) 
    \nabla_{{\phi}} L_{\mathcal{D}_{k}^{\mathrm{va}}}({\phi})  | _{{\phi} = {\phi}(\mathcal{D}_{k}^{\mathrm{tr}} |  \theta)}
    \nonumber \\
    &= \frac{1}{K}\sum_{k=1}^{K}\left(I-\alpha \nabla_{\theta}^{2} L_{ \mathcal{D}_{k}^{\mathrm{tr}}}\left(\theta\right)\right) 
    \nabla_{{\phi}} L_{\mathcal{D}_{k}^{\mathrm{va}}}({\phi}) | _{{\phi} = {\phi}(\mathcal{D}_{k}^{\mathrm{tr}} |  \theta)} ,
  \end{align}
  where  $\nabla_{\theta} {\phi}^{\mathrm{ma}}(\mathcal{D}_{k}^{\mathrm{tr}} |  \theta)$ represents the Jacobian of the updated parameter in \eqref{eq:pertask_para_ma} with respect to the initial parameter $\theta$.
  Therefore the update of $\theta$ in line~7 of Algorithm~\ref{alg:maml_train_0} is specified as
  \begin{equation}
  \label{eq:outer_update_ma}
  \theta \leftarrow \theta-\frac{\beta}{K}\sum_{k=1}^{K}\left(I-\alpha \nabla_{\theta}^{2} L_{ \mathcal{D}_{k}^{\mathrm{tr}}}\left(\theta\right)\right) \nabla_{{\phi}} L_{\mathcal{D}_{k}^{\mathrm{va}}}({\phi})  | _{{\phi} = {\phi}^{\mathrm{ma}}(\mathcal{D}_{k}^{\mathrm{tr}} |  \theta)}.
  \end{equation}
  The convergence rate of MAML has been first established in \cite{fallah2020convergence}, and later been improved in \cite{chen2021solving}.


\subsection{Implicit MAML}
\label{subs:imaml}
In implicit MAML (iMAML),
the per-task parameter ${\phi}$ is updated using hyperparameter vector $\theta$ by solving an $l_2$-regularized empirical risk minimization problem that penalizes deviations between per-task parameter ${\phi}$ and the hyperparameter $\theta$. 
Accordingly, the meta-training loss function $\mathcal{L}^{\mathrm{im}}_{\mathcal{D}^{\rm mtr}}\left(\theta \right)$ is defined as
\begin{subequations}\label{eqn.imaml}
	  \begin{align}
  &\mathcal{L}^{\mathrm{im}}_{\mathcal{D}^{\rm mtr}}\left(\theta \right)
  =\frac{1}{K} \sum_{k=1}^{K} L_{\mathcal{D}_{k}^{\mathrm{va}}}\left({\phi}^{\mathrm{im}}( \mathcal{D}_{k}^{\mathrm{tr}} |  \theta) \right) \\
  \label{eq:pertask_para_im}
  &\mathrm{s.t.}~~ {\phi}^{\mathrm{im}}( \mathcal{D}_{k}^{\mathrm{tr}} |  \theta)
  =\underset{\phi}{\arg \min }
  \left\{L_{\mathcal{D}_{k}^{\mathrm{tr}}}\left(\phi \right)+\frac{\lambda}{2}\left\|\phi-\theta\right\|^{2}\right\},
  \end{align}
\end{subequations}
  where $\lambda >0$ is a regularization constant.
  As compared to MAML, the gradient update in \eqref{eq:pertask_para_ma} is replaced by the minimizer of problem~\eqref{eq:pertask_para_im}.
  Note that, if the loss function $L_{\mathcal{D}_k^{\rm tr}} (\phi )$ is replaced in \eqref{eq:pertask_para_im} by its first-order Taylor expansion at $\theta$, i.e., by
  \begin{align}
    L_{\mathcal{D}_k^{\rm tr}}(\phi )
    = L_{\mathcal{D}_k^{\rm tr}}(\theta ) + \nabla L_{ \mathcal{D}_k^{\rm tr}}(\theta) ^{\top } (\phi - \theta),
  \end{align}
then problem~\eqref{eqn.imaml} coincides with problem~\eqref{eqn.maml}.
  
  The adapted parameter ${\phi}^{\rm im}( \mathcal{D}_k^{\rm tr}  |  \theta)$ in \eqref{eq:pertask_para_im} can be explained in terms of the proximal mapping for the per-task training loss $L_{\mathcal{D}_k^{\rm tr}}(\phi )$~\cite{Zhou2019_proxmaml}.
  This function is defined as
  \begin{align}\label{eq:prox_func}
    \mathrm{prox}_{L_{{\cal D}_k^{\rm tr}}, \lambda}(\theta) 
    = \mathop{\arg\min}_{\phi} \frac{\lambda}{2}\|\phi-\theta\|^{2} +
      L_{\mathcal{D}_{k}^{\rm tr}}(\phi ) .
  \end{align} 
  Therefore, the constraint in \eqref{eq:pertask_para_im} can be written as 
  \begin{align}\label{eq:prox_im}
    {\phi}^{\mathrm{im}}( \mathcal{D}_{k}^{\mathrm{tr}} |  \theta)
  = \mathrm{prox}_{L_{{\cal D}_k^{\rm tr}}, \lambda}(\theta) .
  \end{align}
  Based on the chain rule of differentiation and the implicit function theorem, the gradient descent update of hyperparameter $\theta$ during meta-learning is obtained from problem~\eqref{eqn.imaml}-\eqref{eq:pertask_para_im}  as \cite{rajeswaran2019_imaml}
  \begin{equation}
    \label{eq:outer_update_im}
  \theta \leftarrow \theta-\frac{\beta}{K}\sum_{k=1}^{K}\left(I+\frac{1}{\lambda} \nabla^{2}_{\theta} L_{{\mathcal{D}}_{k}^{\mathrm{tr}}}\left(\theta\right)\right)^{-1} \nabla_{{\phi}} L_{ \mathcal{D}_{k}^{\mathrm{tr}}}({\phi})  | _{{\phi} = {\phi}^{\mathrm{im}}( \mathcal{D}_{k}^{\mathrm{tr}} |  \theta)}.
  \end{equation}
The iMAML algorithm is summarized in Algorithm~\ref{alg:imaml_train}. 
\begin{algorithm}[H]
  \caption{iMAML}
  \label{alg:imaml_train}
\begin{algorithmic}[1]
  \State{\textbf{Input:}~
  Initial iterate $\theta$;
    meta-training data $\mathcal{D}$;
    loss function $\ell(z  |  \phi)$;
    stepsize $\beta$; regularization weight $\lambda$}
    \While{not converged}
    \State{Sample batch of tasks $\tilde{\mathcal{K}} \subseteq \mathcal{K} = \{1,\ldots, K\}$}
      \For{all $k \in \tilde{\mathcal{K}}$}
      \State{{Compute per-task ${\phi}^{\rm im} ( \mathcal{D}_k^{\rm tr}  |  \theta)$ by solving problem \\
      \hspace{1cm} \eqref{eq:pertask_para_im}}}
    \EndFor
  \State{{Update hyperparameter vector  $\theta$ via the gradient update \eqref{eq:outer_update_im}}}
  \EndWhile
  \end{algorithmic}
\end{algorithm}

\subsection{Implicit MAML for Ridge Regression}
\label{subs:imaml_RR}
In this subsection, we instantiate the iMAML scheme for the example of linear prediction via ridge regression.
Consider a linear prediction problem in which each $k$ task amounts to the optimization of a linear prediction over the model parameter vector ${\phi} \in \mathbb{R}^d$ given input vector $x_{k} \in \mathbb{R}^{d}$, which is computed as
\begin{align}
  \hat{y}_{k} = {\phi}^\top x_{k}.
\end{align}
The training data set is given as $\mathcal{D}_k^\text{tr} = (X_k^\text{tr}, \mathrm{y}_k^\text{tr})$, where \\
$X_k^\text{tr}=[x_{k,1}^\top,\ldots,x_{k,N^{\rm tr}}^\top]^\top$ is the $N^{\rm tr} \times d$ matrix that contains by row the transpose of the input vectors $\{ x_{k,n} \}_{n=1}^{N^{\rm tr}}$, and $\mathrm{y}_k^\text{tr}=[y_{k,1},\ldots,y_{k,N^{\rm tr}}]^\top$ as the $N^{\rm tr} \times 1$ vector that collects the corresponding labels $\{ y_{k,n}\}_{n=1}^{N^{\rm tr}}$. Similarly, we  define $\mathcal{D}_k^\text{va} = (X_k^\text{va}, \mathrm{y}_k^\text{va})$ as $X_k^\text{va}=[x_{k,1}^\top,\ldots,x_{k,N^\text{va}}^\top]^\top$ as the $N^\text{va}\times d$ input data and $\mathrm{y}_k^\text{va}=[y_{k,1},\ldots,y_{k,N^\text{va}}]^\top$ as the $N^\text{va}\times 1$ target labels for the validation data of the $k$-th task.

Given the task-specific model parameter ${\phi}$, the mean squared error (MSE) prediction loss given the data set $\mathcal{D}_k^\text{tr}$ can be written as
\begin{align}\label{eq:imaml_linear}
  L_{\mathcal{D}_k^\text{tr}}({\phi}) = \| X_k^\text{tr}{\phi} - \mathrm{y}_k^\text{tr} \|^2.
\end{align}
%
With the quadratic loss in~\eqref{eq:imaml_linear},
%
the solution of the inner problem \eqref{eq:pertask_para_im}, i.e.,  the proximal function in \eqref{eq:prox_im}, can be obtained analytically as
\begin{align}
  {\phi}^\text{im}(\mathcal{D}_k^\text{tr}  |  \theta) = 
  \Big(X_k^{\text{tr} \top} X_k^\text{tr} + \frac{\lambda}{2} I \Big)^{-1}
  \Big(X_k^{\text{tr} \top} \mathrm{y}_k^\text{tr} + \frac{\lambda}{2} \theta \Big).
\end{align}
As a result, the solution of the meta-training problem \eqref{eqn.imaml} can also be  computed in closed form as
\begin{align}
  \hat{\theta} &= \argmin_{\theta} \sum_{k=1}^K ||\tilde{X}_k^\text{va}\theta - \tilde{\rm y}_k^\text{va}||^2 \nonumber\\
&= \tilde{X}^{\dagger} \tilde{\rm y},
\label{eq:closed_from_sol_meta_scalar}
\end{align}
where the $N^\text{va}\times d$ matrix $\tilde{X}_k^\text{va}$ contains by row the transpose of the pre-conditioned input vectors $\{ \frac{\lambda}{2} (A_k^\text{tr})^{-1} x_{k,n}^\text{va} \}_{n=1}^{N^\text{va}}$, with $A_k^\text{tr} = (X_k^\text{tr})^\top X_k^\text{tr} + \frac{\lambda}{2} I$; $\tilde{y}_k^\text{va}$ is $N^\text{va}\times 1$ vector containing vertically the transformed outputs $\{ y^\text{va}_{k,n} - (\mathrm{y}_k^\text{tr})^\top X_k^\text{tr} (A_{k}^\text{tr})^{-1} x_{k,n}^\text{va} \}_{n=1}^{N^\text{va}}$; the $KN^\text{va}\times d$ matrix $\tilde{X} = [\tilde{X}_1^{\text{va}}, \ldots,\tilde{X}_K^{\text{va}}]^\top$ stacks vertically the $N^\text{va}\times d$ matrices $\{\tilde{X}_k^\text{va}\}_{k=1}^K$; and the $KN^\text{va}\times 1$ vector $\tilde{\rm y} = [\tilde{\rm y}_{1}^{\text{va}},\ldots,\tilde{\rm y}_{K}^{\text{va}}]^\top$ stacks vertically the $N^\text{va}\times 1$ vectors $\{\tilde{\rm y}_k^\text{va}\}_{k=1}^K$. Further discussions  can be found in~\cite{denevi2018_l2l_common_mean,bai2021important,chen2022_bamaml}.

\subsection{Sharp-MAML}

The nested  structure of the MAML problem~\eqref{eqn.maml}-\eqref{eq:pertask_para_ma} may cause the optimization landscape in the space of the hyperparameter $\theta$ to have many saddle points and local minima.
To illustrate this point, Figure~\ref{fig:MAML_JL_losslandscape} shows the loss landscapes of  MAML on $\mathcal{L}^{\rm ma}_{\mathcal{D}^{\rm mtr}}(\theta)$ given by \eqref{eqn.maml},  as compared to a standard joint learning model (see~\cite{abbas2022_sharpmaml} for details). 
Reference~\cite{abbas2022_sharpmaml} provides a formal statement of the observation in Figure~\ref{fig:MAML_JL_losslandscape} that the loss landscape of MAML is more involved as compared to joint learning, making the optimization problem potentially difficult to solve. 
\begin{figure}[ht]
  \centering
  \includegraphics[width=.7\textwidth]{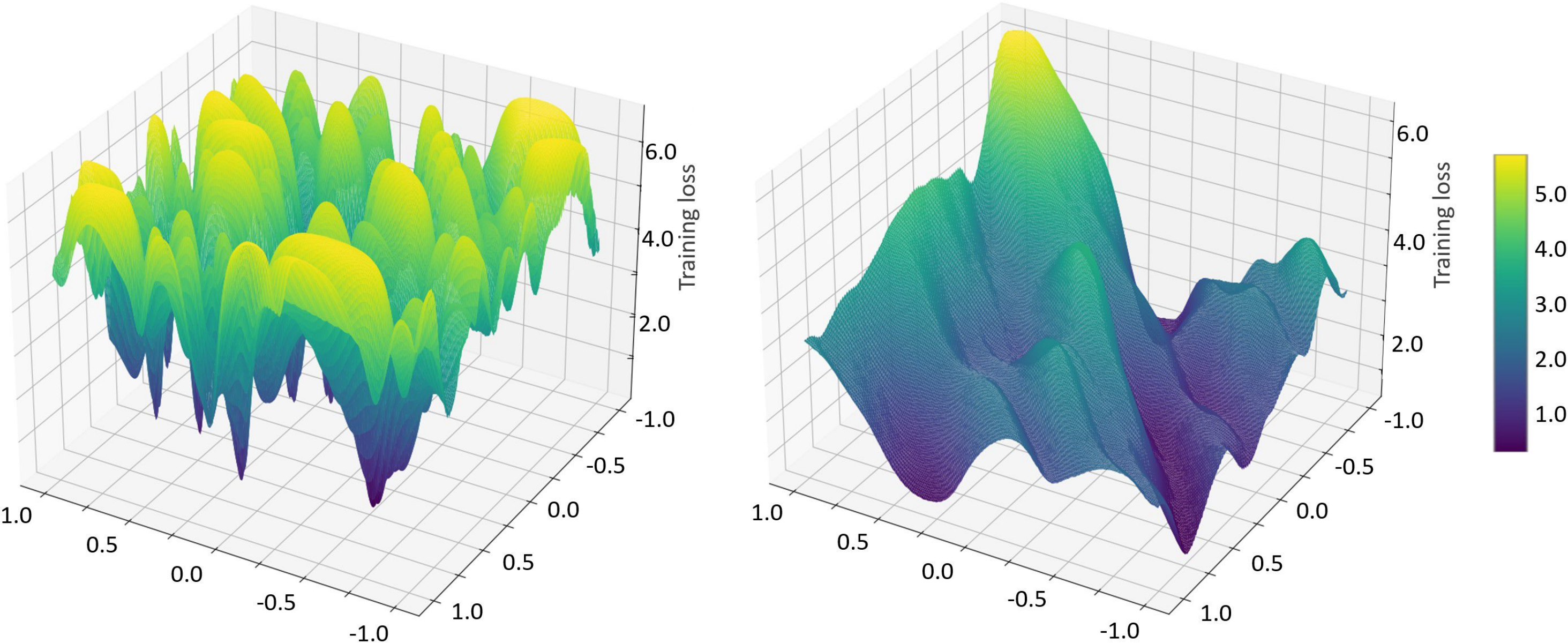}
  \caption{Loss landscapes for the MAML loss  $\mathcal{L}^{\rm ma}_{\mathcal{D}^{\rm mtr}}(\theta)$ (\textbf{left}) and for joint learning (see~\eqref{eq:ch14:joint2}) (\textbf{right}) for a single task on CIFAR-100 dataset~\cite{abbas2022_sharpmaml}.}
  \label{fig:MAML_JL_losslandscape}
\end{figure}

While some of the local minimizers in the loss landscape of MAML are indeed effective few-shot learners, there are a number of sharp local minimizers in MAML that may have undesired generalization performance.
Therefore, it is of interest to develop a method that can find local minimizers with better generalization ability, which motivates the Sharp-MAML algorithm introduced in~\cite{abbas2022_sharpmaml}.

Sharp-MAML is inspired by the recent development of the sharpness-aware minimization (SAM) algorithm~\cite{foret2020_sam}, which  avoids  sharp  local minimizers of the loss landscape to improve the generalization ability of the algorithm.
The idea is to find a solution such that the maximum loss of the parameter in the neighborhood of this solution is minimized.

Since MAML is formulated in \eqref{eqn.maml} as a bilevel optimization problem, ideally  the solutions of both inner-level and outer-level problems should have good generalization.
Sharp-MAML applies the idea of SAM  to both the inner- and outer-level problems~\eqref{eqn.maml} and \eqref{eq:pertask_para_ma}.
The resulting minimax problem is approximated by adding perturbations along the gradient ascent direction for both inner- and outer-level parameters, which are denoted as $\epsilon_k(\theta)$ and $\epsilon(\theta)$.
The loss function of Sharp-MAML is accordingly given as
\begin{subequations}
	\begin{align}
  &\mathcal{L}^{\mathrm{sm}}_{\mathcal{D}^{\rm mtr}}\left(\theta \right)
  =\frac{1}{K} \sum_{k=1}^{K} L_{\mathcal{D}_{k}^{\mathrm{va}}}\left({\phi}^{\mathrm{sm}}( \mathcal{D}_{k}^{\mathrm{tr}} |  \theta) \right) \\
  \mathrm{ s.t. }~~ 
  &{\phi}^{\mathrm{sm}}( \mathcal{D}_{k}^{\mathrm{tr}} |  \theta)=\theta+\epsilon\left(\theta\right)-\alpha \nabla_{\theta} L_{\mathcal{D}_{k}^{\mathrm{tr}}}\left(\theta+\epsilon\left(\theta\right)+\epsilon_{k}\left(\theta\right)\right),
  \label{eq:pertask_sm}
\end{align}
\end{subequations}
where the perturbations $\epsilon_k(\theta)$ and $\epsilon(\theta)$ are given as
\begin{subequations}
	\begin{align}
  \epsilon_k(\theta) &= \alpha_{\rm in} 
  \nabla_{\theta} L_{\mathcal{D}_k^{\rm tr}}(\theta )/\|\nabla_{\theta} L_{\mathcal{D}_k^{\rm tr}}(\theta )\|_2 \\
  \epsilon(\theta) &= \alpha_{\rm ot} 
  { \nabla_{\theta} {L}_{\mathcal{D}_k^{\rm va}} (\tilde{\phi}(\mathcal{D}_{k}^{\mathrm{tr}} |  \theta))}/{\|\nabla_{\theta} {L}_{\mathcal{D}_k^{\rm va}} (\tilde{\phi}(\mathcal{D}_{k}^{\mathrm{tr}} |  \theta))\|_2} \\
  \text{with }
  &\tilde{\phi}( \mathcal{D}_{k}^{\mathrm{tr}}  |  \theta) 
  = \theta-\alpha \nabla_{\theta} L_{\mathcal{D}_{k}^{\mathrm{tr}}}\left(\theta+\epsilon_{k}\left(\theta\right) \right),
\end{align}
\end{subequations}
with $\alpha_{\rm in}$ and $\alpha_{\rm ot}$ denoting the scalar hyperparameters for inner and outer-level perturbations to be used in \eqref{eq:pertask_sm}.

The outer-level update for Sharp-MAML is
\begin{equation}
  \theta \leftarrow \theta-\frac{\beta}{K} \sum_{k=1}^{K} \nabla_{\theta} L_{\mathcal{D}_{k}^{\mathrm{va}}}\left({\phi}^{\mathrm{sm}}(\mathcal{D}_{k}^{\mathrm{tr}}  |  \theta) \right).
\end{equation}

\section{First-Order Optimization-Based Meta-Learning} 
\label{sec:first_order_algorithms}
In this section, we {cover} optimization-based meta-learning algorithms that, 
unlike the second-order methods described in Section~\ref{sec:second_order_algorithms}, do not require computing the second-order Hessian of the loss function during training, leading to significantly reduced computational complexity.
These methods include first-order MAML ~\cite{Finn2017_maml}, ES-MAML, Reptile~\cite{nichol2018_reptile}, and Proximal MAML (Prox-MAML)~\cite{Zhou2019_proxmaml}.

\subsection{FOMAML}
First-order MAML (FOMAML), originally proposed in~\cite{Finn2017_maml}, uses the same formulation as MAML in \eqref{eqn.maml}.
  However, for the update of the hyperparameter $\theta$, FOMAML replaces the Jacobian $\nabla_{\theta} {\phi}^{\mathrm{ma}}(\mathcal{D}_{k}^{\mathrm{tr}} |  \theta)$ in \eqref{eq:maml_chainrule} by an identity matrix, hence foregoing the computation of the Hessian $\nabla_{\theta}^2 L_{\mathcal{D}_k^{\rm tr}} (\theta )$.
  The outer update of FOMAML is given by
  \begin{equation}
    \label{eq:outer_update_fo}
  \theta \leftarrow \theta-\frac{\beta}{K}\sum_{k=1}^{K} \nabla_{{\phi}} L_{\mathcal{D}_{k}^{\mathrm{va}}}({\phi} ) | _{{\phi} = {\phi}( \mathcal{D}_{k}^{\mathrm{tr}} |  \theta)},
  \end{equation}
  where the function ${\phi}^{\mathrm{fo}}( \mathcal{D}_{k}^{\mathrm{tr}} |  \theta)$ is computed by
  \begin{align}\label{eq:pertask_para_fo}
    {\phi}^{\mathrm{fo}}(\mathcal{D}_{k}^{\mathrm{tr}} |  \theta)=\theta-\alpha \nabla_{\theta} L_{ \mathcal{D}_{k}^{\mathrm{tr}}}\left(\theta\right).
  \end{align}
  The FOMAML algorithm is summarized in Algorithm \ref{alg:fomaml_train}.  
  \begin{algorithm}[H]
    \caption{FOMAML}
    \label{alg:fomaml_train}
  \begin{algorithmic}[1]
    \State{\textbf{Input:}~
    Initial iterate $\theta$; meta-training data $\mathcal{D}$;
    loss function $\ell(z  |  \phi)$;
    stepsizes $\alpha, \beta$}
    \While{not converged}
    \State{Sample batch of tasks $\tilde{\mathcal{K}} \subseteq \mathcal{K} = \{1,\ldots, K\}$}
      \For{all $k \in \tilde{\mathcal{K}}$}
        \State{ {Compute per-task parameter ${\phi}^{\rm fo}(\mathcal{D}_k^{\rm tr}  |  \theta)$ using \eqref{eq:pertask_para_ma}}}
      \EndFor
    \State{ {Update hyperparameter vector  $\theta$ via the gradient update \eqref{eq:outer_update_fo}}}
    \EndWhile
    \end{algorithmic}
  \end{algorithm}

\subsection{ES-MAML}
ES-MAML~\cite{song2019_esmaml} addresses the MAML problem in \eqref{eqn.maml} via evolution strategies (ES), a black-box optimization algorithm~\cite{Beyer2002}.  
In a nutshell,
similar to MAML, the task-specific parameter ${\phi}^{\mathrm{es}}$ is also obtained via one-step gradient update initialized at the hyperparameter $\theta$. 
The difference with MAML concerns the meta-update in lines 5 and 8 of Algorithm~\ref{alg:maml_train_0}, in which
the gradient 
${\nabla}_{\theta} L_{\mathcal{D}_{k}^{\mathrm {tr}}}\left(\theta \right)$ is replaced with the ES multi-point gradient estimator.
Accordingly, 
the update  of the hyperparameter $\theta$ is obtained as
\begin{align}
  \label{eq:outer_update_es}
  &\theta \leftarrow \theta-\frac{\beta}{K}\sum_{k=1}^{K}\left(I-\alpha H_{k}^{\rm es}\right) \hat{\nabla}_{{\phi}} L_{\mathcal{D}_{k}^{\mathrm{va}}}({\phi}) | _{{\phi} = {\phi}( \mathcal{D}_{k}^{\mathrm{tr}} |  \theta)}\\
  \label{eq:outer_update_es2}
\text{or as~~} &\theta \leftarrow \theta-\frac{\beta}{K}\sum_{k=1}^{K} \hat{\nabla}_{{\phi}} L_{\mathcal{D}_{k}^{\mathrm{va}}}({\phi}) | _{{\phi} = {\phi}^{\mathrm{es}}( \mathcal{D}_{k}^{\mathrm{tr}} |  \theta)},
\end{align}    
where $\hat{\nabla}_{{\phi}} L_{\mathcal{D}_{k}^{\mathrm{va}}}({\phi})$ is the ES multi-point gradient estimator of  
$\nabla_{{\phi}} L_{\mathcal{D}_{k}^{\mathrm{va}}}({\phi})$, which queries multiple points in the parameter space of the hyperparameter $\theta$, along with their loss function values. 
And $H_{k}^{\rm es}$ denotes the ES Hessian estimator of $\nabla_{\theta}^{2} L_{\mathcal{D}_{k}^{\mathrm {tr}}}(\theta)$. 
The gradient is estimated by the sample average of the function value difference in randomly sampled directions.
Specifically, the $n$-point ES gradient estimator of a loss function $L(\phi)$ is computed as
\begin{align}\label{eq:es_gradient_estimator}
  \hat{\nabla}_{\phi} L(\phi)
  = \frac{1}{n}\sum_{i=1}^n \Big[\frac{{u}_i}{\delta}
  \Big(L(\phi + \delta {u}_i) \Big)\Big],
\end{align}
where ${u}_i$ is a random vector sampled from distribution $\mathcal{N}(\mathrm{0}, \mathrm{I})$ in the same space as $\phi$;
and
$\delta$ is a fixed parameter that controls the distance between the two points used to estimate the gradient.

Analogously, the ES Hessian estimator $H^{\rm es}$ can be computed by applying the gradient estimator twice, yielding
\begin{align}
  H^{\rm es} = \frac{1}{\delta^2}
  \Big(\frac{1}{n}\sum_{i=1}^n L(\phi + \delta {u}_i)
  {u}_i {u}_i^{\top} 
  - \frac{1}{n}\sum_{i=1}^n L(\phi + \delta {u}_i) \mathrm{I} \Big).
\end{align}

The ES-MAML algorithm is summarized in Algorithm \ref{alg:esmaml_train}. 
\begin{algorithm}[t]
  \caption{ES-MAML}
  \label{alg:esmaml_train}
\begin{algorithmic}[1]
  \State{\textbf{Input:}~
  Initial iterate $\theta$;
  meta-training data $\mathcal{D}$;
  loss function $\ell(z  |  \phi)$;
  stepsizes $\alpha, \beta$}
  \While{not converged}
  \State{Sample batch of tasks $\tilde{\mathcal{K}} \subseteq \mathcal{K} = \{1,\ldots, K\}$}
      \For{all $k \in \tilde{\mathcal{K}}$}
      \State{{Compute per-task parameter ${\phi}^{\rm es}(\mathcal{D}_k^{\rm tr}  |  \theta)$ using \\
      \hspace{1cm}${\phi}^{\mathrm{es}}\left(\mathcal{D}_{k}^{\mathrm{tr}} |  \theta \right)=\theta-\alpha \hat{\nabla}_{\theta} L_{ \mathcal{D}_{k}^{\mathrm{tr}}}\left(\theta\right)$ estimated via \eqref{eq:es_gradient_estimator}}}
    \EndFor
  \State{{Update hyperparameter vector $\theta$ via the gradient update \eqref{eq:outer_update_es} \\
  \hspace{0.5cm} or \eqref{eq:outer_update_es2}}}
  \EndWhile
  \end{algorithmic}
\end{algorithm}

\subsection{Reptile}
Reptile~\cite{nichol2018_reptile} shares the same general {formulation} as FOMAML. Considering the one-step per-task gradient update
    \begin{align}
    \label{eq:pertask_para_re}
    {\phi}^{\mathrm{re}}(\mathcal{D}_{k}^{\rm tr} |  \theta )=
    \theta-\alpha \nabla_{\theta} L_{ \mathcal{D}_{k}^{\mathrm{tr}}}\left(\theta\right),
    \end{align}
which coincides with the FOMAML update~\eqref{eq:pertask_para_fo}.
Reptile follows an approach akin to the Fed Avg algorithm~\cite{mcmahan2017communication} to update the hyperparameter $\theta$.  
Specifically, the hyperparameter vector $\theta$ is updated in the direction of the average of the task-specific parameters in \eqref{eq:pertask_para_re} as
  \begin{equation}
    \label{eq:outer_update_re}
  \theta \leftarrow(1-\beta) \theta+\frac{\beta}{K}\sum_{k=1}^{K} {\phi}^{\mathrm{re}}(\mathcal{D}_k^{\rm tr} |  \theta),
  \end{equation}
  where $\beta > 0$ is a constant.
Reptile is summarized in Algorithm \ref{alg:reptile_train}.
\begin{algorithm}[H]
  \caption{Reptile}
  \label{alg:reptile_train}
\begin{algorithmic}[1]
  \State{\textbf{Input:}~
  Initial iterate $\theta$;
  meta-training data $\mathcal{D}$;
  loss function $\ell(z  |  \phi)$;
  stepsizes $\alpha,\beta$}
  \While{not converged}
  \State{Sample batch of tasks $\tilde{\mathcal{K}} \subseteq \mathcal{K} = \{1,\ldots, K\}$}
      \For{all $k \in \tilde{\mathcal{K}}$}
      \State{ {Compute per-task parameter ${\phi}^{\rm re}(\mathcal{D}_{k}^{\rm tr} |  \theta )$ using \eqref{eq:pertask_para_re}}}
    \EndFor
  \State{ {Update hyperparameter vector  $\theta$ by the gradient update \eqref{eq:outer_update_re}}}
  \EndWhile
  \end{algorithmic}
\end{algorithm}

\subsection{Prox-MAML}
\label{subs:proxmaml}
Prox-MAML~\cite{Zhou2019_proxmaml} adopts a bilevel formulation where the inner-level loss function is the same as that of iMAML in \eqref{eq:pertask_para_im}, and the outer-level meta-loss is the average of the inner-level loss across all tasks.
Mathematically, the bilevel problem is formulated as
  \begin{align}
  &\hspace{-0.3cm}\mathcal{L}^{\mathrm{pr}}_{\mathcal{D}^{\rm mtr}}\left(\theta\right)
  =\frac{1}{K} \sum_{k=1}^{K} L_{\mathcal{D}_{k}^{\rm mtr}}\left({\phi}^{\mathrm{pr}}(\mathcal{D}_{k}^{\rm mtr} |  \theta ) \right)+\frac{\lambda}{2}\left\|{\phi}^{\mathrm{pr}}(\mathcal{D}_{k}^{\rm mtr} |  \theta )-\theta\right\|^{2} \\
  \label{eq:pertask_para_pr}
  &\mathrm{s.t.}~~ 
  {\phi}^{\mathrm{pr}}(\mathcal{D}_{k}^{\rm mtr} |  \theta )=\mathrm{prox}_{L_{{\cal D}_k^{\rm mtr}},\lambda} (\theta),
  \end{align}
  where we have used the definition of proximal mapping in \eqref{eq:prox_func}.

    The gradient of the hyperparameter $\theta$ can be derived as
    \begin{align}\label{eq:grad_proxmaml}
    \! \nabla_{\theta} \mathcal{L}^{\rm pr}_{\mathcal{D}^{\rm mtr}}(\theta)
     \! =\! 
      &\frac{1}{K} \sum_{k=1}^{K} 
      \nabla_{\theta} {\phi}^{\mathrm{pr}}(\mathcal{D}_{k}^{\rm mtr} |  \theta ) \nabla_{{\phi}} \Big( L_{\mathcal{D}_{k}^{\rm mtr}}({\phi} )  \!+  \!\frac{\lambda}{2}\|{\phi}-\theta\|^{2}\Big)  \Big| _{{\phi}={\phi}^{\mathrm{pr}}(\mathcal{D}_{k}^{\rm mtr} |  \theta )} \nonumber \\
      &+ \frac{1}{K} \sum_{k=1}^{K} \lambda (\theta - {\phi}^{\mathrm{pr}}(\mathcal{D}_{k}^{\rm mtr} |  \theta )).
    \end{align}
    Furthermore, by \eqref{eq:pertask_para_pr},  for all $k \in [K]$, we have the equality 
    \begin{align}
      \nabla_{{\phi}} \Big( L_{\mathcal{D}_{k}}({\phi} )+\frac{\lambda}{2}\left\|{\phi}-\theta\right\|^{2}\Big)  \Big| _{{\phi}={\phi}^{\mathrm{pr}}(\mathcal{D}_{k} |  \theta )} = {0},
    \end{align}
    implying that the gradient $\nabla_{\theta} \mathcal{L}^{\rm pr}_{\mathcal{D}^{\rm mtr}}(\theta)$ in \eqref{eq:grad_proxmaml} can be simplified as
    \begin{align}
      \nabla_{\theta} \mathcal{L}^{\rm pr}_{\mathcal{D}^{\rm mtr}}(\theta)
      =& \frac{1}{K} \sum_{k=1}^{K} \lambda (\theta - {\phi}^{\mathrm{pr}}(\mathcal{D}_{k}^{\rm mtr} |  \theta )).
    \end{align}
    It follows that the update equation for Prox-MAML is given as
    \begin{align}
      \label{eq:outer_update_pr}
    \theta \leftarrow \theta-\beta{\lambda\left(\theta-\frac{1}{K} \sum_{k=1}^{K} {\phi}^{\mathrm{pr}}\left(\mathcal{D}_{k} |  \theta \right)\right)}.
    \end{align}

The Prox-MAML algorithm is summarized in Algorithm \ref{alg:proxmaml_train}. 
\begin{algorithm}[H]
  \caption{Prox-MAML}
  \label{alg:proxmaml_train}
\begin{algorithmic}[1]
  \State{\textbf{Input:}~
  Initial iterate $\theta$;
  meta-training data $\mathcal{D}^{\rm mtr}$;
  loss function $\ell(z  |  \phi)$;
  stepsizes $\alpha,\beta$}
  \While{not converged}
  \State{Sample batch of tasks $\tilde{\mathcal{K}} \subseteq \mathcal{K} = \{1,\ldots, K\}$}
      \For{all $k \in \tilde{\mathcal{K}}$}
      \State{ {Compute per-task parameter ${\phi}^{\rm pr}(\mathcal{D}_{k}^{\rm tr} |  \theta )$ using \eqref{eq:pertask_para_pr}}}
    \EndFor
  \State{ {Update hyperparameter vector  $\theta$ via gradient update \eqref{eq:outer_update_pr}}}
  \EndWhile
  \end{algorithmic}
\end{algorithm}

\section{Bayesian Meta-Learning}
\label{subs:bayes_maml}
MAML optimizes a conventional frequentist learning process that outputs an optimized model parameter ${\phi}$ for each task $k$.
Frequentist learning is well known to be ineffective at quantifying uncertainty, and at providing well-calibrated decision (see e.g., {\cite{simeoneCUP,zecchin2022robust}}).
In contrast, Bayesian learning, 
retains information about uncertainty in the model parameter space by evaluating,
ideally, the posterior distribution, $p(\phi  |  \mathcal{D}_{k}^{\mathrm{tr}}, \theta)$ of the task-specific parameter ${\phi}$,
given the training data set $\mathcal{D}_{k}^{\mathrm{tr}}$.
According to the Bayes rule, the  posterior distribution is
\begin{align}\label{eq:post_bayesrule}
  p(\phi  |  \mathcal{D}_{k}^{\mathrm{tr}}, \theta)
  =\frac{ p(\mathcal{D}_{k}^{\mathrm{tr}}  |  \phi ) 
  p(\phi  |  \theta)}
  {p(\mathcal{D}_{k}^{\mathrm{tr}}  |  \theta)},
\end{align}
where
$p(\mathcal{D}_{k}^{\mathrm{tr}}  |  \phi )$ is the likelihood  of parameter ${\phi}$; $p(\phi  |  \theta)$ is the prior of the parameter  ${\phi}$, which is allowed to depend on the hyperparameter $\theta$;
and $p(\mathcal{D}_{k}^{\mathrm{tr}}  |  \theta)$ is the evidence or the normalizing constant, with $p(\mathcal{D}_{k}^{\mathrm{tr}}  |  \theta) = \int p(\mathcal{D}_{k}^{\mathrm{tr}}  |  \phi ) 
p(\phi  |  \theta) d {\phi} $.
Importantly, by~\eqref{eq:post_bayesrule}, we assume that the prior distribution $p({\phi} | \theta)$ can be controlled via a vector $\theta$ of hyperparameters, paving the way for the use of meta-learning.
\begin{figure}[ht]
  \centering
  \includegraphics[width=.3\linewidth]{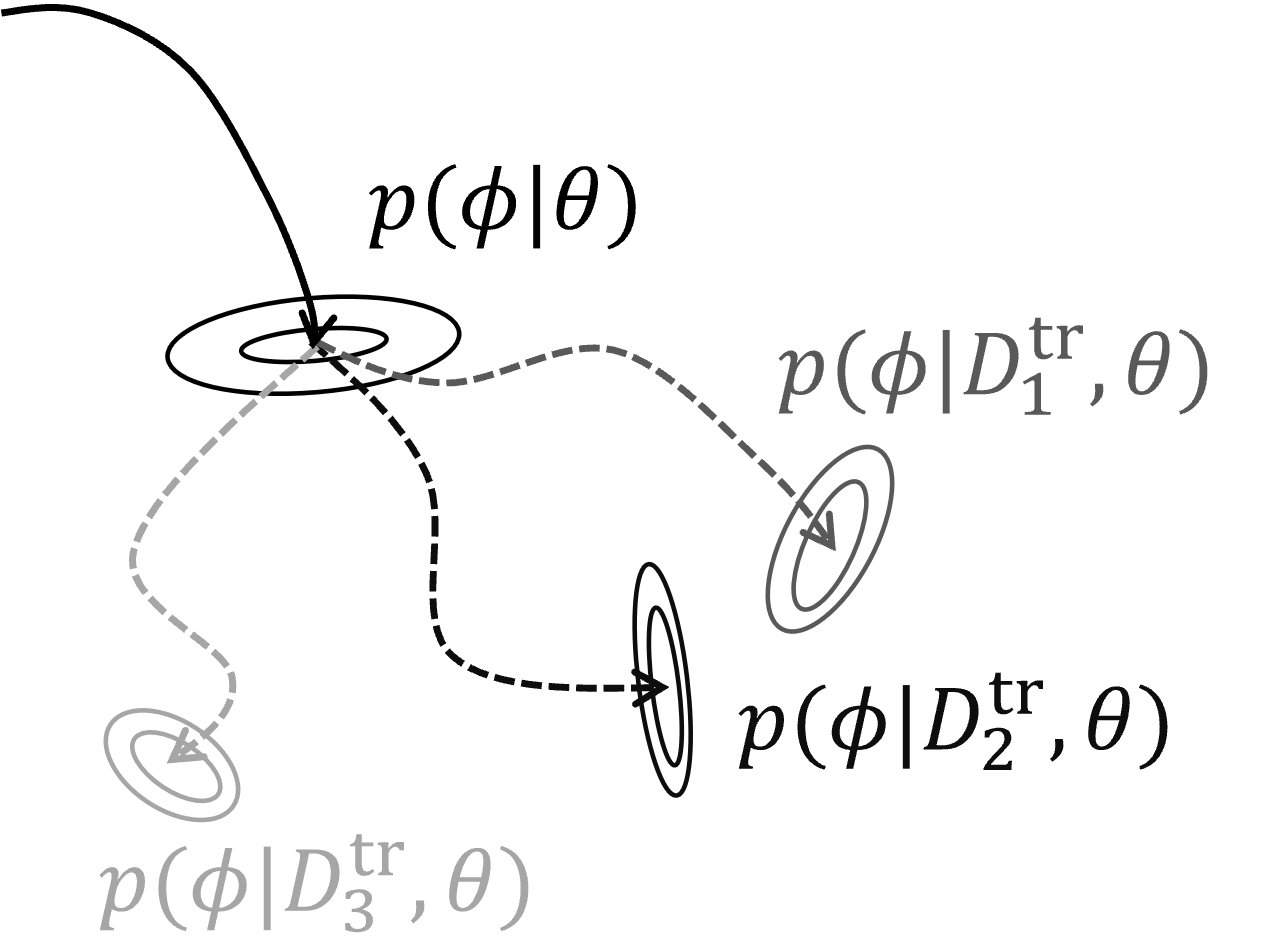}
  \caption{Illustration of Bayesian meta-learning: Bayesian meta-learning obtains a posterior distribution $p(\phi  |  \mathcal{D}_{k}^{\mathrm{tr}}, \theta)$ of the $k$-th task parameter by updating a prior distribution $p(\phi|\theta)$ shared across tasks and determined by the hyperparameter $\theta$.}
  \label{fig:bamaml}
\end{figure}

In problems of practical interest,
the normalizing constant in \eqref{eq:post_bayesrule} is typically intractable. 
Therefore, instead of the exact computation of the posterior \eqref{eq:post_bayesrule}, Bayesian learning algorithms obtain an approximation, $\hat{p}(\phi  |  \mathcal{D}_{k}^{\mathrm{tr}}, \theta)$.
Among the most common techniques, the posterior distribution can be approximated by Laplace approximation~\citep{grant2018recasting}, by parametric or non-parametric variational inference~\citep{nguyen2020_VAMPIRE,yoon2018_BMAML}, or via Monte Carlo sampling methods~\citep{Wang2020Bayesian} (see also reviews in \cite{andrieu2003introduction,simeoneCUP}).

Here we focus on the variational inference formulation, 
which minimizes the divergence between the approximate and the true posterior distributions.   
For two distributions $p(\phi)$ and $q(\phi)$ defined on a common space, the Kullback-Leibler (KL) divergence is defined as
\begin{align}
  \mathrm{D}_{\mathrm{KL}}(p(\phi) \| q(\phi))
  = \mathbb{E}_{p(\phi)} [\log p(\phi) - \log q(\phi)].
\end{align}
Bayesian meta-learning aims at optimizing the hyperparameter $\theta$ of the prior distribution $p({\phi}|\theta)$ that is shared across all tasks.
Bayesian learning via variational inference optimizes the approximate posterior
$\hat{p}({\phi}  |  \mathcal{D}_k^{\rm tr}, \theta)$ within a set $\mathcal{Q}$ of parametric distributions, e.g., the set of Gaussian distributions parameterized by the mean and covariance. 
Bayesian meta-learning aims at optimizing the prior distribution $p(\phi|\theta)$.
This is achieved by minimizing the KL divergence $\mathrm{D_{KL}}\big(\hat{p}(\phi  |  \mathcal{D}_{k}^{\mathrm{tr}}, \theta) \| {p}(\phi  |  \mathcal{D}_{k}^{\mathrm{tr}}, \theta)\big)$, equivalent to minimizing the variational free energy~\cite{bishop2006pattern,simeoneCUP}, given by
\begin{align}\label{eq:variational_free_energy}
  \hat{p}(\phi  |  \mathcal{D}_{k}^{\mathrm{tr}}, \theta)
=\underset{q\left(\phi\right) \in \mathcal{Q}}
{\arg \min }~
-\mathbb{E}_{q(\phi) }
\Big[\log p(\mathcal{D}_{k}^{\mathrm{tr}}  |  \phi ) \Big] 
+ \mathrm{D}_{\rm KL}\Big( q(\phi) \|p(\phi  |  \theta) \Big) .
\end{align}
The variational free energy in \eqref{eq:variational_free_energy} is the average training log-loss -- first term in \eqref{eq:variational_free_energy}, penalized by the deviation of the approximation $q({\phi})$ from the prior $q({\phi}|\theta)$ via the second term in \eqref{eq:variational_free_energy}. 
%
Accordingly, the meta-training loss for Bayesian meta-learning is given as
  \begin{align}
  \label{eq:bmaml}
  &\mathcal{L}^{\mathrm{ba}}_{\mathcal{D}^{\rm mtr}}\left(p(\phi  |   \theta) \right)=\frac{1}{K} \sum_{k=1}^{K} L_{\mathcal{D}_{k}^{\mathrm{va}}}\left(\hat{p}(\phi  |  \mathcal{D}_{k}^{\mathrm{tr}}, \theta) \right) \\
  &\label{eq:bmaml_lower}
  \mathrm{s.t.}~
  \hat{p}(\phi  |  \mathcal{D}_{k}^{\mathrm{tr}}, \theta)
  =\underset{q\left(\phi\right) \in \mathcal{Q}}
  {\arg \min }~
  -\mathbb{E}_{q(\phi) }
  \Big[\log p(\mathcal{D}_{k}^{\mathrm{tr}}  |  \phi ) \Big] 
  + \mathrm{D}_{\rm KL}\Big( q(\phi) \|p(\phi  |  \theta) \Big),
  \end{align}
  where the loss function $L_{\mathcal{D}_{k}^{\mathrm{va}}}(\hat{p}(\phi  |  \mathcal{D}_{k}^{\mathrm{tr}}, \theta) )$ is typically specified as the  negative log-loss computed on validation data based on the approximate posterior $\hat{p}(\phi  |  \mathcal{D}_{k}^{\mathrm{tr}}, \theta)$, i.e.,~\cite{yoon2018_BMAML} 
  \begin{align}
    L_{\mathcal{D}_{k}^{\mathrm{va}}}(\hat{p}(\phi  |  \mathcal{D}_{k}^{\mathrm{tr}}, \theta) )
    = -\log \int p(\mathcal{D}_{k}^{\mathrm{va}}  |  {\phi})
    \hat{p}(\phi  |  \mathcal{D}_{k}^{\mathrm{tr}}, \theta)
    d {\phi}.
  \end{align}
  The objective is typically estimated via the Monte Carlo sampling~\cite{yoon2018_BMAML}. 
    
  Theoretically, the performance of Bayesian meta-learning compared to MAML and iMAML has been established in~\cite{chen2022_bamaml}.
  Practically,
  there exist a variety of Bayesian meta-learning algorithms~\cite{grant2018recasting,yoon2018_BMAML,finn2018_PLATIPUS,nguyen2020_VAMPIRE,ravi2018_ABML}, which mainly differ in the definitions of the set $\mathcal{Q}$ used in~\eqref{eq:bmaml_lower}, and in the approximation methods used to approximate the solution of the variational free energy minimization problem~\eqref{eq:bmaml_lower}.
  BMAML~\cite{yoon2018_BMAML} adopts a non-parametric variational inference approximation method, which approximates  the posterior $p(\phi  |  \mathcal{D}_{k}^{\mathrm{tr}}, \theta)$
  via a set of particles $\bphi_k = \{\phi_{k,1},\dots,\phi_{k,M}\}$, 
  and also specifies the prior distribution $p(\phi  |   \theta)$ via a set of particles $\btheta=\{\theta_{1},\dots,\theta_{M}\}$.
  Specifically, BMAML adopts 
  the Stein Variational Gradient Descent (SVGD) algorithm~\cite{liu2016_svgd} to  update the particles $\bphi_k$ when addressing  problem~\eqref{eq:bmaml_lower}. 
  Accordingly, the updates for the per-task particles  $\bphi_k$ and the set of hyperparameter vectors $\btheta$ are, respectively, given as
  \begin{subequations}
  	 \begin{align}
    \label{eq:update_Theta_bmaml}
 \bphi_{k}(\mathcal{D}_k^{\rm tr}  |  \btheta) &\leftarrow 
  \operatorname{SVGD}
  (\btheta, \mathcal{D}_{k}^{\mathrm{tr}}, \alpha) \\
  \label{eq:outer_update_bm}
  \text{and }~~~\btheta &\leftarrow \btheta- {\beta} \nabla_{\btheta}  \mathcal{L}^{\mathrm{ba}}_{\mathcal{D}^{\rm mtr}} \left(\btheta\right).
  \end{align}
  \end{subequations}
In~\eqref{eq:update_Theta_bmaml},
the SVGD update is given by~\cite{liu2016_svgd}
\begin{align}
  &\mathrm{SVGD}(\btheta, \mathcal{D}_k^{\rm tr},\alpha) \nonumber\\
  =&\theta + \alpha \frac{1}{M} \sum_{m=1}^M
  \Big[ \kappa(\theta^{m}, \theta) \nabla_{\theta^{m}} 
  \log p(\theta^{m}  |  \mathcal{D}_k^{\rm tr})
  +\nabla_{\theta^{m}} \kappa(\theta^{m}, \theta)
  \Big], \forall \theta \in \btheta,
\end{align}
where $\alpha >0$ is the step size, and $\kappa(\theta, \theta')$ is a positive definite kernel, e.g., the radial basis function kernel~\cite{yoon2018_BMAML}.

The BMAML algorithm is summarized in Algorithm~\ref{alg:bmaml_train}.
  \begin{algorithm}[t]
    \caption{BMAML}
    \label{alg:bmaml_train}
  \begin{algorithmic}[1]
    \State{\textbf{Input:}~
    Initial particles $\btheta$;
    meta-training data $\mathcal{D}$;
    loss function $\ell(z  |  \phi)$;
    stepsizes $\alpha, \beta$}
    \While{not converged}
    \State{Sample batch of tasks $\tilde{\mathcal{K}} \subseteq \mathcal{K} = \{1,\ldots, K\}$}
    \For{all $k \in \tilde{\mathcal{K}}$}
        \State{{Update per-task parameter particles ${\bphi}_k(\mathcal{D}_k^{\rm tr}  |  \btheta)$ using \eqref{eq:update_Theta_bmaml}}}
      \EndFor
    \State{{Update hyperparameter vectors $\btheta$ via the SVGD update \eqref{eq:outer_update_bm}}}
    \EndWhile
    \end{algorithmic}
  \end{algorithm}

\subsection{Discussion on Empirical Performance}
In this subsection,
we evaluate the empirical performance on regression and classification tasks of some of the meta-learning algorithms introduced in this section.

We first consider the standard benchmark regression problem in which testing tasks are characterized by different ground-truth sinusoidal regression functions~\cite{Finn2017_maml}.
\begin{figure}[t]
  \centering
  \begin{subfigure}[t]{0.3\textwidth}
  \centering
  \includegraphics[width=.95\linewidth]{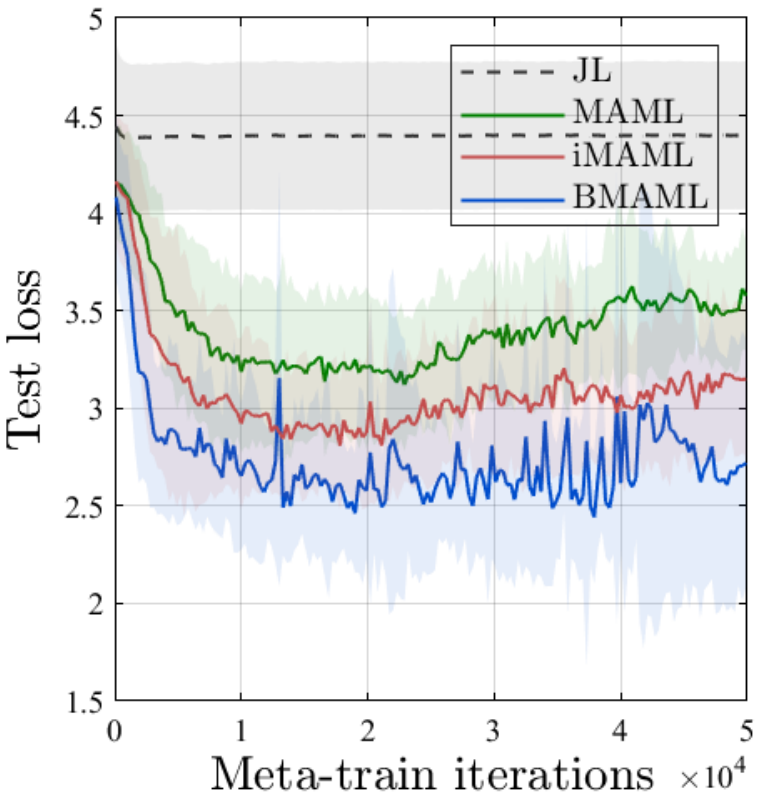}
  \caption{Training curve}
  \label{sfig:testcurve_T100_N10}
  \end{subfigure}%
  \begin{subfigure}[t]{0.3\textwidth}
  \centering
  \includegraphics[width=.91\linewidth]{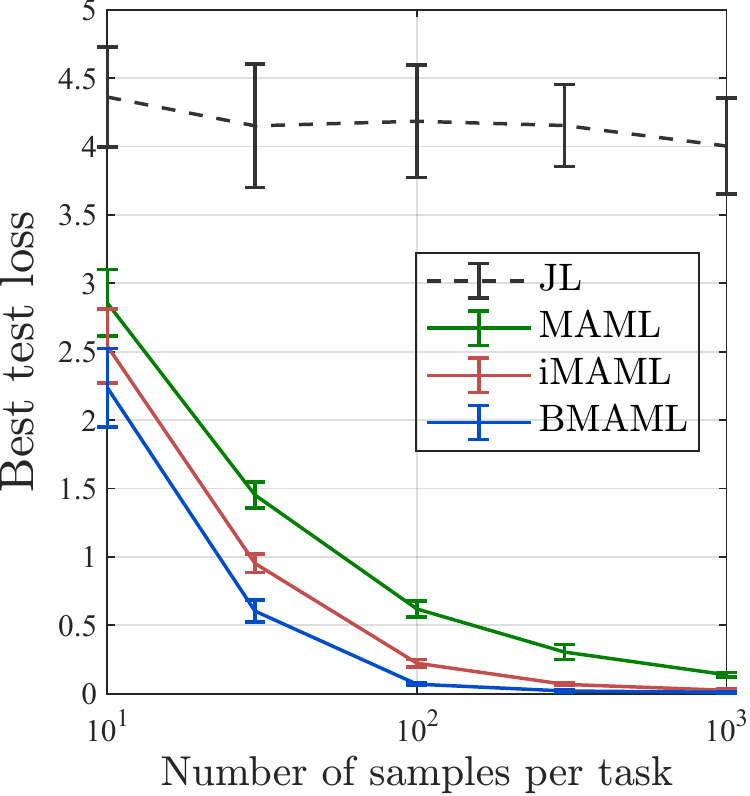}
  \caption{Best test loss vs. $N$}
  \label{sfig:best_test_loss_vs_N_T10_v6}
  \end{subfigure}
  \begin{subfigure}[t]{0.3\textwidth}
    \centering
    \includegraphics[width=.95\linewidth]{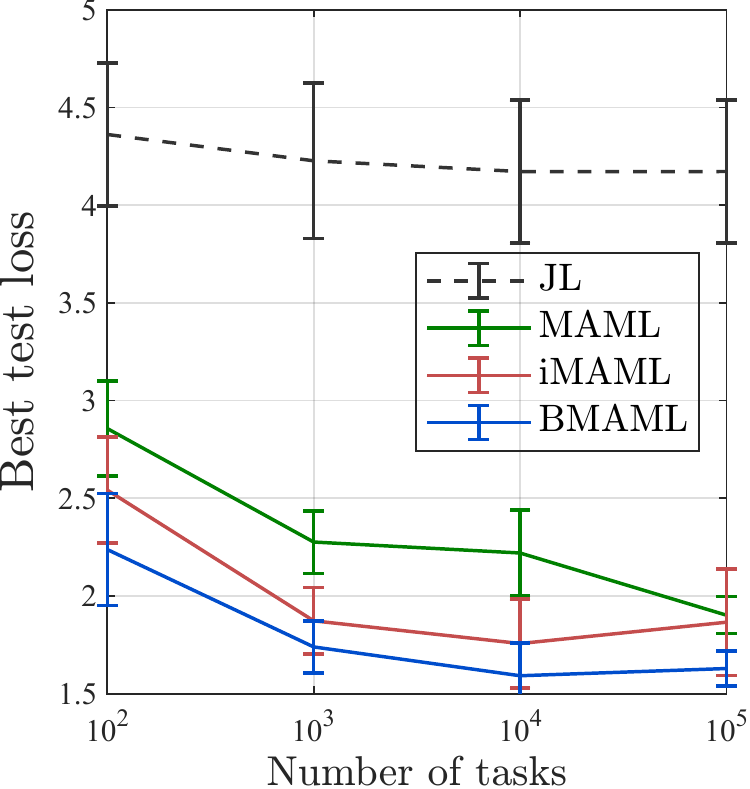}
    \caption{Best test loss vs. $K$}
    \label{sfig:best_test_loss_vs_T_N10_v6}
    \end{subfigure}
  \caption{Comparison of the performance of joint learning (JL), MAML, iMAML and BMAML in the sinusoidal regression problem introduced in~\cite{chen2022_bamaml}:
  (a) Test loss vs. meta-training iterations; 
  (b) best test loss vs. number of data per task $N$;
  (c) best test loss vs. number of tasks $K$.
  }
  \label{fig:test_loss_NT}
\end{figure}
%
%
{Compare the empirical performance of joint learning (JL), MAML, iMAML and BMAML with the same neural network architecture (see~\cite{yoon2018_BMAML,nguyen2020_VAMPIRE,chen2022_bamaml} for details) in Figure~\ref{fig:test_loss_NT}.}
For more results under different hyperparameters, refer to~\cite{yoon2018_BMAML,nguyen2020_VAMPIRE,chen2022_bamaml}. 
JL is observed to be unable to effectively  adapt to new tasks, in contrast to  meta-learning methods.
Among meta-learning algorithms, BMAML is observed to outperform iMAML and MAML when a small number of training data and tasks are given because of its ability to manage uncertainty. All three meta-learning methods have close to zero test loss when a sufficiently large number of training data per task, or when a sufficiently large number of tasks are given.

\begin{table}[ht]
  \fontsize{10}{10}\selectfont
  \caption{Accuracy (\%) of few-shot image classification on Mini-Imagenet (5-way).}
  \label{tab:acc_miniimagenet}
  \begin{center}
  \begin{tabular}{l cc}
  \toprule
  Algorithms  & 5-way 1-shot & 5-way 5-shot\\
  \midrule
  MAML \cite{Finn2017_maml}                         
  & 48.70 
  &63.11 
  \\ iMAML~\cite{rajeswaran2019_imaml}
  & 49.30 
  & - \\
  CAVIA~\cite{zintgraf2019_cavia}
  & 47.24 
  & 59.05 
  \\ FOMAML~\cite{nichol2018_fomaml} 
  & 48.07 
  &63.15 
  \\ Reptile~\cite{nichol2018_reptile} 
  & 49.97 
  &65.99 
  \\ Prox-MAML~\cite{Zhou2019_proxmaml} 
  & 50.77 
  & 67.43 
  \\ BMAML~\cite{yoon2018_BMAML}
  & 49.17 
  &64.23 
  \\ Sharp-MAML~\cite{abbas2022_sharpmaml}                         
  & {50.28} &65.04 \\
  \bottomrule
  \end{tabular}
  \end{center}
\end{table}

{We then turn to the more complex benchmark of few-shot image classification on the Mini-Imagenet dataset.}
The results reported in Table~\ref{tab:acc_miniimagenet} highlight that
Sharp-MAML outperforms other meta-learning methods in this setting,
with BMAML generally outperforming other non-Bayesian methods.
For results on other datasets and for further discussion, we refer to~\cite{yin2020_maml_pp,Zhou2019_proxmaml,abbas2022_sharpmaml,nguyen2020_VAMPIRE}.

\section{Modular Meta-Learning}
\label{sec:ch2:modular}

The methods described thus far aim at parametric generalization. In contrast, modular meta-learning aims at fast \textit{combinatorial generalization}. Rather than transferring knowledge across tasks via hyperparameter, modular meta-learning generalizes to new tasks by optimizing a set of reusable neural network modules that can be composed in different ways to solve a new task. By reusing modules across tasks, modular meta-learning makes, in a sense,  ``infinite use of finite means”, and represents a scalable approach towards generalization, particularly in settings which are heavily constrained in terms of data \cite{alet2018modular, alet2019neural, nikoloska2021black}. 

More formally, modular meta-learning assumes a shared module set $\mathcal{M} = [\theta^{(1)}, ..., \theta^{(M)}]$ of size $M$ which is optimized during meta-training. During meta-testing, the module-set is fixed, and a subset of the modules are selected, combined and applied to the new task. This enables an efficient adaptation based on limited data via the selection of modules from the set $\mathcal{M}$. 

Let $S_k (\mathcal{M})$ denote the assignment of a subset of modules from set $\mathcal{M}$ to a particular task $k$. Let also ${\phi}^{(S_k (\mathcal{M}))}$ represent the model obtained by combining the selected modules $S_k (\mathcal{M})$.  The meta-training loss for modular meta-learning problem is given by 
  \begin{subequations}
  	  \begin{align}
    \label{eqn.mod}
  &\mathcal{L}^{\mathrm{mod}}_{\mathcal{D}^{\rm mtr}}\left(\mathcal{M} \right)
  =\frac{1}{K} \sum_{k=1}^{K} L_{\mathcal{D}_{k}^{\mathrm{va}}}\left({\phi}^{\mathrm{mod}}( \mathcal{D}_{k}^{\mathrm{tr}} |  \mathcal{M}) \right) \\
  \label{eq:pertask_para_mod}
  &\mathrm{s.t.}~~ {\phi}^{\mathrm{mod}}( \mathcal{D}_{k}^{\mathrm{tr}} |  \mathcal{M}) \,\, = \,\, \underset{S_k (\mathcal{M})}{\text{arg min}} \,\, L_{\mathcal{D}_{k}^{\mathrm{va}}}\left({\phi}^{(S_k (\mathcal{M}))} \right).
  \end{align}
  \end{subequations}
The inner optimization in \eqref{eq:pertask_para_mod} selects the module set for task $k$, while the outer problem \eqref{eq:pertask_para_mod} optimizes over the module set $\mathcal{M}$. The outer problem in \eqref{eqn.mod} is typically tackled by gradient descent, while the optimization of the assignment in the inner problem \eqref{eq:pertask_para_mod} is a discrete optimization problem. Previous works have addressed this problem by adopting combinatorial optimization techniques like simulated annealing \cite{alet2018modular, alet2019neural}, or using reparametrization and gradient descent \cite{nikoloska2021black}.

Modular meta-learning is summarized in Algorithm \ref{alg:modular_train}.

  \begin{algorithm}[t]
    \caption{Modular Meta-learning}
    \label{alg:modular_train}
  \begin{algorithmic}[1]
    \State{\textbf{Input:}~
    Initial module set $\mathcal{M}$;
    meta-training data  $\mathcal{D}$;
    loss function $\ell(z  |  \phi)$;
    stepsizes $\alpha, \beta$}
    \While{not converged}
    \State{Sample batch of tasks $\tilde{\mathcal{K}} \subseteq \mathcal{K} = \{1,\ldots, K\}$}
      \For{all $k \in \tilde{\mathcal{K}}$}
        \State{{Compute assignment parameters  $S_k (\mathcal{M})$ via problem~\eqref{eq:pertask_para_mod} }}
        \State{{Compute shared module parameters $\mathcal{M}$  via problem~\eqref{eqn.mod} }}
      \EndFor
    \EndWhile
    \end{algorithmic}
  \end{algorithm}

\section{Model-Based Meta-Learning}\label{sec:ch2:cavia}
As an example of model-based meta-learning, in this section, we review
the  Context Adaptation Via Meta-Learning (CAVIA) algorithm introduced in~\cite{zintgraf2019_cavia}.
Unlike optimization-based schemes, the model is shared across all tasks, and not adapted based on training data from each task.
Therefore, the model parameter vector can be considered to be the hyperparameter $\theta$ shared across tasks.
What is adapted to each task is a \textbf{context parameter} $\phi$
that serves as an additional input vector to the model as illustrated in Figure~\ref{fig:cavia}.
The rationale for this choice is that vector $\phi$ can embed information about the task that can control the { output} of the model.
\begin{figure}[ht]
  \centering
  \includegraphics[width=.4\textwidth]{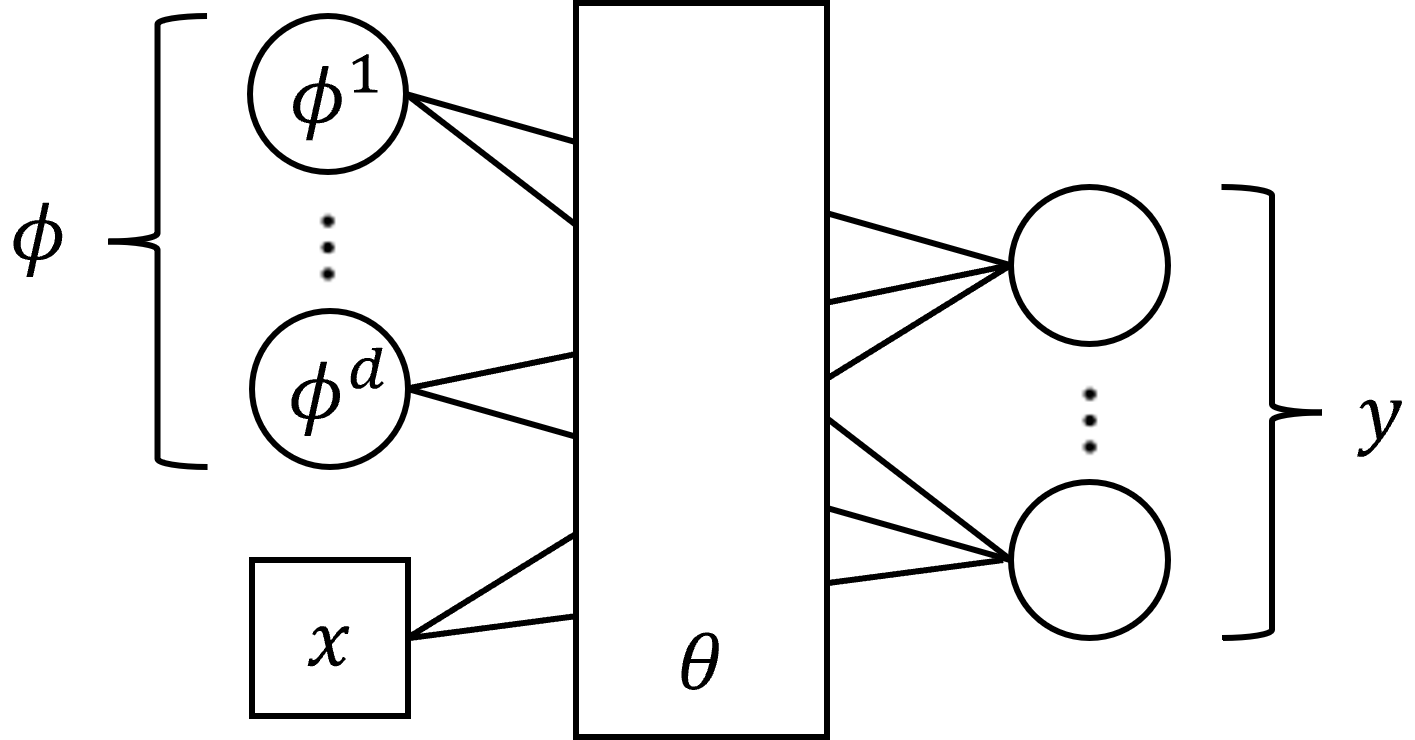}
  \caption{Illustration of CAVIA: 
  While the model parameter vector $\theta$ are shared across tasks,
  the per-task context parameter vector $\phi$, consisting of entries $\phi^1 \dots, \phi^d$, is adapted to the training data of each task, and it
  serves as an additional input vector to the model.}
  \label{fig:cavia}
\end{figure}
Let us define as $L_{\mathcal{D}_k^{\rm tr}}(\theta,\phi)$ the training loss for task $k$ given model parameter $\theta$ and context vector $\phi$.
By reducing the number of parameters to be updated, CAVIA can be more sample efficient than optimization-based scheme.
The meta-training loss function $\mathcal{L}^{\mathrm{ca}}_{\mathcal{D}^{\rm mtr}}\left(\theta\right)$ is  given by
  \begin{subequations}
  	\begin{align}
	&\mathcal{L}^{\mathrm{ca}}_{\mathcal{D}^{\rm mtr}}\left(\theta\right)=\frac{1}{K} \sum_{k=1}^{K} L_{\mathcal{D}_{k}^{\mathrm{va}}}\left(\theta, {\phi}^{\mathrm{ca}}( \mathcal{D}_{k}^{\mathrm{tr}}  |  \theta) \right) \\
	\label{eq:pertask_para_ca}
	&\mathrm{s.t.}~~ {\phi}^{\mathrm{ca}}( \mathcal{D}_{k}^{\mathrm{tr}}  |  \theta)=\phi_0 -\alpha \nabla_{\phi_0} L_{\mathcal{D}_{k}^{\mathrm{tr}}}\left(\theta, \phi_0 \right),
\end{align}
  \end{subequations}
where $\phi_0$ is some fixed initialization, e.g., the all-zero vector.
Setting $\phi_0 = 0$ and 
using the chain rule of differentiation,
the update of the hyperparameter $\theta$ during training procedure of CAVIA is given by  
\begin{align}
	\label{eq:outer_update_ca}
	\theta \leftarrow 
  &\theta - \beta \nabla_{\theta} \mathcal{L}_{\mathcal{D}^{\rm mtr}}^{\rm ca}(\theta) 
  \nonumber \\
  =&\theta - \frac{\beta}{K} \sum_{k=1}^{K} \nabla_{\theta} L_{\mathcal{D}_{k}^{\mathrm{va}}}\left(\theta, {\phi}^{\mathrm{ca}} \right)| _{{\phi}^{\mathrm{ca}} = {\phi}^{\mathrm{ca}}( \mathcal{D}_{k}^{\mathrm{tr}}  |  \theta)} 
  \nonumber \\
  &+\frac{\alpha\beta}{K} \sum_{k=1}^{K} \nabla_{\theta \phi_0 }^2 L_{\mathcal{D}_{k}^{\mathrm{va}}}(\theta, {\phi}_{0})
  \nabla_{{\phi}} L_{\mathcal{D}_{k}^{\mathrm{va}}}\left(\theta, {\phi} \right) | _{{\phi} = {\phi}^{\mathrm{ca}}( \mathcal{D}_{k}^{\mathrm{tr}}  |  \theta)}
  .
\end{align}
The CAVIA algorithm is summarized in Algorithm \ref{alg:cavia_train}.
\begin{algorithm}[H]
	\caption{CAVIA}
	\label{alg:cavia_train}
	\begin{algorithmic}[1]
		\State{\textbf{Input:}~
			Initial iterate $\theta$;
			meta-training data $\mathcal{D}$;
			loss function $\ell(z  |  \phi)$;
			stepsize $\beta$; regularization weight $\lambda$}
      \While{not converged}
      \State{Sample batch of tasks $\tilde{\mathcal{K}} \subseteq \mathcal{K} = \{1,\ldots, K\}$}
          \For{all $k \in \tilde{\mathcal{K}}$}
		\State{{Compute per-task parameter ${\phi}^{\rm ca}( \mathcal{D}_{k}^{\mathrm{tr}}  |  \theta)$ by solving~\eqref{eq:pertask_para_ca}}}
		\EndFor
		\State{{Update hyperparameter vector  $\theta$ via the gradient update \eqref{eq:outer_update_ca}}}
		\EndWhile
	\end{algorithmic}
\end{algorithm}

\section{Conclusions}
In this section, 
we have provided an overview of meta-learning algorithms by mostly focusing on 
optimization-based strategies.
We have categorized optimization-based algorithms into second-order and first-order algorithms based on whether they require  second-order derivatives during meta-training.
All algorithms were formulated as solutions to bilevel optimization problems, which follows a generic form  as
\begin{subequations}\label{eq:opt0-bi}
	\begin{align}
	&{\cal L}_{\mathcal{D}^{\rm mtr}}(\theta)
  =\frac{1}{K} \sum_{k=1}^{K} L_{\mathcal{D}_{k}^{\mathrm{va}}}\left({\phi}(\mathcal{D}_{k}^{\mathrm{tr}} |  \theta) \right) \\
	&~{\rm s. t.}~~
  {\phi}(\mathcal{D}_k^{\rm tr}|\theta)= \argmin_{\phi\in \mathbb{R}^{d}}~
  \tilde{L}_{{\cal D}_k^{\rm tr}}(\theta, \phi ), \label{eq:opt0-inner}
\end{align} 
\end{subequations}
where the lower-level function $\tilde{L}_{{\cal D}_k^{\rm tr}}(\theta, \phi )$ can be different from the upper-level function $L_{\mathcal{D}_{k}^{\mathrm{va}}}(\cdot)$ and it depends on both $\theta$ and $\phi$. 
Different meta-learning algorithms introduced in this section mainly differ in the corresponding inner-level problem~\eqref{eq:opt0-inner}.
In the next section, we will elaborate on the unifying perspective of meta-learning as a bilevel optimization problem, and review results on the convergence of gradient-based bilevel optimization algorithms for such problems.



%% file: Chapter3/Chapter3.tex
\chapter{Bilevel Optimization for Meta-Learning}\label{sec:ch3:bilevel-opt}

In the previous sections, we have reviewed the meta-learning setup and the main meta-learning algorithms. 
In this section, we take a unified view to describe the operation of meta-learning algorithms through the lens of \emph{bilevel optimization}.

\section{A Brief Introduction to Bilevel Optimization}
Stochastic optimization methods, including stochastic gradient descent (SGD) \cite{robbins1951} are prevalent for solving large-scale machine learning problems. Plain-vanilla SGD is applicable to stochastic optimization problems such as  empirical risk minimization, which underlies conventional learning. As we have seen in Section~\ref{sec:ch2}, most meta-learning algorithms go beyond the single-level minimization structure of conventional learning by adopting nested formulations based on bilevel optimization  \cite{stackelberg}. 
In this section, we  review a unified bilevel optimization framework to describe meta-learning algorithms. We start this subsection by presenting a brief history of bilevel optimization, as well as by introducing its mathematical formulation. 

\begin{figure}[t]
  \centering
  \includegraphics[width=.9\textwidth]{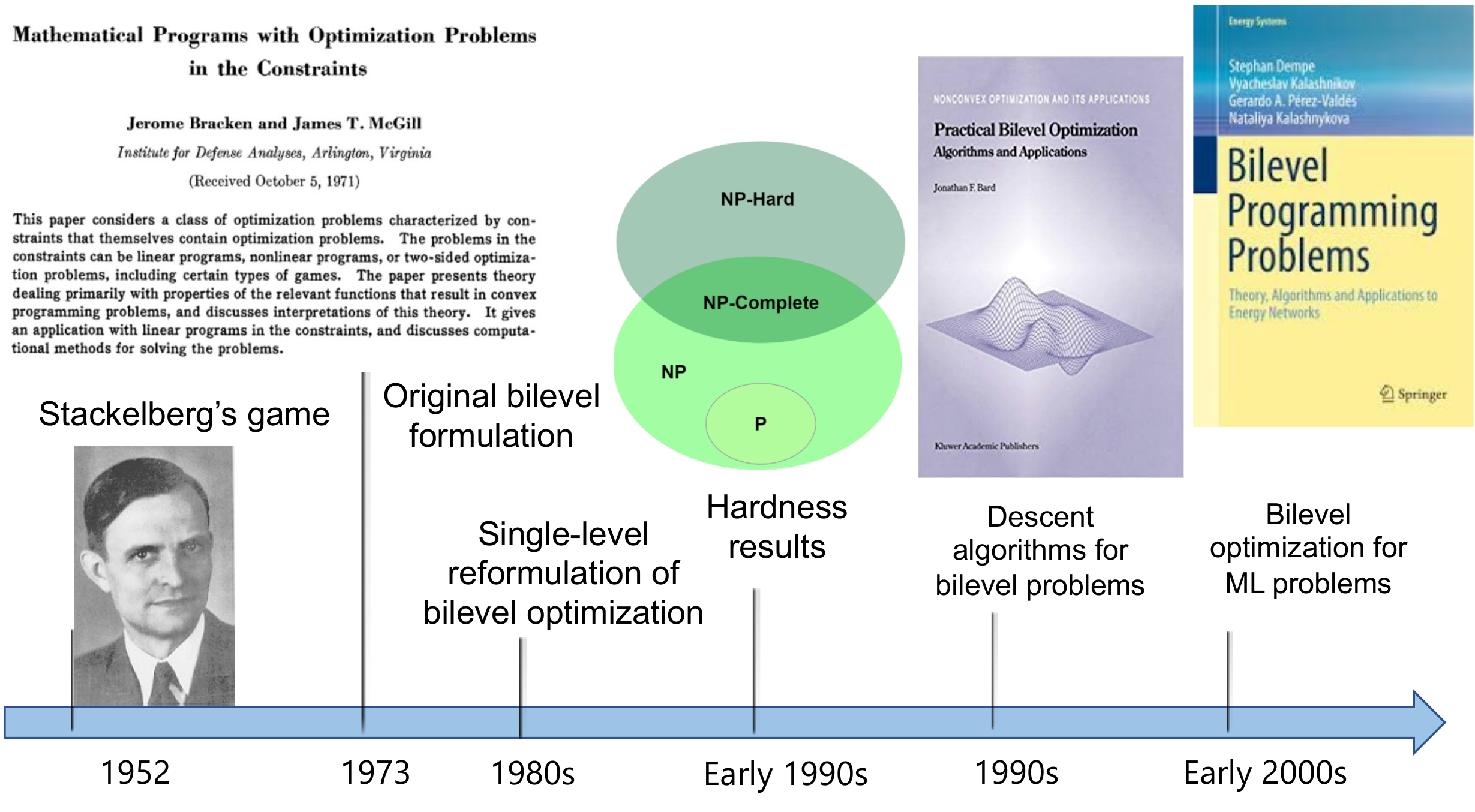}
  \caption{A brief history of bilevel optimization with covers from \cite{stackelbergimage,bracken1973mathematical,bard2013practical,dempe2015bilevel}.}
  \label{fig:BLO_history}
\end{figure}

\subsection{History of Bilevel Optimization}
\textbf{Bilevel optimization (BLO)} is a hierarchical optimization framework, whereby the set of solutions of the lower-level problem serves as a constraint for the upper-level problem \citep{ye1995optimality,colson2007overview}. 
It can be viewed as a generalization of two-stage stochastic programming \citep{shapiro2009}, in which the upper-level objective function depends on the optimal lower-level objective value rather than on the lower-level solution set. 
As illustrated in Figure \ref{fig:BLO_history}, BLO has a long history in operations research, which dates back to von Stackelberg's seminal work on leader-follower games in the 1950s \cite{stackelberg}. 
Research interest on BLO has intensified since the 1970s \cite{bracken1973mathematical}, with  researchers soon realizing that BLO is very challenging: Even an ``easy'' class of linear BLO problems is strongly NP-hard \cite{vicente1994bilevel}. 

Recently, bilevel optimization has gained growing popularity in
a number of machine learning applications such as
meta-learning \cite{rajeswaran2019_imaml}, reinforcement learning \cite{konda1999actor}, continual learning \cite{borsos2020coresets}, and image processing \cite{kunisch2013bilevel}. 
Many recent efforts have been made to address bilevel optimization problems.
One successful approach is to reformulate the bilevel problem as a single-level problem by replacing the lower-level problem by its optimality conditions~\citep{colson2007overview,kunapuli2008}, which belongs to the general class of mathematical programs with equilibrium constraints \cite{luo1996mathematical}. Recently, gradient-based methods for bilevel optimization have gained popularity, whereby the (stochastic) gradient of the upper-level problem is iteratively approximated  \citep{pedregosa2016hyperparameter,sabach2017jopt,franceschi2018bilevel,shaban2019truncated,grazzi2020iteration,ghadimi2018bilevel,hong2020ac,ji2020provably,chen2021single,khanduri2021momentum,guo2021randomized,yang2021provably}; see also two recent surveys \cite{dempe2020bilevel,liu2021investigating}.

\subsection{Generic Formulation}
Bilevel optimization problems of interest for meta-learning can be expressed in the form of the \textbf{stochastic bilevel problem} \cite{ghadimi2018bilevel,hong2020ac,chen2021solving,chen2021single,ji2020feb}
\begin{subequations}\label{opt0}
	\begin{align}
	&\min_{\theta\in \mathbb{R}^d}~~~{\cal L}(\theta):=\mathbb{E}_{\xi}\left[f\left(\theta, \phi^*(\theta);\xi\right)\right]
  ~~~~~~~~~~~~~~~{\rm\sf (upper)} \\
	&~{\rm s. t.}~~~~~\phi^*(\theta) = 
  \argmin_{\phi\in \mathbb{R}^{\hat{d}}}~\mathbb{E}_{\hat{\xi}}
  [g(\theta, \phi;\hat{\xi})]~~~~~~~~~~~{\rm\sf (lower)}, \label{opt0-low}
\end{align} 
\end{subequations}
where $f(\theta, \phi; \xi)$ and $g(\theta, \phi;\hat{\xi})$ are differentiable but possibly nonconvex functions of $\theta$ and $\phi$; and $\xi$ and $\hat{\xi}$ are random variables with given distributions $P(\xi)$ and $P(\hat{\xi})$, respectively.  
In \eqref{opt0}, the upper-level optimization problem over the upper-level variable $\theta\in  \mathbb{R}^d$ depends on the solution $\phi^*(\theta)$ of the lower-level optimization over vector $\phi\in  \mathbb{R}^{\hat{d}}$. Crucially, the solution of the lower-level problem,  $\phi^*(\theta)$, depends on the upper-level variable $\theta$ through the lower-level objective function $g(\theta, \phi;\hat{\xi})$. In the following, for convenience, we define the deterministic functions $g(\theta, \phi):=\mathbb{E}_{\hat{\xi}}[g(\theta, \phi;\hat{\xi})]$ and $f(\theta, \phi):=\mathbb{E}_{\xi}[f(\theta, \phi;\xi)]$.

Many meta-learning problems reviewed in Section 2  can be formulated as the stochastic bilevel problem \eqref{opt0}. 
For example, we can recover the iMAML formulation in \eqref{eqn.imaml} by defining the vector $\phi^*(\theta):=[\phi_1^*(\theta)^{\top}, \cdots, \phi_K^*(\theta)^{\top}]^{\top}$, $\xi:=[\xi_1,\cdots, \xi_K]^{\top}$, $\hat{\xi}:=[\hat{\xi}_1, \cdots, \hat{\xi}_K]^{\top}$, with $\mathcal{D}_{k}^{\mathrm{va}}=\xi_k$, $\mathcal{D}_{k}^{\mathrm{tr}}=\hat{\xi}_k$, and with the upper- and lower-level functions $f(\theta, \phi; \xi)$ and $g(\theta, \phi;\hat{\xi})$ as \cite{rajeswaran2019_imaml,hu2020biased,chen2021solving,ji2022theoretical}
	\begin{equation}
	 f\left(\theta, \phi^*(\theta); \xi\right):=\frac{1}{K} \sum_{k=1}^{K} 
   L_{\mathcal{D}_{k}^{\mathrm{va}}}\left(\phi^*_{k}\left(\theta\right) \right) \label{eq:outer}
\end{equation} 
and $g(\theta, \phi;\hat{\xi}):=\frac{1}{K}\sum_{k=1}^{K}g_k(\theta, \phi_k;\hat{\xi}_k)$, where we have
\begin{equation}
  g_k(\theta, \phi_k;\hat{\xi}):= L_{\mathcal{D}_{k}^{\mathrm{tr}}}\left(\phi_k  \right) + \frac{\lambda}{2}\|\phi_k-\theta\|^2. \label{eq:inner}
\end{equation}

The goals of the rest of this section are to provide
a unified bilevel optimization algorithm for meta-learning that addresses problem  \eqref{opt0}, and 
   to review the convergence properties of the unified bilevel algorithm.

\section{A Unified Bilevel Optimization Framework}
In this section, we introduce a unified algorithmic framework for solving the bilevel problem \eqref{opt0}, and we discuss its connection to some of the meta-learning algorithms reviewed in Section \ref{sec:ch2}. 

\subsection{Bilevel SGD: Definition and Challenges}
Solving bilevel stochastic problems via traditional stochastic optimization techniques faces a number of challenges. In this subsection, we highlight the technical issues that arise when applying SGD directly to the bilevel problem \eqref{opt0}. 

To address the bilevel problem \eqref{opt0}, a natural solution is to apply \emph{alternating SGD updates} on the vectors $\theta$ and $\phi$ based on their respective stochastic gradients as
\begin{equation}\label{eq.sgd2}
	 \phi^{i+1} =\phi^i -\beta^i h_g^i ~~~~~~{\rm and}~~~~~~\theta^{i+1}=\theta^i - \alpha^i h_f^i,
\end{equation}
where $h_g^i$ is an unbiased stochastic gradient for the lower-level objective $g(\theta, \phi)$ at the iterate $(\theta, \phi)=(\theta^i, \phi^i)$; $h_f^i$ is the (possibly biased) stochastic gradient for the upper-level objective ${\cal L}(\theta)$ at $\theta=\theta^i$; and, $\beta^i$ and $\alpha^i$ are stepsizes. 
More precisely, the updates in \eqref{eq.sgd2} are typically run in a way that alternate between the upper- and lower-level problems. 

A first approach is to run SGD updates on the lower-level variable $\phi^i$ in \eqref{eq.sgd2} multiple times before updating the upper-level variable $\theta^i$, which yields a double-loop algorithm. To guarantee convergence, this approach typically requires either increasing number of lower-level $\phi$-update, or growing the batch size used to estimate the gradient $h_g^i$ \cite{ghadimi2018bilevel,ji2020provably}. 
The second method is to update vector $\phi^i$ with a larger learning rate so that the iterates $\theta^i$ are relatively static with respect to $\phi^i$. This can be done by setting learning rates that satisfy the limit $\lim_{i\rightarrow \infty}\alpha^i/\beta^i=0$ \cite{hong2020ac}. 
The third method is to modify the update direction $h_g^i$  by incorporating additional momentum and acceleration terms \cite{chen2021single,khanduri2021momentum,guo2021randomized,yang2021provably,huang2021biadam}.  

The challenge of running the iteration \eqref{eq.sgd2} in one of the ways described above is that the (stochastic) gradient $h_f^i$ for the upper-level variable $\theta$ is often prohibitively expensive to compute. 
To illustrate this point, we now derive the gradient of the upper-level function ${\cal L}(\theta)$ in \eqref{opt0}. 
To this end, we first define the Hessian matrix $\nabla_{\phi\phi}^2 g\big(\theta, \phi\big)$ of function $g(\theta, \phi)$ with respect to $\phi$ as 
\begin{align*}
  \nabla_{\phi\phi}^2g\big(\theta, \phi\big) :=    \begin{bmatrix}
    \frac{\partial^2}{\partial \theta_1\partial \theta_1 }g\big(\theta, \phi\big) & \cdots & \frac{\partial^2}{\partial \theta_1\partial \theta_{\hat{d}}}g\big(\theta, \phi\big)\\
    & \cdots & \\
    \frac{\partial^2}{\partial \theta_{\hat{d}}\partial \theta_1 }g\big(\theta, \phi\big) & \cdots & \frac{\partial^2}{ \partial \theta_{\hat{d}}\partial \theta_{\hat{d}}}g\big(\theta, \phi\big)
  \end{bmatrix}
\end{align*}
as well as the matrix $  \nabla_{\theta\phi}^2g\big(\theta, \phi\big)$ as
\begin{align*}
  \nabla_{\theta\phi}^2g\big(\theta, \phi\big) :=    \begin{bmatrix}
    \frac{\partial^2}{\partial \theta_1\partial \phi_1 }g\big(\theta, \phi\big) & \cdots & \frac{\partial^2}{\partial \theta_1\partial \phi_{\hat{d}}}g\big(\theta, \phi\big)\\
    & \cdots & \\
    \frac{\partial^2}{\partial \theta_d\partial \phi_1 }g\big(\theta, \phi\big) & \cdots & \frac{\partial^2}{ \partial \theta_d\partial \phi_{\hat{d}}}g\big(\theta, \phi\big)
  \end{bmatrix}.
\end{align*}
Under certain differentiability assumptions of the upper and lower-level functions, the gradient $\nabla {\cal L}(\theta)$ is obtained as \cite{ghadimi2018bilevel}
	\begin{align}\label{grad-deter-2}
 \nabla {\cal L}(\theta)=&\nabla_{\theta}f(\theta, \phi^*(\theta))\nonumber\\
 &-\nabla_{\theta\phi}^2g(\theta, \phi^*(\theta))\!\left[\nabla_{\phi\phi}^2 g(\theta, \phi^*(\theta))\right]^{-1}\nabla_{\phi} f(\theta, \phi^*(\theta)).
\end{align}

By \eqref{grad-deter-2}, evaluating an unbiased stochastic estimate of the gradient $\nabla {\cal L}(\theta)$ faces the following main difficulties: 
\begin{enumerate}
	\item [$\bullet$]  The gradient $\nabla {\cal L}(\theta)$ depends on the solution of the lower-level problem $\phi^*(\theta)$, which is estimated via SGD in  \eqref{eq.sgd2} and hence varies across the iterations (see Figure \ref{fig:BLO_drift}); 
\item [$\bullet$]  The gradient $\nabla {\cal L}(\theta)$ requires the second derivatives $\nabla_{\theta\phi}^2g(\theta, \phi)$ and $\nabla_{\phi\phi}^2g(\theta, \phi)$ of the lower-level objective function $g$. 
\item [$\bullet$] 
An unbiased estimate of the gradient $\nabla {\cal L}(\theta)$ cannot be obtained via the empirical average over functions $g(\theta, \phi;\hat{\xi})$ with samples $\hat{\xi}\sim P(\hat{\xi})$ due to the nonlinear term $[\nabla_{\phi\phi}^2 g(\theta, \phi^*(\theta))]^{-1}$. 
\end{enumerate}

 \begin{figure}[t]
  \centering
  \includegraphics[width=.7\textwidth]{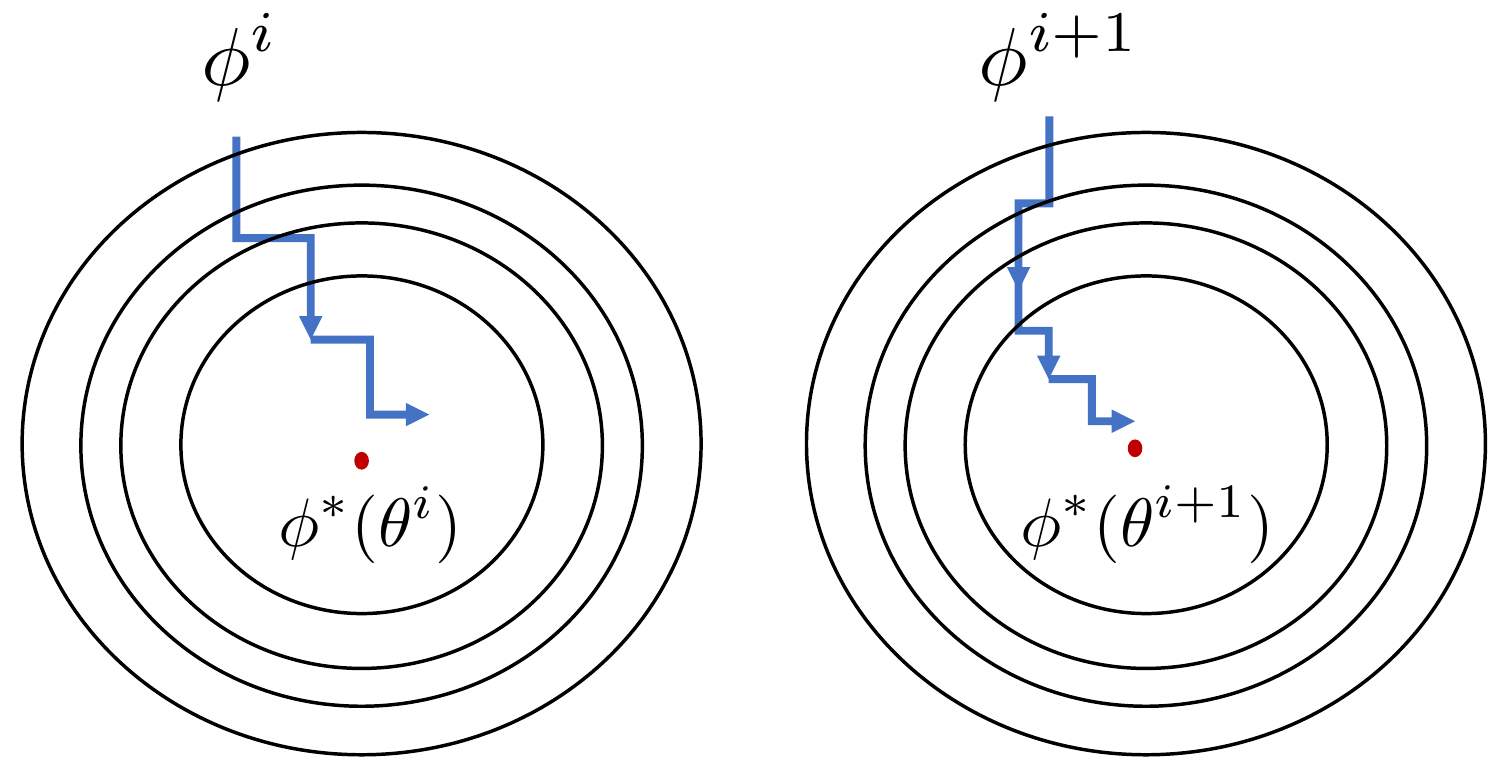}
  \caption{An illustration of minimizers' drift in bilevel SGD.}
  \label{fig:BLO_drift}
\end{figure}

These challenges can be addressed via \emph{implicit-gradient} or \emph{explicit-gradient} methods. 
\textbf{Implicit gradient methods} treat the lower-level solution $\phi^*(\theta)$ as an implicit function of $\theta$, and they directly attempt to evaluate the gradient $\nabla {\cal L}(\theta)$ via the expression \eqref{grad-deter-2}. We will discuss an example of such methods in the next subsection. 
\textbf{Explicit gradient methods} model the optimal lower-level solution $\phi^*(\theta)$ as an explicit function of vector $\theta$. This is typically done by unrolling the iterations of an optimization algorithm such as SGD in \eqref{eq.sgd2}, and by then using the final iteration as a proxy for the lower-level solution $\phi^*(\theta)$ \cite{franceschi2017forward,franceschi2018bilevel}. Explicit gradient methods suffer from the high-memory cost of storing the algorithm's trajectory in the $\phi$-space. In practice, this cost can be controlled by truncating the rolling horizon. 

While these methods deal with bilevel optimization problems with a unique solution for the lower-level problem, recent works have also studied the case in which the lower-level problem may have multiple solutions, which will be further discussed in Section \ref{sec:ch7}.

\subsection{Implicit-Gradient SGD Methods}
In this subsection, we describe a representative implicit-gradient algorithm for the bilevel problem \eqref{opt0}, and then provide a convergence result. 
The algorithm, proposed in \cite{chen2021alset}, is referred to as the ALternating Stochastic gradient dEscenT (ALSET) method.

To overcome the challenge in evaluating the gradient $\nabla {\cal L}(\theta)$ reviewed above, the ALSET algorithm estimates the gradient 
\begin{equation}\label{grad-deter-3}
 \overline{\nabla}_{\theta}f\big(\theta, \phi\big):=\nabla_{\theta}f\big(\theta, \phi\big) -\nabla_{\theta\phi}^2g\big(\theta, \phi\big)\left[\nabla_{\phi\phi}^2 g\big(\theta, \phi\big)\right]^{-1}\nabla_{\phi} f\big(\theta, \phi\big)
\end{equation}
for a fixed value $\phi$. 
Unbiased estimates of the terms $\nabla_{\theta}f\big(\theta, \phi\big)$ and $\nabla_{\phi} f\big(\theta, \phi\big)$ can be obtained by averaging the gradients $\nabla_{\theta}f\big(\theta, \phi; \xi\big)$ and $\nabla_{\phi} f\big(\theta, \phi; \xi\big)$ over one or multiple samples $\xi\sim P(\xi)$. Similarly, an unbiased estimate of the term $\nabla_{\theta\phi}^2g\big(\theta, \phi\big)$ can be obtained by averaging the matrix $\nabla_{\theta\phi}^2g\big(\theta, \phi;\hat{\xi}\big)$ over one or multiple samples $\hat{\xi}\sim P(\hat{\xi})$. 
For the term 
 $\left[\nabla_{\phi\phi}^2 g\big(\theta, \phi\big)\right]^{-1}$, an estimate is evaluated as
\begin{equation}\label{eq.inverse-estimator}
\left[\nabla_{\phi\phi}^2 g\big(\theta, \phi\big)\right]^{-1}\approx \Big[\frac{N}{L_{g}}\prod\limits_{n=1}^{N'}\Big(I-\frac{1}{L_{g}}\nabla_{\phi\phi}^2g(\theta,\phi;\hat{\xi}_{(n)})\Big)\Big],
\end{equation}
where $L_{g}$ is a constant that depends on function $\nabla g(\theta,\phi)$ \cite{chen2021alset}; integer $N'$ is drawn from $\{1, 2, \ldots, N\}$ uniformly at random; and $\{\hat{\xi}^{(1)}, \ldots,\hat{\xi}^{(N')}\}$ are i.i.d. samples from the distribution $P(\hat{\xi})$.
It was shown in \cite{ghadimi2018bilevel} that the bias of the estimate \eqref{eq.inverse-estimator} decreases exponentially with $N$.

At each iteration $k$, ALSET alternates between  stochastic gradient updates on the lower-level vector $\phi^i$ and on the upper-level vector $\theta^i$ by running $T$ steps of SGD on the lower-level variable $\phi^i$ before updating upper-level variable $\theta^i$. With $\alpha^i$ and $\beta^i$ denoting the stepsizes used for the $\theta$- and $\phi$-updates, respectively, the ALSET updates are given as
\begin{subequations}\label{eq.ALSET}
\begin{align}
\!\!\!\!    \phi^{i,t+1}&=\phi^{i,t} \!-\beta^i h_g^{i,t},\,t=0,\ldots, T~~{\rm with}~\phi^{i+1}:=\phi^{i,T} \label{eq.ALSET3}\\
\!\!\!\!    \theta^{i+1}&=\theta^i \!-\! \alpha^i h_f^i,		\label{eq.ALSET2}
\end{align}
\end{subequations}
where index $t$ runs over the inner-loop of $\phi$-updates, while index $k$ runs over the $\theta$-updates. In \eqref{eq.ALSET}, the update direction for vector $\phi$ is the stochastic gradient  
\begin{equation}\label{eq.yg-estimator}
	h_g^{i,t} := \nabla_{\phi}g(\theta^i,\phi^{i,t}; \hat{\xi}^{i,t})
\end{equation}
with $\hat{\xi}^{i,t}$ being i.i.d. samples from distribution $P(\hat{\xi})$; and, with the Hessian inverse estimator \eqref{eq.inverse-estimator}, the update direction of $\theta$ is  given by the biased gradient
\begin{align}\label{eq.biased-estimator}
&h_f^i:=\nabla_{\theta} f(\theta^i,\phi^{i+1};\xi^i)-\nabla_{\theta\phi}^2g(\theta^i,\phi;\hat{\xi}_{(0)}^i)\nonumber\\
&\times\Big[\frac{N}{L_{g,1}}\prod\limits_{n=1}^{N'}\Big(I-\frac{1}{L_{g,1}}\nabla_{\phi\phi}^2g(\theta^i,\phi^{i+1};\hat{\xi}_{(n)}^i)\Big)\Big]\nabla_{\phi}f(\theta^i,\phi^{i+1};\xi^i),
\end{align}
where $\xi^i$ and $\{\hat{\xi}_{(n)}^i\}_{n=0}^{N'}$ are i.i.d. samples from distribution $P(\xi)$. 
Algorithm \ref{alg: ALSET} provides a summary of the ALSET algorithm. 
Similar algorithms include BSA  \citep{ghadimi2018bilevel}, TTSA \citep{hong2020ac} and stocBiO \citep{ji2020provably}. 
We refer to \cite{chen2021alset} for a comparison among these algorithms.  

\begin{algorithm}[H]
\caption{ALSET for the stochastic bilevel problem \eqref{opt0}}\label{alg: ALSET}
\begin{algorithmic}[1]
    \State{\textbf{initialize:} $\theta^0, \phi^0$, stepsizes $\{\alpha^i, \beta^i\}$} 
        \For{$i=0,1,\ldots, I_{\mathrm{max}}-1$}
            \For{$t=0,1,\ldots,T-1$}
                \State{update $\phi^{i,t+1}=\phi^{i,t}-\beta^ih_g^{i,t}$ using \eqref{eq.yg-estimator}}~~~\Comment{set $\phi^{i,0}=\phi^i$}
            \EndFor
            \State{update $\theta^{i+1}=\theta^i-\alpha^i h_f^i$  using \eqref{eq.biased-estimator}}     ~~~~~~~~~       \Comment{set $\phi^{i+1}=\phi^{i,T}$}
        \EndFor
\end{algorithmic}
\end{algorithm}

\subsection{Application to Meta-Learning}
Next we will illustrate how we can recover various meta-learning algorithms introduced in Section 2 as special cases of the ALSET algorithm.

\noindent\textbf{MAML.}
The MAML algorithm in Algorithm \ref{alg:maml_train_0} is recovered by applying ALSET in Algorithm \ref{alg: ALSET} to the following problem
  \begin{align}
  &\min_{\theta}~{\cal L}^{\mathrm{ma}}_{\mathcal{D}^{\rm mtr}}\left(\theta\right)
  \coloneqq \frac{1}{K} \sum_{k=1}^{K} L_{\mathcal{D}_{k}^{\mathrm{va}}}\underbrace{\left({\phi}^{\mathrm{ma}}(\mathcal{D}_k^{\rm tr}|\theta) \right)}_{f_k(\theta, \phi; \xi_k)} \\
  &\mathrm{s.t.}~~ 
  {\phi}^{\mathrm{ma}}(\mathcal{D}_k^{\rm tr}|\theta)
  =\argmin_{\phi} ~ \underbrace{\nabla L_{\mathcal{D}_{k}^{\mathrm{tr}}}(\theta )^{\top}(\phi-\theta)+\frac{1}{2\beta}\|\phi-\theta\|^2}_{g_k(\theta, \phi;\hat{\xi}_k)},~\forall k.\nonumber
  \end{align}
  Note that in this case we have $d=\hat{d}$, and thus the stochastic gradients $\nabla_{\theta}f_k(\theta, \phi; \xi_k)$ and $\nabla_{\phi}f_k(\theta, \phi; \xi_k)$ used in the upper-level gradient \eqref{grad-deter-3} become 
    \begin{equation}
 \!\!   	\nabla_{\theta}f_k(\theta, \phi; \xi_k)=0~~{\rm and}~~\nabla_{\phi}f_k(\theta, \phi; \xi_k)=\nabla_{\phi} L_{\mathcal{D}_{k}^{\mathrm{va}}}\left({\phi}^{\mathrm{ma}}(\mathcal{D}_k^{\rm tr}|\theta) \right),
    \end{equation}
  and the stochastic Hessian and the Jacobian matrices $\nabla_{\phi\phi}^2 g_k\big(\theta, \phi;\hat{\xi}_k\big)$ and $\nabla_{\theta\phi}^2g_k\big(\theta, \phi;\hat{\xi}_k\big)$ used in \eqref{grad-deter-3} reduce to 
  \begin{equation}
  	\nabla_{\phi\phi}^2 g_k(\theta, \phi;\hat{\xi}_k)=\frac{1}{\beta}I,~~\nabla_{\theta\phi}^2 g_k(\theta, \phi;\hat{\xi}_k)=\nabla_{\theta}^{2} L_{\mathcal{D}_{k}^{\mathrm{tr}}}\left(\theta \right)-\frac{1}{\beta}I,
  \end{equation}
  where $I\in \mathbb{R}^{\hat{d}\times \hat{d}}$ is an identity matrix.
  
\noindent\textbf{iMAML.}
We can recover the iMAML algorithm in Algorithm \ref{alg:imaml_train} by applying ALSET in Algorithm \ref{alg: ALSET} to the following problem
  \begin{align}
  &\min_{\theta}~{\cal L}^{\mathrm{im}}_{\mathcal{D}^{\rm mtr}}\left(\theta\right)
  :=\frac{1}{K} \sum_{k=1}^{K} \underbrace{L_{\mathcal{D}_{k}^{\mathrm{va}}}\left({\phi}^{\mathrm{im}}(\mathcal{D}_k^{\rm tr}|\theta) \right)}_{f_k(\theta, \phi; \xi_k)} \\
  &\mathrm{s.t.}~~ 
  {\phi}^{\mathrm{im}}(\mathcal{D}_k^{\rm tr}|\theta)
  =\argmin_{\phi}~  \underbrace{L_{\mathcal{D}_{k}^{\mathrm{tr}}}\left(\phi \right)+\frac{1}{2\beta}\left\|\phi-\theta\right\|^{2}}_{g_k(\theta, \phi;\hat{\xi}_k)}, ~\forall k.\nonumber
  \end{align}

\section{Convergence Analysis for Bilevel Optimization}
In this subsection, we will present a convergence result of ALSET that was established in \cite{chen2021alset}. 
  Given the connection between ALSET and the algorithms in Section 2, performance guarantee for ALSET that we will introduce next will also apply to specific MAML algorithms by using the corresponding upper- and lower-level functions. 
The results rely on the following assumptions, which are common in the bilevel optimization literature \cite{ghadimi2018bilevel,ji2020provably,hong2020ac,guo2021randomized,chen2021alset}.

\begin{assumption}[Lipschitz continuity]\label{chap3-ass1} 
Functions $f(\theta,\phi), \nabla f(\theta,\phi), \nabla g(\theta,\phi)$ and $\nabla^2g(\theta,\phi)$ are Lipschitz continuous with respect to $\theta$ and $\phi$.
\end{assumption}

\begin{assumption}[Strong convexity of $g(\theta,\phi)$ in $\phi$]\label{chap3-ass2} 
For any fixed $\theta$, $g(\theta,\phi)$ is strongly convex in $\phi$.
\end{assumption}

\begin{assumption}[Bias and variance]\label{chap3-ass3} 
The stochastic derivatives $\nabla f(\theta,\phi;\xi)$, $\nabla g(\theta,\phi;\hat{\xi})$, $\nabla^2g(\theta, y, \hat{\xi})$ are unbiased with bounded variances.
\end{assumption}

%

 \begin{figure}[t]
  \centering
  \includegraphics[width=.98\textwidth]{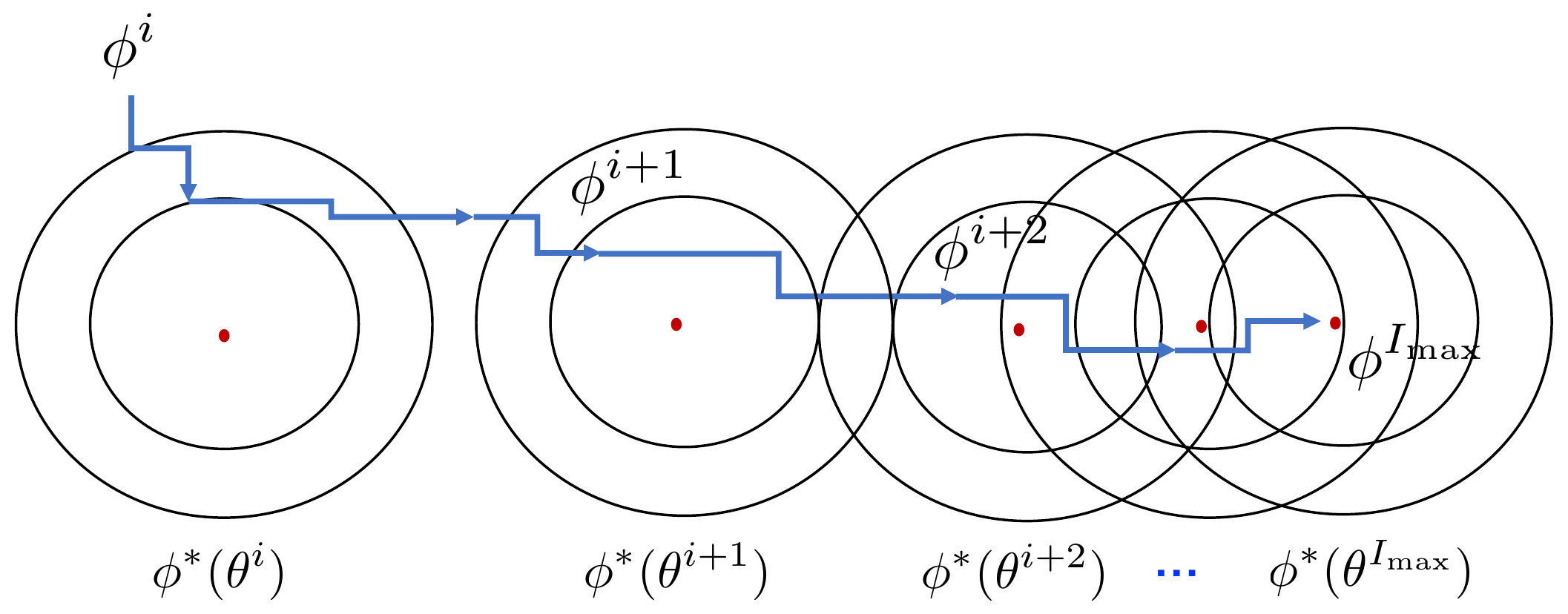}
  \caption{An illustration of vanishing minimizers' drift in ALSET.}
  \label{fig:BLO_drift2}
\end{figure}

\begin{theorem}[Bilevel problems {\cite[Theorem 1]{chen2021alset}}]\label{theorem1}
Suppose Assumptions \ref{chap3-ass1}--\ref{chap3-ass3} hold. With some proper constants $\alpha>0$ and $\beta>0$, choose the upper- and lower-level stepsizes as 
\begin{equation}\label{eq.stepsize}
 \alpha^i = \frac{\alpha}{\sqrt{I_{\mathrm{max}}}} ~~~{\rm and}~~~\beta^i=\frac{\beta}{\sqrt{I_{\mathrm{max}}}},~~~ {\rm for}~i=1,2, \cdots, I_{\mathrm{max}},
\end{equation}
where $I_{\mathrm{max}}$ is the total number of upper-level iterations. 
Set $N={\cal O}(\log I_{\mathrm{max}})$ in the Hessian inversion estimator \eqref{eq.inverse-estimator}.
For any $T\geq 1$, the iterates $\{\theta^i, \phi^i\}$ generated by Algorithm \ref{alg: ALSET} satisfy the condition
\begin{subequations}\label{eq.theorem1}
		\begin{align}
   &\frac{1}{I_{\mathrm{max}}} \sum_{i=1}^{I_{\mathrm{max}}}\EE\left[\left\|\nabla {\cal L}(\theta^i)\right\|^2\right] = {\cal O}\Big(\frac{1}{\sqrt{I_{\mathrm{max}}}}\Big)\\
   &\EE\left[\left\|\phi^{I_{\mathrm{max}}}\!-\phi^*(\theta^{I_{\mathrm{max}}})\right\|^2\right] = {\cal O}\Big(\frac{1}{\sqrt{I_{\mathrm{max}}}}\Big), 
\end{align}
\end{subequations}
where $\phi^*(\theta^{I_{\mathrm{max}}})$ is the minimizer of the lower-level problem in \eqref{opt0-low}.
\end{theorem}

Theorem \ref{theorem1} demonstrates that   the alternating SGD-type algorithm ALSET can  achieve the same convergence rate ${\cal O}\Big(\frac{1}{\sqrt{I_{\mathrm{max}}}}\Big)$ of SGD (see e.g.,\cite{ghadimi2013sgd}). Therefore, the given class of bilevel learning problems can be efficiently solved by ALSET without sacrificing iteration efficiency as compared to the standard single-level learning problems. 
Recent advances improving the above unified result also include relaxing the assumption \cite{shen2022single}, replacing the inner-loop \eqref{eq.biased-estimator} via fully single-loop update \cite{li2022fully}, and allowing online update \cite{tarzanagh2022online}. 

Figure \ref{fig:BLO_drift2} gives some intuition as to why ALSET can preserve the same convergence rate of SGD for single-level learning problems. Specifically, given the decaying stepsizes $\alpha^i={\cal O}\Big(\frac{1}{\sqrt{I_{\mathrm{max}}}}\Big)$ for the upper-level $\theta$-update, the drifts of the lower-level minimizers $\phi^*(\theta)$ tend to vanish with $k$ at the rate of ${\cal O}\Big(\frac{1}{\sqrt{I_{\mathrm{max}}}}\Big)$. As a result, the performance in terms of the meta-loss ${\cal L}(\theta)$ are dominated by the variance of the upper-level $\theta$-gradient, as for the single-level SGD, without introducing additional noise due to the lower-level updates.

 \section{Conclusions}
In this section, we have revisited the bilevel learning framework and its connection to the meta-learning problems. 
We have described a unified ALternating Stochastic gradient dEscenT (ALSET) method for bilevel optimization problems, and connected it to many of the meta-learning algorithms reviewed in Section 2. 
For a certain class of bilevel optimization problems, ALSET requires ${\cal O}(\epsilon^{-2})$ iterations in total to achieve an $\epsilon$-stationary point of the bilevel learning problem. 
This matches the iteration complexity of SGD for  single-level problems. 

%% file: Chapter4/Chapter4.tex
\newcommand \Dtrk {\mathcal{D}^{\mathrm{tr}}_k}
\newcommand \DtrT {\mathcal{D}^{\mathrm{tr}}_T}
\newcommand \Dtr {\mathcal{D}^{\mathrm{tr}}}
\newcommand \KL {\mathrm{KL}}
\newcommand \base {p(\phi|\Dtrk)}
\newcommand \basetheta {p(\phi|\Dtrk,\theta)}
\newcommand \basethetaT {p(\phi|\DtrT,\theta)}
\newcommand \sampledist {p(Z|T_k)}
\newcommand \Ebb {\mathbb{E}}
\newcommand \taskdist {p(T)}
\newcommand \metalearner {p(\theta|\{\Dtrk\}_{k=1}^K)}
\newcommand \metaset {\{\Dtrk\}_{k=1}^K}
\newcommand \taskset {\{T_k\}_{k=1}^K}
\newcommand \Lscr {\mathcal{L}}
\chapter{Statistical Learning Theory for Meta-Learning}
\label{sec:ch4}

While the previous section described meta-learning as an optimization process, this section studies the generalization performance of meta-learning algorithms from a statistical learning-theoretic viewpoint. Generalization of a meta-learning algorithm, also known as \textit{meta-generalization}, refers to the capacity of the algorithm to provide solutions that perform well outside the meta-training data, {i.e.}, for new tasks. Towards this goal, we first introduce basic statistical learning-theoretic concepts for conventional learning in Section~\ref{sec:genbounds_conventional learning}, and then extend the presentation to meta-generalization in Section~\ref{sec:genbounds_metalearning}.  Adopting an information-theoretic approach, Section~\ref{sec:infobounds_metageneralization} presents generic upper bounds on the \textit{expected} generalization error of meta-learning algorithms. The meta-generalization error measures the discrepancy between the losses accrued on meta-trainining and meta-test data sets. In contrast, Section~\ref{sec:PACBayesianbounds_metalearning} is dedicated to \textit{high probability}, so-called \textit{PAC-Bayes}, upper bounds on the meta-generalization error. We end this section with a discussion on information-theoretic analysis of the optimality error, {i.e},  the discrepancy between actual and optimal meta-test losses, of Bayesian meta-learning in Section~\ref{sec:Bayesianmetalearning}.

\section{Generalization Error for Conventional Learning} \label{sec:genbounds_conventional learning}
In this subsection, we study the generalization error incurred in conventional learning that targets a single learning task. Let $T_k$ denote the $k$th task under study.
Task $T_k$ is described by an unknown data distribution $\sampledist$, which generates data samples $Z \sim \sampledist$. Note that the data sample $Z$ can denote a tuple $(X,Y)$ of feature vector $X$ and label $Y$ as in supervised learning, or it can denote unlabelled data as in unsupervised learning problems. We use upper case letters to emphasize that these quantities are treated as random variables in statistical learning theory.

A learning algorithm, also called \textit{base-learner}, observes a training data set $\Dtrk=(Z_1,\hdots,Z_N)$ of $N$ samples generated i.i.d. according to the data distribution $\sampledist$. Assuming that the model class $\mathcal{H}$ is parameterized with model parameter vector $\phi$ taking values in space $\Phi$,  the base-learner uses the observed training data set $\mathcal{D}^{\mathrm{tr}}_k$ to optimize the model parameter vector. The performance of the optimized model parameter $\phi$ on a data sample $Z$ is measured using a positive real-valued loss function $\ell(Z|\phi)$.

Ideally, the goal of the base-learner is to find the model parameter vector that minimizes the \textit{population loss},
\begin{align}
	L_{T_k}(\phi) = \Ebb_{\sampledist}[\ell(Z|\phi)], \label{eq:populationloss_conventional}
\end{align} which is the average loss incurred on a test data point $Z \sim \sampledist$ drawn randomly from the data distribution $\sampledist$. In \eqref{eq:populationloss_conventional} and throughout this section, we use $\Ebb_{\bullet}$ to denote the expectation taken over the distribution $\bullet$ in the subscript.  However, the population loss in \eqref{eq:populationloss_conventional} cannot be computed, since the underlying data distribution $\sampledist$ is unknown. Instead,  the base-learner uses the  \textit{training loss},
\begin{align}
	L_{\Dtrk}(\phi) = \frac{1}{N} \sum_{j=1}^N \ell( Z_j|\phi) \label{eq:trainingloss_conventional},
\end{align} which is the empirical average loss incurred on the training data set $\Dtrk=\{Z_j\}_{j=1}^N$.

The difference between the population loss and the training loss is the \textit{generalization error},
\begin{align}
	\Delta L_{k}(\phi) = L_{T_k}(\phi) - L_{\Dtrk}(\phi) \label{eq:genloss_conventionalearning},
\end{align} which is a measure of how well the empirical training loss approximates the population loss. If the learning algorithm producing model parameter vector $\phi$ overfits the training data, and hence the training loss is close to zero, the trained model $\phi$ may not perform well on the unseen test data, thereby resulting in large population loss, and thus in a large generalization error.
Therefore, understanding the generalization error of a learning algorithm can help diagnose and quantify problems with the test performance of a trained model.

Of central interest in statistical learning theory is the problem of understanding and quantifying the generalization capacity of learning algorithms. This is typically accomplished by studying upper bounds on the generalization error \eqref{eq:genloss_conventionalearning}. Traditional bounds hold uniformly with high probability for all models in the model class $\mathcal{H}$, and are referred to as  \textit{ probably approximately correct (PAC) bounds}. These bounds hold with high probability with respect to any random distribution $p(Z|T_k)$ of the training data, and they quantify the generalization error as a function of the ``complexity'' of the model, in a manner that is agnostic to the true data distribution $p(Z|T_k)$. The model complexity is captured via properties of the model class $\mathcal{H}$  such as the Vapnik-Chervonenkis (VC) dimension \cite{blumer1989learnability} or the Rademacher complexity \cite{bousquet2003new}. PAC bounds demonstrate that highly complex models tend to overfit, {i.e.}, to yield large generalization errors \eqref{eq:genloss_conventionalearning}, when trained on few data samples.

The above insights obtained from PAC bounds, however, fail to 
explain the exceptional generalization performance of highly complex deep neural network models.  A major reason for the failure of PAC bounds is attributed to the fact that they ignore the fit of the model class to the specific data distribution, as well as the properties of training algorithms such as SGD.

 \textit{PAC-Bayes theory} also obtains high-probability bounds on the generalization error, but PAC-Bayes bounds are functions of the training algorithm, which is modelled as a random transformation \cite{alquier2021user}. Finally, \textit{ information-theoretic bounds} have been introduced  to quantify the \textit{average } generalization error, and they account for the properties of the learning algorithm,  data distribution, as well as the specific loss function (see \cite{simeone2020maml} for an introduction).

In the rest of this subsection, we first review information-theoretic bounds and then we present PAC-Bayes bounds, which are then extended to meta-learning in the following subsections.
\subsection{Information-Theoretic Generalization Bounds}\label{sec:infobounds_conventionallearning}
In the PAC-Bayes and information-theoretic approaches to the study of the generalization error, a  base-learner is modelled via a conditional distribution $\base$, which in turn describes a stochastic mapping from  training data $\Dtrk$  to  model parameters $\phi$. Examples of stochastic learning algorithms include SGD and its variants; as well as Bayesian, sampling-based, schemes such as stochastic gradient Langevin dynamics (SGLD) \cite{rabinovich2015variational}, \cite{ simeoneCUP}. Given the randomness of training data,  as well as the learning algorithm,  the information-theoretic framework aims to obtain upper bounds on the \textbf{absolute average generalization error},
\begin{align}
	|\mathbb{E}_{p(\Dtrk,\phi)} [\Delta L_{k}(\phi)]|,\label{eq:avgabsgenerror}
\end{align} where the expectation is taken with respect to the joint distribution \begin{align} p(\Dtrk,\phi)=p(\Dtrk)\base \label{eq:jointdistribution} \end{align} of training data and model parameter, with   $p(\Dtrk)=\prod_{j=1}^N p(Z_j|T_k)$. 

Under appropriate assumption on the loss function $\ell(Z|\phi)$, the analysis in \cite{xu2017information} gives an upper bound on the absolute average generalization error in \eqref{eq:avgabsgenerror}  as a function of the mutual information (MI), $I(\phi;\Dtrk)$, between the model parameter vector $\phi$ and the training data $\Dtrk$, and of the number $N$ of training data samples. For any two jointly distributed random variables $A$ and $B$ with the distribution $p(A,B)$, and corresponding marginal distributions $p(A)$ and $p(B)$, the MI   \begin{align}I(A;B)=\Ebb_{p(A,B)}\biggl[\log \frac{p(A,B)}{p(A)p(B)}\biggr]\end{align} is a measure of statistical dependence between $A$ and $B$. We first state the main technical assumption, and then give the main result.
\begin{assumption}\label{assum:convlearning}
The loss function $\ell(Z|\phi)$ is $\sigma^2$-sub-Gaussian\footnote{A random variable $X \sim p(X)$ is said to be $\sigma^2$-sub-Gaussian if the inequality $\log \Ebb_{p(X)}[\exp(\lambda (X-\Ebb_{p(X)}[X]))]\leq \frac{\lambda^2 \sigma^2}{2}$ holds for all $\lambda \in \mathbb{R}$.} with respect to the data distribution $Z \sim \sampledist$ for all model parameters $\phi \in \Phi$.
\end{assumption}
\begin{theorem}\label{thm:infobound_convlearning}
Under Assumption~\ref{assum:convlearning}, the following upper bound on the absolute average generalization error holds
	\begin{align}
		|\mathbb{E}_{p(\Dtrk,\phi)} [\Delta L_{k}(\phi)]| \leq \sqrt{\frac{2\sigma^2}{N} I(\phi;\Dtrk)} \label{eq:infoupperbound},
	\end{align}where
	$
	I(\phi;\Dtrk) 
	$ is the mutual information under the joint distribution $p(\phi,\Dtrk)$ defined in \eqref{eq:jointdistribution}.
\end{theorem}
\begin{proof}
The proof of \eqref{eq:infoupperbound} starts by noting the equivalent representation of average generalization error in \eqref{eq:avgabsgenerror} given by
\begin{align}
\mathbb{E}_{p(\Dtrk,\phi)}[\Delta L_{k}(\phi)]= \mathbb{E}_{p(\Dtrk)p(\phi)}[ L_{\Dtrk}(\phi)] -\mathbb{E}_{p(\Dtrk,\phi)} [ L_{\Dtrk}(\phi)] \label{eq:proof1_1}.
\end{align} The equality in \eqref{eq:proof1_1} holds since the first term in the right-hand side of \eqref{eq:proof1_1} equals the average population loss $\Ebb_{p(\phi)}[L_{T_k}(\phi)]$. In fact, the population loss $L_{T_k}(\phi)$ can be written as the expectation of the training loss $L_{\Dtrk}(\phi)$ over the training data distribution $p(\Dtrk)$, {i.e.}, as $L_{T_k}(\phi)=\Ebb_{p(\Dtrk)}[L_{\Dtrk}(\phi)]$, for any fixed model parameter $\phi$. 

Let us define as $\mathrm{D}_{\KL}(p(x) \lVert q(x))=\Ebb_{p(x)}\Bigl[\log \frac{p(x)}{q(x)}\Bigr]$ the Kullback-Leibler (KL) divergence between the distributions $p(x)$ and $q(x)$.
The key ingredient required to upper bound \eqref{eq:proof1_1} is the Donsker-Varadhan (DV) change-of-measure lemma,  which gives the following inequality (see, {e.g.}, \cite{jose2021free})
\begin{align}
  \mathrm{D}_{\mathrm{KL}}(p(X)|| q(X)) \geq \mathbb{E}_{p(X)}[f(X)]-\log \mathbb{E}_{q(X)}[\exp(f(X))]  \label{eq:DVI},
\end{align}which holds for any bounded, measurable function $f(X)$.

 In \eqref{eq:DVI}, set $X=\Dtrk$, $f(X)=\lambda L_{\Dtrk}(\phi)$, where $\lambda \in \mathbb{R}$, $p(X)=p(\Dtrk|\phi)$, and $q(X)=p(\Dtrk)$  to get the inequality
\begin{align}
  \mathrm{D}_{\mathrm{KL}}(p(\Dtrk|\phi)|| p(\Dtrk))&\geq    \mathbb{E}_{p(\Dtrk|\phi)}[\lambda L_{\Dtrk}(\phi)]-\log \mathbb{E}_{p(\Dtrk)}\biggl[\exp(\lambda L_{\Dtrk}(\phi))\biggr]  \nonumber \\
  & \geq \mathbb{E}_{p(\Dtrk|\phi)}[\lambda L_{\Dtrk}(\phi)]-\mathbb{E}_{p(\Dtrk)}[\lambda L_{\Dtrk}(\phi)]-\frac{\lambda^2 \sigma^2}{2N} \label{eq:proof1_2}.
\end{align} The inequality in \eqref{eq:proof1_2} follows from Assumption~\ref{assum:convlearning} and from the fact that the training set $\Dtrk$ consists of i.i.d. data samples. Taking the average over $\phi \sim p(\phi)$ on both sides of \eqref{eq:proof1_2}  yields the inequality
\begin{align}
 I(\phi;\Dtrk) \geq -\lambda \mathbb{E}_{p(\Dtrk,\phi)}[ \Delta L_{k}(\phi)] -  \frac{\lambda^2 \sigma^2}{2N},\label{eq:proof1_3}
\end{align} where we have used the identity $I(\phi;\Dtrk)=\Ebb_{p(\phi)}[ \mathrm{D}_{\mathrm{KL}}(p(\Dtrk|\phi)|| p(\Dtrk)) ] $. Inequality \eqref{eq:proof1_3} is a non-negative parabola in $\lambda$, whose discriminant must be non-positive, which implies  the required upper bound \eqref{eq:infoupperbound}.
\end{proof}

The  MI $I(\phi;\Dtrk)$ in \eqref{eq:infoupperbound} is a measure of the \textit{sensitivity} of the  base-learner $\base$ to the input training data. A highly-sensitive base-learner may overfit the training data, resulting in a larger generalization error as reflected by the bound \eqref{eq:infoupperbound}. The upper bound of \eqref{eq:infoupperbound} also depends on the unknown data distribution $\sampledist$ through the MI term, as well as on the sub-Gaussian parameter $\sigma^2$, which is also a function of the the loss function $\ell(Z|\phi)$ via Assumption \ref{assum:convlearning}.

\subsection{Information-Risk Minimization}\label{sec:IRM}
The bound \eqref{eq:infoupperbound} provides useful quantitative insights into the generalization performance of a learning algorithm for a given data distribution. However, its dependence on the data distribution  makes it impossible to directly evaluate the bound \eqref{eq:infoupperbound}. We now present a relaxation of the bound of \eqref{eq:infoupperbound} that motivates a generalized Bayesian learning criterion known as \textit{ information risk minimization} \cite{zhang2006information}. Unlike the bound \eqref{eq:infoupperbound}, this criterion, already used in \eqref{eq:inner}, only depends on the training algorithm and on the training data set $\Dtrk$.

The relaxed bound is based on the
the following variational bound on the mutual information \cite{poole2019variational}, 
\begin{align}
	I(\phi;\Dtrk)&=\Ebb_{p(\Dtrk)}[\mathrm{D}_{\mathrm{KL}}(\base ||p(\phi))]\nonumber \\&\leq \mathbb{E}_{p(\Dtrk)}[\mathrm{D}_{\mathrm{KL}}(\base ||q(\phi))], \label{eq:variationalbound} \end{align}
which holds for any \textit{distribution} $q(\phi)$ on the space $\Phi$ of model parameters. In \eqref{eq:variationalbound}, the distribution $p(\Dtrk)$ represents the marginal of the joint distribution \eqref{eq:jointdistribution}.
Together with the inequality $\sqrt{ab} \leq \frac{ a \beta}{2} + \frac{2 b}{\beta}$ for $\beta>0$,  the inequality \eqref{eq:variationalbound} on the  bound of \eqref{eq:infoupperbound} yield the  
following upper bound on the population loss
\begin{align}
	&\mathbb{E}_{p(\Dtrk,\phi)} [L_{T_k}(\phi)]\nonumber \\& \leq \mathbb{E}_{p(\Dtrk)}\mathbb{E}_{\base}\biggl[\underbrace{  L_{\Dtrk}(\phi)+\frac{\mathrm{D}_{\mathrm{KL}}(\base\lVert q(\phi))}{\beta}}_{:=L^{\beta}_{\Dtrk}(\phi)}\biggr] +\frac{\beta \sigma^2}{2m}. \label{eq:regularizedloss_conventional}
\end{align}

 Inequality \eqref{eq:regularizedloss_conventional} upper bounds the average population loss in terms of a \textbf{regularized training loss}  $L^{\beta}_{\Dtrk}(\phi)$. The regularized training loss presents the KL divergence between the learning algorithm and the distribution $q(\phi)$ as a regularizer that measures the sensitivity of the learning algorithm $\base$ to the training data. The bound \eqref{eq:regularizedloss_conventional} motivates the use of regularized training loss as a training criterion. 

This yields the \textbf{information risk minimization} (IRM) problem  \cite{zhang2006information}
\begin{align}
	\min_{\base} \mathbb{E}_{ \base} \biggl[L_{\Dtrk}(\phi) + \frac{1}{\beta}\mathrm{D}_{\mathrm{KL}}(\base||q(\phi))\biggr], \label{eq:IRM}
\end{align}where the minimization is over the set of all probability distributions defined on the space of model parameters $\Phi$. The minimization \eqref{eq:IRM} corresponds to a generalized form of Bayesian learning \cite{knoblauch2019generalized, simeoneCUP}. In fact, the solution of the above unconstrained optimization problem is given by the \textbf{Gibbs posterior},
\begin{align}
	p^{\mathrm{Gibbs}}(\phi|\Dtrk) \propto q(\phi) \exp(-\beta L_{\Dtrk}(\phi)) \label{eq:Gibbsposterior}.
\end{align}
The Gibbs posterior \eqref{eq:Gibbsposterior} ``tilts" the ``prior" distribution $q(\phi)$ by an amount that depends on the training loss $L_{\Dtrk}(\phi)$ through the exponential function $\exp(-\beta L_{\Dtrk}(\phi))$. In particular, for $\beta=1$ and loss function $\ell(Z|\phi)=-\log p(Z|\phi)$, the Gibbs posterior reduces to the conventional Bayesian posterior \cite{knoblauch2019generalized}.

\subsection{PAC-Bayesian Bounds}\label{sec:PACBayesianbounds}
The information-theoretic bounds discussed in Section~\ref{sec:infobounds_conventionallearning} considered the absolute average of the  generalization error $\Delta L_{k}(\phi)$ in \eqref{eq:avgabsgenerror} over the randomized base-learner as well as over the training dataset. In contrast, PAC-Bayes theory seeks to bound the generalization error, $\mathbb{E}_{\base}[\Delta L_{k}(\phi)]$, on average over the models output by the base-learner, with high probability with respect to the distribution of the training dataset $\Dtrk \sim p(\Dtrk)$. The ``Bayesian" flavor of the bound comes through the definition of a prior distribution $q(\phi)$ defined on the space of model parameters $\Phi$ in a manner similar to \eqref{eq:IRM}.

Under Assumption~\ref{assum:convlearning}, the PAC-Bayesian bound can be stated as follows \cite{alquier2021user}.
\begin{theorem} For any prior distribution $q(\phi)$ defined on the space $\Phi$ of model parameters and $\beta>0$, the following inequality holds with probability at least $1-\delta$, for $\delta \in (0,1)$, with respect to the random draws of training dataset $\Dtrk \sim p(\Dtrk)$:
\begin{align}
 \mathbb{E}_{ \base} [L_{T_k}(\phi)] &\leq   \mathbb{E}_{\base} [L^{\beta}_{\Dtrk}(\phi)]+ \frac{1}{\beta}\log \frac{1}{\delta}+\frac{\beta \sigma^2}{2 N}, \label{eq:PACBayesian_bound}
\end{align}where $L^{\beta}_{\Dtrk}(\phi)$ is the regularized training loss in \eqref{eq:regularizedloss_conventional}. The bound holds simultaneously for all distributions $\base$. \end{theorem}
\begin{proof}
The PAC-Bayesian bound in \eqref{eq:PACBayesian_bound} can be derived by using Markov's inequality, followed by the application of change of measure as outlined next. Let $U(\Dtrk)=$ $\mathbb{E}_{q(\phi)}[\exp(\beta \Delta L_k(\phi))]$ denote the average $\beta$-exponentiated generalization error of the $k$th task. From Markov's inequality, we get that with probability at least $1-\delta$ over the random training dataset $\Dtrk$, the following inequalities hold
\begin{align}U(\Dtrk) \leq \frac{\mathbb{E}_{p(\Dtrk)}\mathbb{E}_{q(\phi)}[\exp(\beta \Delta L_k(\phi))]}{\delta} \leq \frac{\exp(\beta^2 \sigma^2/2N)}{\delta}, \label{eq:proof2_1}\end{align} where the last inequality follows from Assumption~\ref{assum:convlearning}. The left-hand side of \eqref{eq:proof2_1} can be equivalently rewritten, via a change-of-measure step, as   $$U(\Dtrk)= \mathbb{E}_{\base}\biggl[\exp\biggl(\beta \Delta L_k(\phi)-\log \frac{\base}{q(\phi)}\biggr)\biggr].$$ By \eqref{eq:proof2_1}, this implies that with probability at least $1-\delta$, we have the inequality
\begin{align}
	 \mathbb{E}_{\base}\biggl[\exp\biggl(\beta \Delta L_k(\phi)-\log \frac{\base}{q(\phi)}\biggr)\biggr] \leq \frac{\exp(\beta^2 \sigma^2/2N)}{\delta} \label{eq:proof2_2},
\end{align}for all $\base$. Applying Jensen's inequality on the left hand side of \eqref{eq:proof2_2} to take the expectation inside the exponential function, and subsequently taking logarithm on both sides,  yield the PAC-Bayesian bound in \eqref{eq:PACBayesian_bound}.
\end{proof}
\subsection{Information Risk Minimization Revisited}
The PAC-Bayesian bound \eqref{eq:PACBayesian_bound} has two important distinguishing features as compared to the information-theoretic bound \eqref{eq:infoupperbound}: $(a)$ it is data-distribution independent, while only depending on the available training data; and $(b)$ it holds uniformly over all learning algorithms. This formally motivates the use of regularized training loss \eqref{eq:regularizedloss_conventional} as a training criterion, providing a more principled derivation of IRM as a learning approach \cite{zhang2006information}.
 
\section{Generalization Error in Meta-Learning}\label{sec:genbounds_metalearning}
We now turn to the analysis of generalization for meta-learning. As discussed in Section~1, meta-learning aims to automatically optimize aspects of the 
inductive bias, encompassing the specifications of the model class  and base-learner (or learning algorithm), that are shared across the learning tasks. In this section, we  fix the model class and consider the inductive bias to be the vector of hyperparameters $\theta$ of the stochastic base-learner. Accordingly, the base-learner is described by the conditional distribution $\basetheta$ that maps the training data $\Dtrk$ and the hyperparameter vector $\theta$ to a vector of model parameters $\phi$. 

The goal of meta-learning is to automatically optimize the hyperparameter vector $\theta$ by observing data from a number of \textit{related} tasks. A key question in the learning-theoretic formulation of meta-learning is how to model the relatedness between the tasks. Following the standard formulation in \cite{baxter1998theoretical},  the tasks are modelled here as belonging to a \textit{task environment}, which describes a probability distribution $\taskdist{}$ over the space $\mathcal{T}$ of tasks as well as per-task data distributions $\{p(Z|T)\}$ for all tasks $T \in \mathcal{T}$. 

During meta-training, a meta-learner observes data from a finite number $K$ of \textit{meta-training tasks} $(T_1, \hdots, T_K)$, which are sampled i.i.d. according to the task distribution $\taskdist$. For each task $T_k \sim \taskdist$,  the meta-learner observes the corresponding training data set $\Dtrk$ of $N$ samples, which are sampled i.i.d. according to the per-task data distribution $\sampledist$. The resulting collection $\{\Dtrk\}_{k=1}^K$ of data sets from $K$ tasks constitute the \textit{meta-training data set}.  The meta-learner uses the meta-training data set $\{\Dtrk\}_{k=1}^K$ to optimize the hyperparameter vector $\theta$.

During meta-testing, the meta-learner encounters a new, previously unobserved, \textit{meta-test task} $T \sim \taskdist$, sampled from the same task environment, and observes the corresponding training dataset $\DtrT$. The base-learner $\basethetaT$  uses the  meta-learned hyperparameter vector $\theta$ and the meta-test task training data $\DtrT$ to optimize a task-specific model parameter $\phi$. 

The ideal goal of the meta-learner is to ensure that the population loss, $L_T(\phi)$, of the meta-test task accrued for the trained model parameter $\phi$, is minimized. As in \eqref{eq:avgabsgenerror}, the loss is averaged over the model parameter vectors $\phi$ output by the base-learner $\basethetaT$. Furthermore, an expectation is also evaluated across the meta-test task and training data set. The resulting problem amounts to finding a hyperparameter vector $\theta$ that  minimizes the \textit{meta-population loss},
\begin{align}
	\mathcal{L}(\theta)= \mathbb{E}_{\taskdist p(\DtrT)} \mathbb{E}_{\basethetaT}[L_{T}(\phi)]=\Ebb_{p(T)}[\Lscr_T(\theta)] \label{eq:metapopulationloss},
\end{align}where
\begin{align}
    \Lscr_T(\theta)= \mathbb{E}_{p(\DtrT)} \mathbb{E}_{\basethetaT}[L_{T}(\phi)],\label{eq:pertaskmetapopulationloss}
\end{align} and the meta-test task population loss $L_T(\phi)$ is as defined in \eqref{eq:populationloss_conventional}. 

The meta-population loss  \eqref{eq:metapopulationloss} cannot be evaluated since the task distribution $\taskdist$ as well as the per-task distribution $p(Z|T)$ are unknown. The meta-learner instead uses the \textit{meta-training loss} (see also \eqref{eq:outer} from previous section),
\begin{align}
	\mathcal{L}_{\{\Dtrk\}_{k=1}^K}(\theta) = \frac{1}{K} \sum_{k=1}^K L_{\Dtrk}(\theta) \label{eq:metatrainingloss},
\end{align} where \begin{align}L_{\Dtrk}(\theta) =\mathbb{E}_{\basetheta}[L_{\Dtrk}(\phi)]\label{eq:auxiliaryloss}\end{align} is the average per-task training loss, defined  in \eqref{eq:trainingloss_conventional},  over all model parameter vectors output by the base-learner.  

In a manner similar to the discussion on conventional learning in the previous subsection, the difference between the meta-population loss and meta-training loss is introduced as the \textit{meta-generalization error}
\begin{align}
	\Delta \mathcal{L}(\theta) = \mathcal{L}(\theta) - \mathcal{L}_{\{\Dtrk\}_{k=1}^K}(\theta) \label{eq:metageneralizationerror}.
\end{align}
A large meta-generalization error is an indication that the meta-learner's choice of the hyperparameter vector $\theta$ overfits to the meta-training data, failing to adapt to new previously, unobserved meta-test tasks. The following example illustrates the concept of meta-generalization error and meta-overfitting.
 As an example, consider the 3D-object pose prediction problem described in \cite{yin2020_maml_pp}, in which the input $X$ consists of a grey-scale image of a rotated object in a 3D space, and the output $Y$ reports the angle of rotation with respect to a canonical pose. A task corresponds to a specific object with a given canonical pose. When meta-training on a limited number of similar objects, the meta-learner may be able to find a single model that assigns the correct rotation angle to all inputs for all meta-training tasks. Such model can be also found via joint learning, whereby the model parameters $\phi_k$ for all meta-training tasks coincide with the hyperparameter vector $\theta$ (see Section 1).  In such cases, when meta-testing on a new, sufficiently different, object, the training algorithm fails to adapt, and the inductive bias optimized via meta-learning impairs training for new tasks. As a result, the meta-generalization error is large, and we say that we have \textit{meta-overfitting}.

In the next subsections, we seek to address the following two main questions: What factors contribute to the meta-generalization error? How do we quantify them? Recall that in conventional learning, the generalization error is the result of the availability of an insufficient number of training samples to train the base-learner. Since meta-learning is a bilevel optimization problem, as detailed in Section \ref{sec:ch3:bilevel-opt}, intuitively, the following factors contribute to the meta-generalization error:
\begin{enumerate}
	\item [$\bullet$] the \textit{within-task generalization error} due to a finite number of observed per-task data samples, as in conventional learning;
	\item [$\bullet$] the \textit{environment-level generalization error} due to the availability  of a finite number of meta-training tasks;
	\item [$\bullet$] and the\textit{ similarity}, or \textit{relatedness}, \textit{between the  tasks} encompassed by the task environment.
\end{enumerate}
 In the next subsection, we discuss information-theoretic bounds on meta-generalization error that address and quantify these three separate contributions to the meta-generalization error.
\section{Information-Theoretic Bounds on Meta-Generalization Error} \label{sec:infobounds_metageneralization}
In this subsection, we  provide an introduction to  information-theoretic  upper bounds on the meta-generalization error. We first extend the analysis in Section~\ref{sec:infobounds_conventionallearning} by accounting for the first two contributions to the meta-generalization error mentioned above. Then, we discuss a novel bound that explicitly quantifies   the third contribution.

The first step to obtain information-theoretic bounds on the meta-generalization error is to define a \textit{stochastic meta-learner}, in a manner analogous to the randomized base-learner studied in Section~\ref{sec:genbounds_conventional learning}. A stochastic meta-learner is described by a conditional distribution $\metalearner$ that maps the meta-training data $\metaset$ to the hyperparameter vector $\theta$.  Using the mapping $\metalearner$, the meta-learner samples a hyperparameter vector $\theta  $ from the conditional distribution $\metalearner$, which is then used by the randomized base-learner $p(\phi|\DtrT,\theta)$ during meta-testing.
\subsection{Information-Theoretic Bounds}

The performance metric of interest in this section is a natural extension from conventional learning to meta-learning \eqref{eq:avgabsgenerror}. Accordingly, we define the \textbf{absolute average meta-generalization error} as the absolute value of the meta-generalization error \eqref{eq:metageneralizationerror} averaged over the outputs of the randomized meta-learner as well as the meta-training set, {i.e.},
\begin{align}
	|\overline{\Delta \mathcal{L}}^{\mathrm{avg}}|=  |\mathbb{E}_{p(\metaset,\theta)}  [ \Delta \mathcal{L}(\theta)]| \label{eq:absavgmetagenerror}.
\end{align} In \eqref{eq:absavgmetagenerror}, the expectation is with respect to the joint distribution \begin{align}
p(\metaset,\theta)=p(\metaset)\metalearner , \end{align} where $p(\metaset)=\prod_{k=1}^K P_{\Dtrk}$ is the distribution of the meta-training set, with $p(\Dtr)$ being the marginal of the joint distribution $p(T,\DtrT)=p(T)p(\DtrT)$.

To obtain an upper bound on \eqref{eq:absavgmetagenerror}, the key step is to decompose the meta-generalization error \eqref{eq:metageneralizationerror} into terms that account for the \textit{within-task generalization error} and for the \textit{environment-level generalization error}. This can be done by defining an auxiliary loss function
\begin{align}
	\bar{\mathcal{L}}(\theta) = \mathbb{E}_{ p(T,\DtrT)}[L_{\mathcal{D}^{\mathrm{tr}}_T}(\theta)]=\Ebb_{p(T)}[\bar{\Lscr}_T(\theta)], \label{eq:Lbar}
\end{align}where
\begin{align}
    \bar{\Lscr}_T(\theta)= \mathbb{E}_{ p(\DtrT)}[L_{\mathcal{D}^{\mathrm{tr}}_T}(\theta)].
\end{align} The function \eqref{eq:Lbar} is the average of the training loss $L_{\mathcal{D}^{\mathrm{tr}}_T}(\theta)$ in \eqref{eq:auxiliaryloss} over randomly sampled meta-test data sets from the task environment. Using this function, the meta-generalization error $\Delta \mathcal{L}(\theta)$ in \eqref{eq:absavgmetagenerror} can be decomposed as the sum
\begin{align}
	\Delta \mathcal{L}(\theta) =\underbrace{\mathcal{L}(\theta) -  \bar{\mathcal{L}}(\theta)}_{\mbox{within-task gen. error}} +\underbrace{\bar{\mathcal{L}}(\theta) - \mathcal{L}_{\metaset}(\theta)}_{\mbox{environment-level gen. error}} \label{eq:decomposition_meta}.
\end{align} 

The first difference in \eqref{eq:decomposition_meta} captures the generalization error of a meta-test task randomly sampled from the task environment. A non-zero difference, $\mathcal{L}(\theta) -  \bar{\mathcal{L}}(\theta)$, is due to the availability of a finite number $N$ of training data samples for the meta-test task. In contrast, the second difference in \eqref{eq:decomposition_meta} accounts for the environment-level generalization error,  which is a consequence of the finite number $K$ of meta-training tasks. 
Together with the triangle inequality, the decomposition \eqref{eq:decomposition_meta} can be used to upper bound the absolute average meta-generalization error as
\begin{align}
	|\overline{\Delta \mathcal{L}}^{\mathrm{avg}}| &\leq |\mathbb{E}_{p(\metaset,\theta)} [  \mathcal{L}(\theta)-\bar{\mathcal{L}}(\theta]| \nonumber \\&+|\mathbb{E}_{p(\metaset,\theta)} [ \bar{\mathcal{L}}(\theta)-\mathcal{L}_{\metaset}(\theta)]|    \label{eq:decomposition_2}.
\end{align}
Each of the terms in \eqref{eq:decomposition_2} can be bounded separately to obtain an upper bound on the absolute average meta-generalization error. To this end, we make the following assumptions on the loss function.
\begin{assumption}\label{assum:1}
	The following assumptions hold:
	\begin{itemize}
		\item[$(a)$] The loss function $\ell(Z|\phi)$ is $\sigma^2_T$-sub-Gaussian with respect to the distribution $p(Z|T)$ of task $T \in \mathcal{T}$  for all $\phi \in \Phi$;
		\item[$(b)$] The average training loss $L_{\mathcal{D}^{\mathrm{tr}}}(\theta)$, defined in \eqref{eq:auxiliaryloss}, is $\delta^2$-sub-Gaussian with respect to the distribution $p(\mathcal{D}^{\mathrm{tr}})$ (which is the marginal of the joint distribution $p(T,\DtrT)$) for all $\theta \in \Theta$.
	\end{itemize}
\end{assumption}
\begin{theorem}\label{thm:infobound_metageneralizationerror}
	Under Assumption~\ref{assum:1} the following upper bound on the absolute average meta-generalization error holds
	\begin{align}
		|\overline{\Delta \mathcal{L}}^{\mathrm{avg}}| &\leq \sqrt{\frac{2 \delta^2}{K}I\Bigl(\theta;\metaset\Bigr)} + \mathbb{E}_{p(T)} \biggl[\sqrt{\frac{2\sigma_T^2}{N}I(\phi;\mathcal{D}^{\mathrm{tr}}_T)} \biggr]. \label{eq:infobound_metageneralizationerror}
	\end{align}
\end{theorem}
\begin{proof}
	To obtain the required upper bound, use Assumption~\ref{assum:1} to bound each of the two terms in \eqref{eq:decomposition_2} in a manner similar to the proof of Theorem~\ref{thm:infobound_convlearning}. We refer the readers to \cite{jose2021information_entropy} for details.
\end{proof}
Theorem~\ref{thm:infobound_metageneralizationerror}  provides an information-theoretic bound on the absolute average meta-generalization error that captures: $(a)$ the within-task generalization error via the ratio of the MI $I(\phi; \DtrT)$ to the number of per-task data samples;  and $(b)$ the environment-level generalization error via the ratio of the MI $I(\theta;\metaset)$ between the hyperparameter vector and meta-training tasks to the number $K$ of meta-training tasks. As discussed in Section~\ref{sec:genbounds_conventional learning}, the MI $I(\phi; \DtrT)$  measures the sensitivity of the base-learner to the input training dataset, while the MI $I(\theta;\metaset)$ captures the sensitivity of the hyperparameter vector to the meta-training dataset. Theorem~\ref{thm:infobound_metageneralizationerror} indicates that, in order to ensure a low meta-generalization error, the  two mutual information terms in \eqref{eq:infobound_metageneralizationerror} must be kept small as compared to $K$ and $N$, respectively.

While the bound in \eqref{eq:infobound_metageneralizationerror} captures the within-task and environment-level generalization errors, it does not provide insights into how the similarity between the tasks affects the meta-generalization error. In fact, the similarity between tasks is determined by the statistical properties of the task-environment $(p(T),\{p(Z|T)\})$ comprising of the task distribution $p(T)$ and the per-task distributions $\{p(Z|T)\}$. Therefore, the marginal  $p(\Dtr)$ of the joint distribution $p(T,\Dtr)$  inherently capture the statistical properties of the task environment. The MI term $I(\theta;\metaset )$ evaluated over meta-training dataset sampled i.i.d. according to the marginal distribution $p(\Dtr)$ hence implicitly accounts for the relatedness between tasks.

In the next section, we discuss an information-theoretic bound that explicitly captures the impact of task relatedness.

\subsection{Impact of Task Similarity on Meta-Generalization Error}
As discussed, the similarity between the tasks is determined by the statistical properties of the task environment. In this subsection, we seek answers to two questions: How to quantify the similarity between the tasks? How does task similarity impact   meta-generalization error?

To address the first question, following \cite{jose2021information}, we consider the following definition of relatedness between tasks in a task environment.
\begin{definition}
	A task environment $(p(T),\{p(Z|T)\})$ is said to be \textbf{$\epsilon$-related} with respect to a divergence measure $\mathrm{D}(\cdot || \cdot)$ if, on average over the independent selection of two tasks $T$ and $T' \sim p(T)$, the divergence $\mathrm{D}(p(\DtrT)\lVert p(\Dtr_{T'}) )$ is smaller than $\epsilon$, {i.e.},  the following inequality is satisfied
	\begin{align}
		\mathbb{E}_{T,T'\sim p(T)}\biggl[D\Bigl(p(\DtrT)\lVert p(\Dtr_{T'}) \Bigr) \biggr] \leq  \epsilon. \label{eq:relatednessmeasure}
	\end{align}
\end{definition}
Of particular interest are the KL divergence and Jensen-Shannon (JS) divergence. In the former case, we say that the task environment is
$\epsilon$-KL related, whereas in the latter case, the task environment is $\epsilon$-JS related. For two distributions $P$ and $Q$, the JS divergence between the distributions is defined as
\begin{align}
	\mathrm{D}_{\mathrm{JS}}(P \lVert Q)=0.5 \mathrm{D}_{\mathrm{KL}}(P \lVert 0.5(P+Q))+0.5 \mathrm{D}_{\mathrm{KL}}(Q \lVert0.5(P+Q)). 
\end{align} 

To get an intuitive understanding of the $\epsilon$-relatedness measure introduced in \eqref{eq:relatednessmeasure}, consider the following example.
\begin{example} \label{example:gaussian}
Assume that the data distribution for task $\tau$ is normally distributed as $p(Z|T=\tau)=\mathcal{N}(\tau,\nu^2)$ with mean $\tau$ and variance $\nu^2$. The task distribution $p(T)=\mathcal{N}(\bar{\mu};\bar{\nu}^2)$ defines a distribution over the mean parameter $\tau$ with mean $\bar{\mu}$ and variance $\bar{\nu}^2$. We then have 
	\begin{align}
		\mathbb{E}_{T,T'\sim p(T)}\biggl[D\Bigl(p(\DtrT)\lVert p(\Dtr_{T'}) \Bigr) \biggr]
		= \frac{N\bar{\nu}^2}{\nu^2}, \label{eq:example}
	\end{align}
	and hence the task environment is $\epsilon$-KL related if the inequality $N\bar{\nu}^2/\nu^2 \leq \epsilon$ holds. Note that, as the per-task data variance $\nu^2$ decreases for a given task variance $\bar{\nu}^2$, the task dissimilarity parameter $\epsilon$ grows   large.
\end{example}

The example also illustrates a potential drawback of using the KL divergence-based measure of task relatedness. Since the KL divergence in \eqref{eq:relatednessmeasure} is taken with respect to the i.i.d. distributions $p(\DtrT)=\prod_{j=1}^N p(Z_j|T)$, the tensorization property \cite{cover1999elements} of the KL divergence results in a KL divergence that scales with $N$, leading to an increasing measure of task dissimilarity with $N$. In contrast, the JS divergence is always bounded, i.e., $\mathrm{D}_{\mathrm{JS}}(P\lVert Q) \leq \log(2)$,  yielding without loss of generality a bounded task relatedness parameter $\epsilon \leq \log(2)$.

Having defined the measures of task-relatedness, the next question is how to explicitly characterize its impact on meta-generalization error. Towards understanding this aspect, note that in the absolute average meta-generalization error \eqref{eq:absavgmetagenerror}, the generalization error corresponding to each selection of meta-training and meta-test tasks from the task environment are ``mixed'' in the sense that their contributions are averaged.  This can be easily seen from the following equivalent characterization of the absolute average meta-generalization error \eqref{eq:absavgmetagenerror}:
\begin{align}
	|\overline{\Delta \mathcal{L}}^{\mathrm{avg}}|=\Bigl|\mathbb{E}_{p(T),p(\{T_k\}_{k=1}^K)}\mathbb{E}_{p(\metaset,\theta)}[\mathcal{L}_T(\theta)-\mathcal{L}_{\metaset}(\theta)]\Bigr|, \label{eq:absavgmetagenerror_2}
\end{align} where $\Lscr_T(\theta)$ in \eqref{eq:pertaskmetapopulationloss} is the per-task meta-population loss. The relatedness between the tasks becomes explicit when one analyze the generalization error incurred when a meta-learner trained on a given set of meta-training tasks is tested on a given meta-test task. Since the generalization error incurred on each selection of meta-training tasks and meta-test task is not separately considered in \eqref{eq:absavgmetagenerror}, the performance criterion $|\overline{\Delta \mathcal{L}}^{\mathrm{avg}}|$ fails to explicitly capture the impact of task relatedness on the meta-generalization error.

To mitigate the above drawback of the performance criterion in \eqref{eq:absavgmetagenerror}, following \cite{jose2021information},  this section adopts as the performance criterion the \textit{average absolute meta-generalization error}, which is defined as
\begin{align}
\!\!	|\overline{\Delta \mathcal{L}}|^{\mathrm{avg}}=\mathbb{E}_{p(T),p(\{T_k\}_{k=1}^K)}\biggl[\Bigl|\mathbb{E}_{p(\metaset,\theta)}[\Lscr_T(\theta)-\mathcal{L}_{\metaset}(\theta)]\Bigr|\biggr]. \label{eq:avgabsmetagenerror}
\end{align} The average absolute meta-generalization error in \eqref{eq:avgabsmetagenerror} evaluates the absolute value of the generalization error corresponding to each selection of meta-test task and meta-training tasks; and the resulting absolute values are averaged over the tasks.

The following result gives upper bound on the average absolute meta-generalization error in \eqref{eq:avgabsmetagenerror}.
\begin{theorem}
Let Assumption~\ref{assum:1} holds with Assumption \ref{assum:1}$(b)$ satisfied for the distribution $p(\DtrT)$ for every choice of task $T \in \mathcal{T}$. If the task environment is $\epsilon$-KL related, then the following upper bound on the average absolute meta-generalization error holds:
	\begin{align}
		|\overline{\Delta \mathcal{L}}|^{\mathrm{avg}} &\leq   \sqrt{2 \delta^2 \biggl(\frac{1}{K}I\Bigl(\theta;\metaset|\taskset\Bigr)+\epsilon\biggr)} \nonumber \\&+ \mathbb{E}_{T' \sim p(T)}\sqrt{2 \sigma_{T'}^2 \frac{I(\phi;\Dtr_{T'}|\taskset)}{N}}. \label{eq:bound_avgabsmetagenerror}
	\end{align} 
\end{theorem}
\begin{proof}
To obtain the required bound, we follow similar steps as in the proof of Theorem~\ref{thm:infobound_metageneralizationerror} by decomposing the meta-generalization error into within-task and environment-level generalization errors as in \eqref{eq:decomposition_meta}. The key difference comes in the evaluation of the environment-level generalization error, which we outline here. Conditioned on the meta-test task and meta-training tasks, the environment-level generalization error evaluates as
\begin{align}
    \Ebb_{p(\theta,\metaset)}[\bar{\Lscr}_T(\theta)-\Lscr_{\metaset}(\theta)],
\end{align}where $\bar{\Lscr}_T(\theta)$ is defined as in \eqref{eq:Lbar}. Note that the loss $\bar{\Lscr}_T(\theta)$ has an inner expectation over training dataset $\DtrT$ of the meta-test task; while the meta-training loss computes the average loss over the meta-training set $\metaset$. This difference can be captured using a change of measure argument, together with the sub-Gaussianity assumption on $L_{\DtrT}(\theta)$ under the distribution $p(\DtrT)$ as in the proof of Theorem~\ref{thm:infobound_convlearning}. This results in an additional KL divergence term $\mathrm{D}_{\KL}(p(\Dtrk)\lVert p(\DtrT))$ for $k=1,\hdots,K$ as compared to \eqref{eq:infobound_metageneralizationerror}. Under the assumption of $\epsilon$-KL relatedness, the above divergence measure can be upper bounded by $\epsilon$. We refer the readers to \cite{jose2021information} for more details.
\end{proof}

The bound \eqref{eq:bound_avgabsmetagenerror} captures explicitly the impact of task-relatedness via the parameter $\epsilon$, while also accounting for the meta-learner and base-learner sensitivities via the conditional mutual information terms as in the bound \eqref{eq:infobound_metageneralizationerror}. Due to this term, unlike \eqref{eq:infobound_metageneralizationerror}, in the asymptotic regime of $N,K \rightarrow \infty$, the bound in \eqref{eq:bound_avgabsmetagenerror} is non-vanishing.

\section{PAC-Bayes Analysis of Meta-Generalization Error}\label{sec:PACBayesianbounds_metalearning}
In Section~\ref{sec:infobounds_metageneralization}, we considered the average meta-generalization error as the performance criterion of interest, where the average was taken over the meta-learner outputs as well as over the meta-training set. In contrast, PAC-Bayesian bounds on meta-generalization error are high-probability bounds on the meta-generalization error, $\Ebb_{\metalearner}[\Delta \Lscr(\theta)]$, averaged over meta-learner outputs, over the random draws of the meta-training tasks $\taskset$, and over the corresponding training sets $\metaset$. 

To proceed, in a manner similar to the PAC-Bayes analysis of conventional learning in Section~\ref{sec:PACBayesianbounds}, we define a \textit{hyper-prior} distribution $q(\theta)$ on the space $\Theta$ of hyperparameter vectors. The hyperparameter vector $\theta$ is assumed to control  the \textit{prior} distribution $q(\phi|\theta)$ on the space of model parameters $\Phi$. The rationale for this choice is that the hyperparameter vector $\theta$ defines a common prior distribution on the model parameter that is meant to serve as useful shared knowledge across all tasks.

Under suitable assumptions on the loss function (see \cite{jose2021transfer}), the  PAC-Bayesian bound can be stated as follows.
\begin{theorem} Under the assumptions stated in \cite[Sec IV]{jose2021transfer}, for any hyperprior distribution $q(\theta)$ and prior $q(\phi|\theta)$, and for any $\beta>0$, the following inequality holds uniformly over all stochastic meta-learning algorithms $\metalearner$, with probability at least $1-\delta$, for $\delta \in (0,1)$, with respect to the random draws of the meta-training tasks $\taskset$ and meta-training data $\metaset$:
\begin{align}
    &\Ebb_{\metalearner}[\Lscr(\theta)] \nonumber \\&\leq \Ebb_{\metalearner}\biggl[\underbrace{\Lscr_{\metaset}(\theta)+\frac{1}{K}\sum_{k=1}^K \mathrm{D}_{\KL}(\basetheta \lVert q(\phi|\theta))}_{\Lscr^{\mathrm{IMRM}}(\theta)} \biggr] \nonumber \\ &+ \frac{1}{\beta} \mathrm{D}_{\KL}(\metalearner \lVert q(\theta)) + \Psi(N,K,\delta), \label{eq:PACBayesian_metagenerror}
\end{align} where  $\Psi(N,K,\delta) $ is a non-negative function of $N$, $K$ and $\delta$. \end{theorem}

The PAC-Bayesian bound on the meta-generalization error in \eqref{eq:PACBayesian_metagenerror} accounts for the \textit{sensitivity} of meta-learner to meta-training set through the KL divergence between the randomized meta-learner and the hyper-prior distribution. The base-learner sensitivity is also similarly accounted for by the KL divergence between the randomized base-learner and the prior distribution. 

The bound \eqref{eq:PACBayesian_metagenerror} holds uniformly overall meta-learners, and hence it provides a valid meta-training criterion. This observation motivates the  \textbf{information meta-risk minimization (IMRM)} approach introduced in \cite{jose2021transfer}, which extends to meta-training the IRM approach described in Section~\ref{sec:IRM}.
For any fixed base-learner $\basetheta$, IMRM minimizes the\textit{ regularized meta-training loss}, given by
\begin{align}
    \min_{\metalearner} \Lscr^{\mathrm{IMRM}}(\theta) + \frac{1}{\beta} \mathrm{D}_{\KL}(\metalearner \lVert q(\theta)) \label{eq:IMRM},
\end{align} where the optimization is over the set of all probability distributions $\metalearner$ on the space $\Theta$ of hyperparameter vectors. In a manner similar to the discussion in Section~\ref{sec:IRM}, for any fixed base-learner $\basetheta$, the optimal solution to problem \eqref{eq:IMRM} is given by the \textit{Gibbs meta-learner}
\begin{align}
    p^{\mathrm{Gibbs}}(\theta|\metaset) \propto q(\theta) \exp\Bigl(-\beta \Lscr^{\mathrm{IMRM}(\theta)}\Bigr). \label{eq:Gibbsmetalearner}
\end{align} The Gibbs meta-learner \eqref{eq:Gibbsmetalearner} ``tilts'' the hyperprior $q(\theta)$ by an amount that depends on the meta-loss $\Lscr^{\mathrm{IMRM}}(\theta)$ through the exponential function $\exp(-\beta \Lscr^{\mathrm{IMRM}}(\theta))$.
The meta-loss $\Lscr^{\mathrm{IMRM}}(\theta)$ in \eqref{eq:PACBayesian_metagenerror} is the average of the regularized per-task training loss over all the observed $K$ tasks, given by
\begin{align}
  \Lscr^{\mathrm{IMRM}}(\theta) =\frac{1}{K}\sum_{k=1}^K \biggl( L_{\Dtrk}(\theta)+ \frac{1}{\beta} \mathrm{D}_{\KL}(\basetheta \lVert q(\phi|\theta))\biggr).
\end{align} As seen in Section~\ref{sec:PACBayesianbounds}, the meta-loss $\Lscr^{\mathrm{IMRM}}(\theta)$ can be minimized by the choice of Gibbs base-learner \eqref{eq:Gibbsposterior} i.e., $\basetheta = p^{\mathrm{Gibbs}}(\phi|\Dtrk,\theta)$.
\section{Minimum Excess Meta-Risk   for Bayesian Meta-Learning}\label{sec:Bayesianmetalearning}
In this subsection, we turn to \textit{Bayesian meta-learning}. Bayesian meta-learning amounts to the application of the IMRM principle \eqref{eq:IMRM} via the meta-posterior distribution \eqref{eq:Gibbsmetalearner} with $\beta=1$ and with log-loss, i.e., $\ell(Z|\phi)=-\log p(Z|\phi)$, at the level of hyperparameter $\theta$; and of the IRM principle \eqref{eq:IRM} with $\beta=1$ via the posterior distribution \eqref{eq:Gibbsposterior} at the level of model parameter. As we will see, 
 under the assumption of well-specified model class, it is possible to provide an exact analysis of the optimality error of Bayesian meta-learning.
 
 A model class $\mathcal{M}=\{p(Z|\phi)| \phi \in \Phi\}$, comprising of conditional distributions $p(Z|\phi)$ parameterized by model parameter $\phi \in \Phi$, is said to be \textit{well-specified} if the true data distribution $p(Z|T)$ belongs to the model class.  Specifically, there exists a model parameter vector $\phi_T \in \Phi$ such that the true distribution equals $p(Z|T)=p(Z|\phi_T)$. In the Bayesian setting, the model parameter $\phi$ is treated as a latent random variable and is endowed with a prior distribution $p(\phi)$. Consequently, the joint distribution of the model parameter $\phi$, training data set $\Dtr$, and test data $Z=(X,Y)$ is assumed to equal \begin{align}
    p(\phi,\Dtr,Z)=p(\phi) p(\Dtr|\phi) p(Z|\phi), \label{eq:Bayesian_jointdistribution}
\end{align} where $p(\Dtr|\phi)=\prod_{j=1}^N p(Z_j|\phi)$. 

Building on \eqref{eq:Bayesian_jointdistribution}, Bayesian meta-learning describes a \textit{hierarchical Bayesian model}: The hyperparameter vector $\theta$ and model parameter vector $\phi$ are assumed to be latent random variables with the joint distribution $p(\theta,\phi)=p(\theta)p(\phi|\theta)$;  the meta-training tasks, described by model parameter vectors $\{\phi_k\}_{k=1}^K$, and the meta-test task, described by the model parameter vector  $\phi_T$,  share a common  hyperparameter vector $\theta$ in the sense that $\{\phi_k\}_{k=1}^K$ and $\phi_T$  are generated i.i.d. according to the distribution $p(\phi|\theta)$. Consequently, the joint distribution of hyperparameter $\theta$, the model parameters $\{\phi_k\}_{k=1}^K$, $\phi_T$, the meta-training set $\{\Dtr_k\}_{k=1}^K$, the meta-test training data $\Dtr_T$ and test input $Z$ equals
\begin{align}
&p(\theta,\{\phi_k\}_{k=1}^K,\phi_T,\{\Dtr_k\}_{k=1}^K,\Dtr_T,Z)\nonumber\\ &= \underbrace{p(\theta)\biggl(\prod_{k=1}^K p(\phi_k|\theta) p(\Dtr_k|\phi_k) \biggr)}_{\textrm{meta-training}} \underbrace{p(\phi_T|\theta) p(\Dtr_T|\phi_T)p(Z|\phi_T) }_{\textrm{meta-testing}}. \label{eq:jointdistribution_metalearning}
\end{align}

The Bayesian meta-learner uses the meta-training data set $\{\Dtr_k\}_{k=1}^K$, the meta-test task training data $\Dtr_T$, and the test input feature $X$, to predict the output label $Y$. The error in predicting the output label $Y$ from observation of the above data is measured via the loss function $\ell(Y|X,\mathcal{D})$ with $\mathcal{D}=(\metaset,\DtrT)$. For simplicity, throughout this subsection,  we consider the log-loss as $\ell(Y|X,\mathcal{D})=-\log p(Y|X,\mathcal{D})$. In particular, we have
\begin{align}
   \ell(Y|X,\mathcal{D}) =-\log \Ebb_{p(\theta,\phi|\mathcal{D},X)}[p(Y|X,\phi)], \label{eq:logloss_predictive}
\end{align}where $p(\theta,\phi|\mathcal{D},X)$ is the meta-posterior distribution   from \eqref{eq:jointdistribution_metalearning}. 

The \textbf{Bayesian predictive meta-risk} is the average predictive loss incurred over the observed meta-training dataset $\metaset$, the test task training data $\DtrT$ and the test feature $X$, given by
\begin{align}
    R_{\log}(Y|X,\mathcal{D})&= \Ebb_{p(X,Y,\mathcal{D})}[-\log p(Y|X,\mathcal{D})] \nonumber \\&=H(Y|X,\mathcal{D}) \label{eq:Bayesianrisk},
\end{align} where the expectation  is with respect to the joint distribution \eqref{eq:jointdistribution_metalearning}. Equation \eqref{eq:Bayesianrisk} shows that under log-loss, the Bayesian meta-predictive risk is quantified exactly by the conditional entropy \begin{align}H(Y|X,\mathcal{D})=\Ebb_{p(X,Y,\mathcal{D})}[-\log p(Y|X,\mathcal{D})], \end{align} which captures the \textbf{total predictive uncertainty} of the Bayesian meta-learner. 

We  note that by taking the expectation over joint posterior inside the log in the loss function \eqref{eq:logloss_predictive}, the Bayesian predictive risk of \eqref{eq:Bayesianrisk} is different from the average meta-population loss \eqref{eq:metapopulationloss} under the log-loss. The latter considers expectation outside the log and thus constitute the inferential risk in determining the true model parameters. We refer the readers to \cite{masegosa2020learning} for more details on this point.

If the Bayesian meta-learner, aided by a genie, had access to the true hyperparameter vector as well as the model parameters, it would incur the predictive loss  $\ell(Y|X,\theta, \phi)=-\log p(Y|X,\theta,\phi)=-\log p(Y|X,\phi)$. The resulting \textbf{genie-aided predictive meta-risk} then evaluates as
\begin{align}
    R_{\log}(Y|X,\phi)&= \Ebb_{p(X,Y,\phi)}[-\log p(Y|X,\phi)] \label{eq:generalrisk_genie} \\&=H(Y|X,\phi) \label{eq:Bayesianrisk_genie}.
\end{align} The genie-aided predictive meta-risk, quantified by the conditional entropy $H(Y|X,\phi)$, captures the \textbf{aleatoric uncertainty}, which accounts for the uncertainty inherent in the data generation process. Note that aleatoric uncertainty is inherent in the model and it cannot be alleviated by gaining access to larger number of data samples.

The difference between the Bayesian predictive meta-risk and the genie-aided predictive meta-risk is the  \textbf{minimum excess meta-risk (MEMR)}, given by
\begin{align}
    \mathrm{MEMR}_{\log} &= R_{\log}(Y|X,\mathcal{D})-R_{\log}(Y|X,\phi). \label{eq:MEMR}
\end{align} The MEMR \eqref{eq:MEMR} can be exactly evaluated as the conditional MI $I(Y;\phi|X,\mathcal{D})$, given by
\begin{align}
    \mathrm{MEMR}_{\log} &=H(Y|X,\mathcal{D})- H(Y|X,\phi)\nonumber\\
    &=I(Y;\phi|X,\mathcal{D})\label{eq:MI}.
\end{align} The conditional MI, and thus the MEMR, capture the \textbf{epistemic uncertainty} of the Bayesian meta-learner resulting from  using finite number $K$ of meta-training tasks and number $N$ of per-task data samples for inference. The relation in \eqref{eq:MI} thus decomposes the total predictive uncertainty  $H(Y|X,\mathcal{D})$ as
\begin{align}
    H(Y|X,\mathcal{D})
&=\mathrm{MEMR}_{\log}+  H(Y|X,\phi,\theta),
\end{align} i.e., as the sum of epistemic uncertainty and aleatoric uncertainty. Importantly, in contrast to the aleatoric uncertainty, the epistemic uncertainty depends on the observed data, and is non-increasing with increasing number of observed tasks $K$ and per-task samples $N$ \cite{jose2022aistats}. 

Leveraging standard information-theoretic tools, the MEMR of \eqref{eq:MI} can be further refined to distil two contributions to the epistemic uncertainty. Specifically, the MI $I(Y;\phi|X,\mathcal{D})$ can be upper bounded as
\begin{align}
    I(Y;\phi|X,\mathcal{D}) \leq \frac{I(\theta;\metaset)}{KN}
+ \frac{I(\phi;\DtrT|\theta)}{N}.\end{align} The first
term captures the sensitivity of the hyperparameter $\theta$ on the meta-training set $\metaset$. The second
term corresponds to the average sensitivity of the
model parameter $\phi$ on the meta-test task training
data $\DtrT$ assuming that the hyperparameter $\theta$ is known.
Thus, the epistemic uncertainty which applies to the domain of
the target variable $Y$, is upper bounded by the sum of
two contributions that pertain the uncertainty levels
in the spaces of hyperparameter and model parameter,
respectively. We refer the readers to \cite{jose2022aistats} for the proof, and  for a treatment of general loss functions.

\section{Sharper Meta-Risk Analysis in Meta Linear Regression}
\label{sec:ch4:meta_linear_risk}
The meta-risk analysis in the previous subsections mostly focuses on the \emph{upper bound} or the \emph{worst case} of generalization performance under general learning problems and models. In a separate line of research, \emph{the precise generalization performance} of meta-learning has been studied in the context of mixed linear regression; see e.g.,  \cite{kong2020meta,gao2020_model_opt_tradeoff_ml,collins2020does,chua2021fine,du2020few,bai2021_trntrn_trnval}.   
In \cite{kong2020meta}, the focus is on finding scenarios when abundant tasks with small data can compensate for lack of tasks with big data. 
In \cite{chua2021fine,du2020few}, the focus is on studying the generalization performance of the representation based meta-learning. 
The meta-risk of MAML and joint learning has been analytically compared in \cite{gao2020_model_opt_tradeoff_ml,collins2020does}, and the regime where MAML has provable performance gain over joint learning has been identified.  
Recently, the impact of splitting training and validation datasets on the performance of iMAML has been studied in \cite{bai2021_trntrn_trnval}.

Complementary to \cite{jose2022aistats}, a unified meta-risk analysis has been recently established in \cite{chen2022_bamaml} under the meta linear regression setting, which provides a solid ground to compare the exact meta-risks of joint learning, MAML, iMAML and Bayesian MAML. Under some regularity assumptions, Bayesian MAML indeed has provably lower meta-risk than iMAML, MAML and joint learning \cite{chen2022_bamaml}.

\section{Conclusions}
This section presented a learning-theoretic study of the meta-learning problem by adopting an information-theoretic framework. In the frequentist meta-learning setting, the information-theoretic approach is used to quantify the meta-generalization error as a function of the cross-task and within-task generalization errors, as well as the relatedness between tasks. The information-theoretic framework is also connected to PAC-Bayesian bounds through the principle of information risk minimization. Finally, we discussed how the information-theoretic framework captures the excess predictive risk in Bayesian meta-learning.

%% file: Chapter5/Chapter5.tex
\chapter{Applications of Meta-Learning to Communications}
\label{sec:ch5}

\section{Overview}
For decades, communication systems have been engineered through carefully designed \textbf{model-based} algorithms that build on an analytical model of the underlying system. More recently, the increased complexity of communication scenarios, encompassing heterogeneous services and flexible software-defined multi-technology radio access networks (RANs), is raising renewed interest in \textbf{data-driven methods}. These techniques are based on machine learning, and are viewed as a complementary, and often synergistic, design approach \cite{simeone2018very}. As an example, in the O-RAN architecture, a leading proposal for 6G ``open-RAN'' systems, many network functionalities, at different temporal and spatial scales, are envisaged to be implemented via AI tools \cite{bonati2021intelligence}.

The main drawback of machine learning methods is given by the often prohibitive requirements in terms of dedicated training data and of computational effort. This issue is especially pronounced for physical-layer and medium-access (MAC) layer functions, which are subject to temporal variations in connectivity conditions. For instance, a coherent receiver at the physical layer, if trained for particular channel setting, generally suffers from degraded performance when the channel conditions change \cite{xia2020note, bourtsoulatze2019deep}. Meta-learning provides an ideal framework to design data-driven methods that can transfer knowledge across different communication settings, enabling adaptation to new connectivity conditions.

This section provides a review of some applications of meta-learning to communication systems by focusing on demodulation; encoding and decoding; channel prediction at the physical layer; and power control at the MAC layer.


\section{Demodulation}
\label{sec:ch5_demodul}
\begin{figure}
	\begin{center}
		\includegraphics[width=0.5\textwidth]{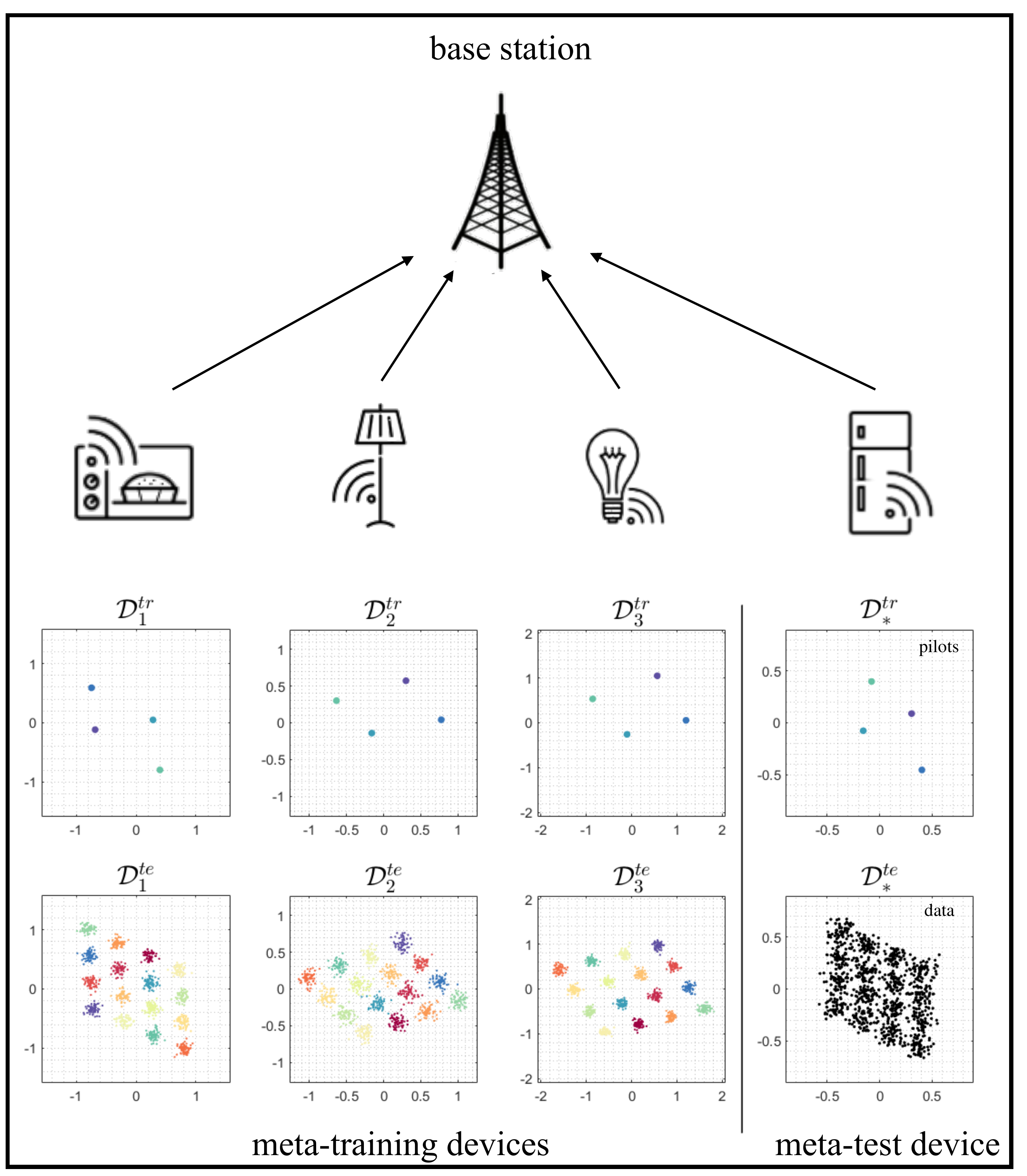}
	\end{center}
	\caption{Meta-learning for demodulation: By utilizing received pilots from multiple previous transmissions by different devices, a meta-learned demodulator can significantly reduce the number of pilots required for demodulation of data sent by a new device.\label{fig0ch5}}
\end{figure} 

Demodulation is a fundamental physical-layer function consisting of the task of estimating the transmitted symbols from the received baseband signals. Demodulators must compensate for the fading effect on the received signal of the transmission channel. This is done by leveraging the transmission of known symbols, referred to as pilots. 

\textbf{Model-based} methods typically assume a \emph{linear} fading channel model with additive white Gaussian noise (AWGN). Under this model, the standard approach first estimates the channel response using the pilots via a minimum mean squared error (MMSE) estimator. Then, the estimated channel is used to obtain a maximum likelihood estimate of the transmitted symbols, which minimizes the symbol error rate (SER) under the assumption that the channel is well estimated. 

In some communication scenarios, especially Internet-of-Things (IoT) systems involving low-complexity devices, linear models may fail to fully describe the relationship between the transmitted symbols and the received signal. In particular, they do not account for non-linear effects such as transmitter's imperfections \cite{tandur2007joint}. By addressing this ``model deficit'' \cite{simeone2018very}, \textbf{data-driven} demodulation can outperform the outlined conventional model-based strategy. This is the subject of this subsection, which follows reference \cite{park2020learning}.

\subsubsection{5.2.1 Problem Definition}
Consider an IoT scenario in which devices transmit short packets sporadically to a base station (BS). As mentioned, IoT devices may be affected by non-linear hardware distortions. An example of distorted constellation points for 16-ary quadrature amplitude modulation (16-QAM) under I/Q imbalance is shown in Fig.~\ref{fig0ch5}. As a result, the conventional model-based demodulator described above is generally suboptimal, as it ignores hardware nonlinearities. Conventional machine learning methods may address this model deficit, but the only available training data is given by the pilots within each short packet. Meta-learning can mitigate this problem. We note that a complementary approach is to integrate data-driven and model-based approaches \cite{raviv2022online, shlezinger2022model}, which will be briefly discussed in Section~\ref{sec:ch7}. 

For an IoT device indexed by an integer $k$, given an input symbol $s_k \in \mathcal{S}$ that lies in the set of all constellation points $\mathcal{S}$. The transmitted signal $x_k$ is a function of the information symbol $s_k \in \mathcal{S}$ that accounts for the hardware distortion caused by imperfections at device $k$. This function is described by a stochastic mapping 
\begin{align}
	\label{eq:ch5:hardware_distortion}
    x_k \sim p_k(\cdot|s_k)
\end{align}
for some conditional distribution $p_k(\cdot|s_k)$. We assume that the received signal $y_k$ can be expressed as the output of a flat fading channel as in
\begin{align}
	\label{eq:ch5:channel_mapping}
    y_k = h_k x_k + z_k,
\end{align} 
where $h_k$ is the complex channel gain between the device $k$ and the BS; and $z_k \sim \mathcal{CN}(0,N_0)$ is additive complex Gaussian noise. The channel is assumed to be constant within a coherence time that is longer than the short packet time duration of the IoT devices.  
Neither the channel $h_k$ nor the mapping $p_k(\cdot|s_k)$ are known to device $k$ or to the BS. 

We assume the transmission of $N$ pilots in each transmitted frame. Accordingly, the training data set for device $k$, referred to as $\mathcal{D}_k$, is given as 
\begin{align}
    \mathcal{D}_k = \{ (s_k^{(i)}, y_k^{(i)}): i=1,...,N \},
\end{align}
where $s_k^{(i)} \in \mathcal{S}$ is the $i$-th pilot symbol sent by device $k$, and $y_k^{(i)}$ is the resulting signal \eqref{eq:ch5:channel_mapping}--\eqref{eq:ch5:hardware_distortion} received by the BS.

\subsubsection{5.2.2 Conventional Learning} 
Let us fix a model class $p(s|y,\phi)$ that defines the probability function of the symbol $s$ given the received signal $y$ based on the model parameter vector $\phi$. The model class $p(s|y,\phi)$ is typically chosen as a neural network with weight vector $\phi$. Given training data set $\mathcal{D}_k$, a conventional machine learning solution trains the demodulator within the given class by minimizing the cross-entropy loss
\begin{align}
    L_{\mathcal{D}_k}(\phi) = - \frac{1}{N} \sum_{(s_k,y_k) \in \mathcal{D}_k} \log p(s_k|y_k,\phi),
\end{align}
over the parameter vector $\phi$, hence addressing the problem 
\begin{align}
	\label{eq:ch5:conven_learning_demod}
	\min_{\phi} L_{\mathcal{D}_k}(\phi).
\end{align}

\subsubsection{5.2.3 Meta-Learning} 
We consider pilot data from $K$ devices as meta-training data. Meta-learning can transfer knowledge from pilots of other devices, each with their own hardware distortions and channel realizations, via an optimized inductive bias. 

\noindent \textbf{Frequentist meta-learning.}
Splitting the data set $\mathcal{D}_k$ with $N$ samples for device $k$ into a training part $\mathcal{D}_k^\text{tr}$ with $N^\text{tr}$ samples and a validation part $\mathcal{D}_k^\text{va}$ with $N^\text{va}$ samples as explained in Section~\ref{sec:ch1}. the meta-learning objective for frequentist meta-learning is given by the problem 
\begin{align}
	\min_\theta \left\{ \mathcal{L}_{\mathcal{D}^\text{mtr}}(\theta) = \frac{1}{K} \sum_{k=1}^K L_{\mathcal{D}_k^\text{va}}( \phi^\text{tr}(\mathcal{D}_k^\text{tr}|\theta) ) \right\},
\end{align}
where the per-device model parameter vector $\phi_k = \phi^\text{tr}(\mathcal{D}_k^\text{tr}|\theta)$ for device $k$ is adapted using the pilots $\mathcal{D}_k^\text{tr}$ for a fixed hyperparameter vector $\theta$ as in \eqref{eq:ch14:inner}, which we denote as, $\phi^{\textrm{tr}}(\mathcal{D}_{k}^{\text{tr}}|\theta)\underset{\theta}{\leftarrow}\min_{\phi}L_{\mathcal{D}_{k}^{\text{tr}}}(\phi)$.

The performance of the data-driven demodulator $\phi$ is measured by symbol error rate 
\begin{align}
	\text{SER} = \mathbb{E}_{ s,y \sim p(s,y) } \mathbf{1}(s \neq \hat{s}(y|\phi)),
\end{align}
where $\hat{s}(y|\phi) = \argmax_{s \in \mathcal{S}} p(s|y,\phi)$ is the output of the demodulator given received signal $y$ in \eqref{eq:ch5:channel_mapping}--\eqref{eq:ch5:hardware_distortion}; while $p(s,y)=p(s)p(y|s)$ is the joint distribution of the symbol $s\in\mathcal{S}$ and of the received signal $y$, with $p(y|s)$ given by \eqref{eq:ch5:channel_mapping}--\eqref{eq:ch5:hardware_distortion}. The symbol distribution $p(s)$ is typically chosen to be uniform over the constellation set $\mathcal{S}$. We next provide numerical results obtained under model \eqref{eq:ch5:channel_mapping}--\eqref{eq:ch5:hardware_distortion} with $p(x|s)$ modelling I/Q imbalance at the transmitter. We refer to \cite{park2020learning} for details. 

Fig.~\ref{fig1ch5} shows the SER of the new, meta-test task, as a function of number of pilots $\tilde{N}^\text{tr}$ available during meta-testing using MAML, REPTILE, and CAVIA, which were introduced in Section~\ref{sec:ch2}. The number of pilots available for the meta-training tasks is set to $N^\text{tr}=4$ and $N^\text{va}=3196$. Note that we deviate here from the assumption that the same number of pilots is used during both meta-training and meta-testing. This allows us to consider the practical case in which the number of pilots for new device may not be known a priori, i.e., during the meta-learning phase. 

As seen in Fig.~\ref{fig1ch5}, meta-learning-aided demodulators outperform the conventional model-based communication scheme based on maximum likelihood (ML) demodulation with MMSE channel estimation; as well as the conventional machine learning scheme that trains from scratch a demodulator for each device. This benefit stems from the capacity of meta-learning to successfully transfer knowledge from pilots of previously active devices.

Next, Fig.~\ref{fig1_2ch5} demonstrates the SER with respect to number of meta-training devices $K$. As discussed in Section~\ref{sec:genbounds_metalearning}, using data from few meta-training devices may yield meta-overfitting, which leads to a high SER for new devices owing to the poor adaptation capability of the training algorithm. In contrast, when $K$ is large enough, the demodulator based on meta-learning can successfully achieve a low SER, while joint learning, which optimizes a single demodulator across all   meta-training devices, fails to transfer useful knowledge to   new devices.

\begin{figure}
	\begin{center}
		\includegraphics[width=0.5\textwidth]{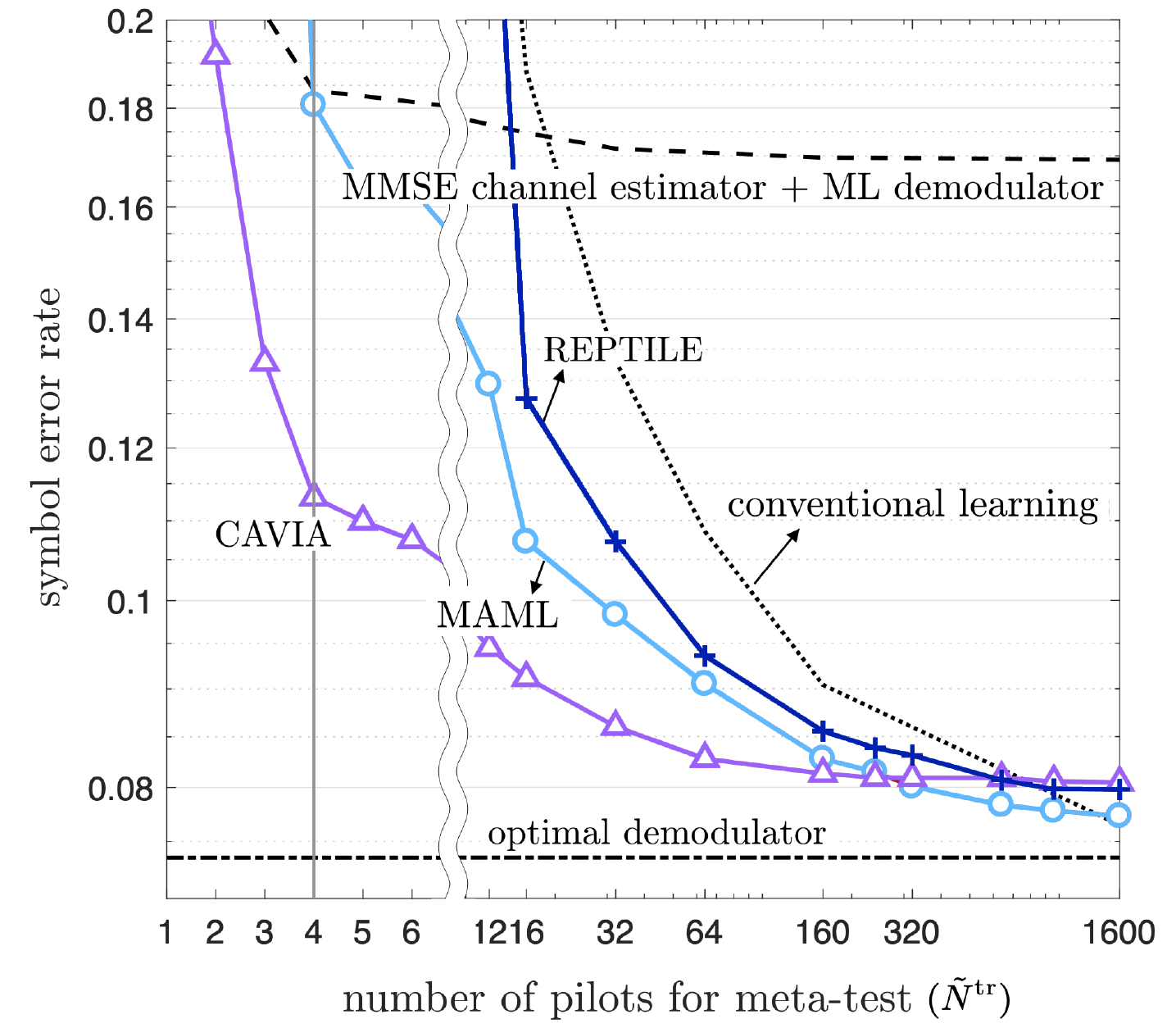}
	\end{center}
	\caption{Meta-learning for demodulation: SER as a function of number of pilots $(\tilde{N}^\text{tr})$ used during meta-testing with $16$-QAM, Rayleigh fading, and I/Q imbalance under a $20\text{ dB}$ signal-to-noise ratio (SNR). $K=1000$ meta-training devices with $N^\text{tr}=4$ and $N^\text{va}=3196$ are assumed during meta-training (adapted from \cite{park2020learning}). \label{fig1ch5} }
\end{figure}

\noindent \textbf{Bayesian meta-learning.}
\begin{figure}%
	\begin{center}
		\includegraphics[width=0.5\textwidth]{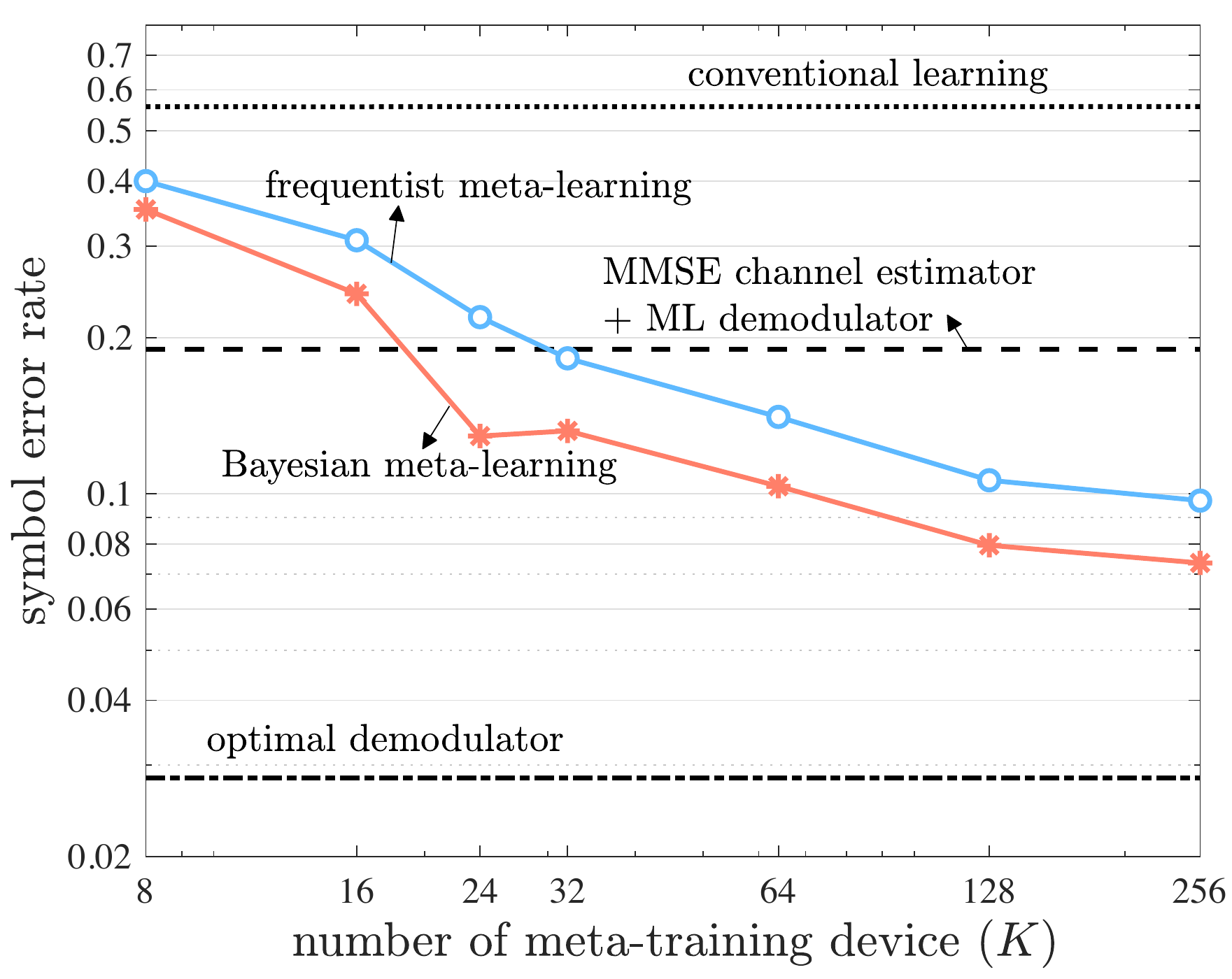}
	\end{center}
	\caption{Meta-learning for demodulation: SER as a function of number $K$ of meta-training device with $16$-QAM, Rayleigh fading, and I/Q imbalance under a $18\text{ dB}$ SNR. $\tilde{N}^\text{tr}=8$ pilots are used for meta-testing (adapted from \cite{kfir2021towards}). \label{fig2_1ch5} }
\end{figure}
While frequentist meta-learning effectively reduces the pilot overhead required for demodulation, the resulting trained demodulator may not be \textbf{well calibrated}, providing overconfident decisions. This is a well-known problem of frequentist learning \cite{guo2017calibration}. Bayesian meta-learning can address this problem by properly accounting for epistemic uncertainty caused by limited training data (see Section~\ref{subs:bayes_maml}) \cite{kfir2021towards}. 

\begin{figure}
	\begin{center}
		\includegraphics[width=0.5\textwidth]{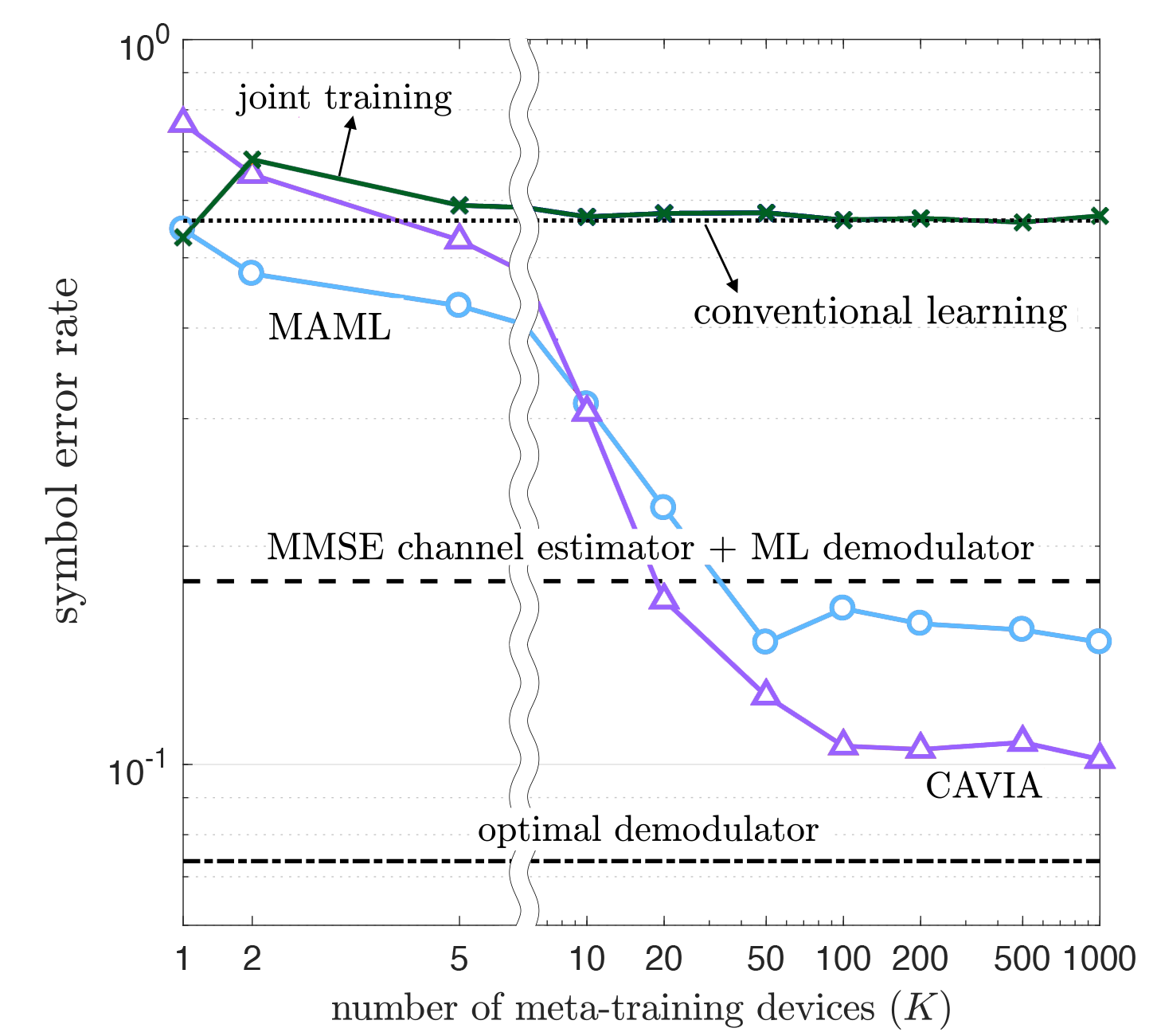}
	\end{center}
	\caption{Meta-learning for demodulation: SER as a function of number $K$ of meta-training device with $16$-QAM, Rayleigh fading, and I/Q imbalance under a $20\text{ dB}$ SNR. $\tilde{N}^\text{tr}=8$ pilots are used for meta-testing (adapted from \cite{park2020learning}). \label{fig1_2ch5} }
\end{figure}

To elaborate on this point, we first describe how to quantify the calibration of a discriminative probabilistic model. Given a demodulator $p(s|y,\phi)$ that yields a point decision $\hat{s}(y|\phi) = \argmax_{s \in \mathcal{S}} p(s|y,\phi)$, the corresponding \textbf{confidence} for the input $y$ is defined as 
\begin{align}
	\label{eq:ch5:conf}
	\text{conf}(y|\phi) = p( \hat{s}(y|\phi) | y,\phi ).
\end{align}
Ideally, the confidence level \eqref{eq:ch5:conf} should be a reliable measure of the true accuracy of the decision $\hat{s}(y|\phi)$. To quantify this aspect, we define the average \textbf{accuracy} for all inputs having a confidence level $p$ as \cite{guo2017calibration}
\begin{align}
	\text{acc}(p) = \mathbb{P} [  \hat{s}(y|\phi) = s | \text{conf}(y|\phi) = p ],
\end{align}
where the probability is taken over the underlying ground-truth distribution $p(y,s)$ for the input $y$ and target $s$. A \textbf{well calibrated} demodulator is a predictor that satisfies the following equality
\begin{align}
	\text{acc}(p) = p,
\end{align}
so that accuracy and confidence level are equal for all $ p \in [0,1]$. \textbf{Reliability diagrams} plot the accuracy $\text{acc}(p)$ versus the confidence level $p$ to gauge the extent to which the confidence level estimated by the model matches the ground-truth accuracy \cite{guo2017calibration}. By replacing the single demodulator $ p(s|y,\phi) $ with the ensemble demodulator $\mathbb{E}_{\phi \sim p(\phi|\mathcal{D})} p(s|y,\phi) $ that accounts for the ``opinions'' of multiple models weighted by the (approximate) posterior distribution $p(\phi|\mathcal{D})$, Bayesian learning can yield better calibrated decisions as compared to frequentist learning. This was investigated in \cite{cohen2021learning, kfir2021towards}.

Fig.~\ref{fig2_1ch5} shows the SER as a function of number of meta-training devices $K$. Similar to Fig.~\ref{fig1_2ch5}, both frequentist and Bayesian meta-learning outperform conventional schemes, validating again the conclusion that meta-learning can transfer useful knowledge from multiple devices. Apart from some improvement in accuracy, the key benefit of Bayesian meta-learning is in terms of calibration, as illustrated by the reliability diagram in Fig.~\ref{fig2ch5}. By capturing epistemic uncertainty caused by the availability of few pilots, here $N^\text{tr}=8$, Bayesian meta-learning produces well-calibrated decisions. In fact, the diagram shows that the confidence of the demodulator matches well the actual accuracy. More details can be found in \cite{kfir2021towards}.

\noindent \textbf{Online meta-learning.}
In the communication setting under study in this subsection, it may be practically useful to accumulate meta-training data set in an online fashion as transmissions from more devices are received by the BS. This setting has been also studied in \cite{park2020learning}, and will be briefly outlined in Section~\ref{sec:ch7}.

\begin{figure}%
	\begin{center}
		\includegraphics[width=0.5\textwidth]{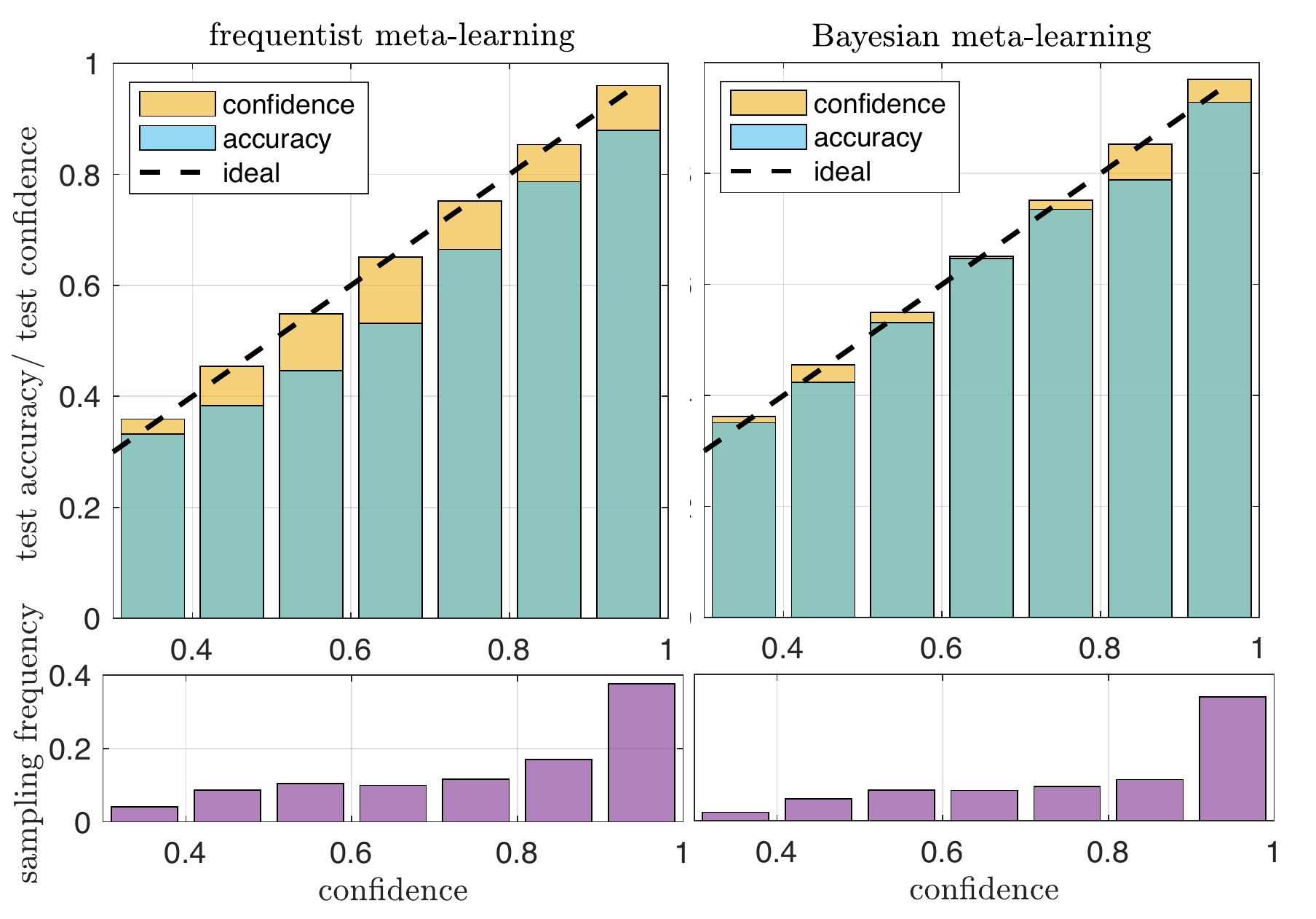}
	\end{center}
	\caption{Meta-learning for demodulation: Reliability diagrams for both frequentist and Bayesian meta-learning. Well calibrated demodulators should follow the dashed line in the figure, i.e., the confidence of the demodulator should match the actual accuracy (adapted from \cite{kfir2021towards}). \label{fig2ch5} }
\end{figure}

\section{Encoding and Decoding} 
\label{sec:ch5:encoding_and_decoding}
While the previous subsection addressed the model deficit problem caused by hardware imperfections, this subsection deals with an instance of \textbf{\emph{algorithm deficit}}, in which the optimal algorithm for the problem of interest is unknown. We specifically focus on the problem of jointly designing encoder and decoder for a communication link over a channel that is only accessible via a simulator as in \cite{o2017introduction, cammerer2020trainable, aoudia2021end}. 

In this setting, the issue is not that of reducing the amount of data, which can be generated at will using the simulator, but rather that of ensuring that a new encoder-decoder pair can be optimized quickly, using limited computational resources, for each new channel coefficients. We show in this subsection that meta-learning can reduce the iteration complexity of training encoder-decoder pairs for new communication conditions. The presentation follows reference \cite{park2020meta}.



\subsubsection{5.3.1 Problem Definition}
Consider a communication link with a known channel model. As illustrated in Fig.~\ref{fig3_1ch5}, the encoder and decoder are implemented via neural networks. Using the approach introduced in \cite{o2017introduction}, training can be done in an unsupervised manner by interpreting the architecture in Fig.~\ref{fig3_1ch5} as an \textbf{\emph{autoencoder}} whose goal is to reproduce the input message $m$ of $k$ bits at the output of the decoder as the estimate $\hat{m}$. This approach generally requires many iterations to optimize encoder and decoder for each new channel realization of interest, and meta-learning can alleviate this problem.  
 
\label{sec:5.3.1}
\begin{figure}
	\begin{center}
		\includegraphics[width=0.7\textwidth]{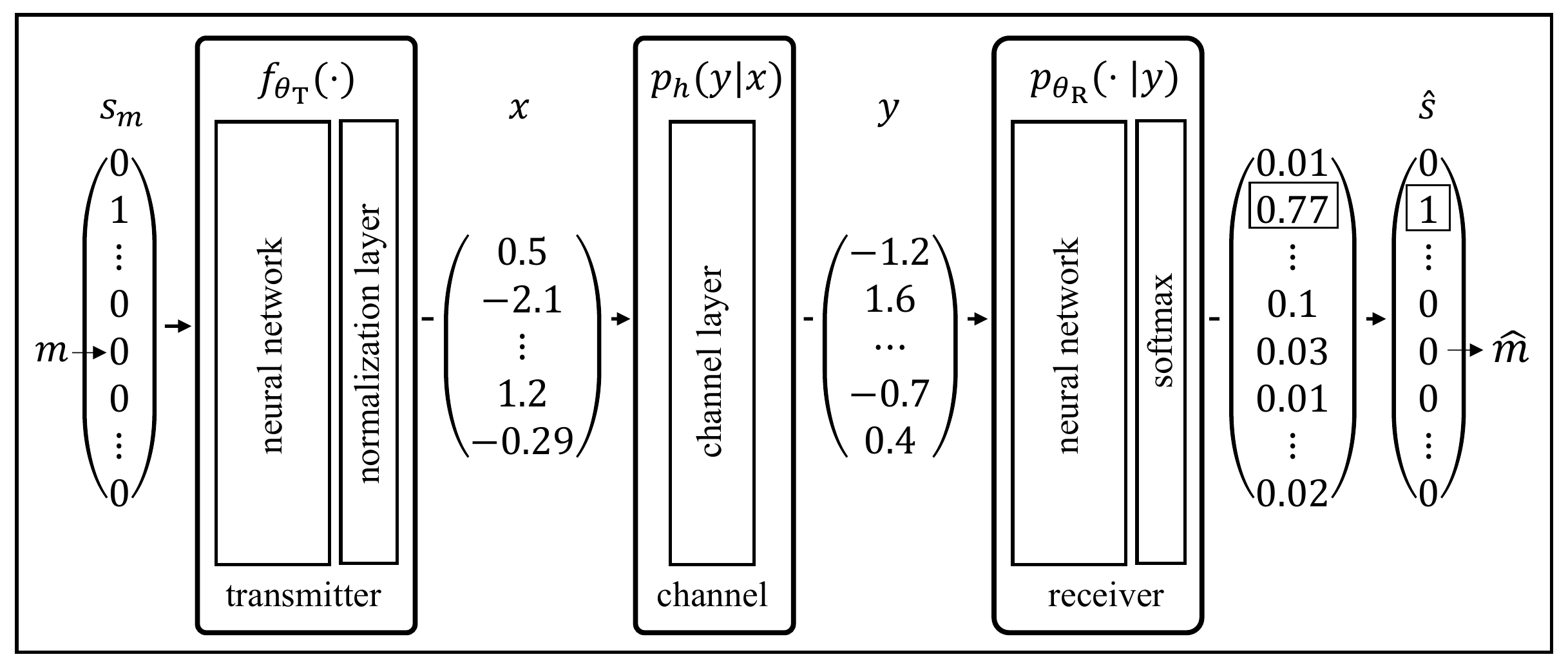}
	\end{center}
	\caption{Meta-learning for encoding and decoding with a known channel model $p_h(y|x)$: A message $m$ is mapped into a codeword $x$ via a trainable encoder $f_{\theta_\text{T}}(\cdot)$, while the received signal $y$, determined by the channel $p_{h}(y|x)$, is mapped into an estimated message $\hat{m}$ through a trainable decoder $p_{\theta_\text{R}}(\cdot|y)$. This setting can be interpreted as modelling a single link as an autoencoder \cite{o2017introduction,cammerer2020trainable, aoudia2021end}.  \label{fig3_1ch5} }
\end{figure} 

The transmitter encodes the message $m$ into the transmitted signal $x$ using a mapping $x=f_{\phi_\text{T}}(s_m)$ where $s_m$ is the $2^k \times 1$ one-hot vector corresponding to message $m$. Signal $x$ is transmitted through a channel described by a known conditional distribution $p_h(y|x)$. Accordingly, the received signal is given as $y\sim p_h(y|x)$, from which the receiver decodes via the stochastic mapping $\hat{m} \sim p_{\phi_\text{R}}(m|y)$. The encoding function $f_{\phi_\text{T}}(\cdot)$ and the decoding operation $ p_{\phi_\text{R}}(\cdot|y)$ depend on model parameter vector $\phi_\text{T}$ and $\phi_\text{R}$, respectively. 

For concreteness, the channel mapping $p_h(y|x)$ is modelled here as
\begin{align}
	\label{eq:ch5:channel_model}
	y = h*x + w, 
\end{align}
where $w \sim \mathcal{CN}(0,N_0)$ represents complex Gaussian i.i.d. noise and ``*'' indicates a linear operation on input $x$ parametrized by a channel vector $h$. The model \eqref{eq:ch5:channel_model} captures frequency selective channels, in which case the operation ``*'' is a convolution; as well as multi-antenna channels, in which case the operation ``*'' is a matrix multiplication. 

\subsubsection{5.3.2 Conventional Learning}
The loss function for particular channel realization $h$ is written as the cross-entropy loss
\begin{align} \label{eq:ch5:encdec}
    L_{h}(\phi) = -\mathbb{E}_{m\sim p(m),y\sim p_h(y|f_{\phi_\text{T}}(s_m))}[\log p_{\phi_\text{R}}(m|y)],
\end{align}
which is averaged over message probability distribution $p(m)$; channel distribution $p_h(y|x)$; and stochastic decoding $p_{\phi_\text{R}}(m|y)$. Here, we have defined the overall model parameter vector $\phi=(\phi_\text{T}, \phi_\text{R})$. Note that the loss $L_h(\phi)$ in \eqref{eq:ch5:encdec} is the population loss, in which the data distribution is determined by the channel $h$. The loss \eqref{eq:ch5:encdec} is approximated by the empirical loss 
\begin{align}
	\label{eq:ch5:empirical_e2e_loss}
	L_{\mathcal{D}_h}(\phi) = -\frac{1}{N}\sum_{j=1}^N \log p_{\phi_\text{R}}(m_j|h*f_{\phi_\text{T}}(s_{m_j}) + w_j),
\end{align}
where the training data set $\mathcal{D}_h$ under channel realization $h$ is generated by drawing i.i.d. random messages $m_1,..., m_N$ from the distribution $p(m)$, along with i.i.d. noise realizations $w_1,...,w_N$. 

Conventional learning addresses the following minimization for each new channel realization $h$:
\begin{align}
	\label{eq:ch5:min_empirical_e2e_loss}
	\min_\phi L_{\mathcal{D}_h}(\phi).
\end{align}
Note that access to a differentiable simulator of the channel model is required for computing the gradient of the loss $L_{\mathcal{D}_h}(\phi)$ with respect to the encoder parameter vector $\phi_\text{T}$. This is trivially true for the simple model \eqref{eq:ch5:channel_model}. 

\subsubsection{5.3.3 Meta-Learning}
A large number of training iterations, consisting of tens of thousands of steps, are generally required for training data-driven encoding and decoding from scratch by solving problem \eqref{eq:ch5:encdec} for each channel realization $h$ of interest \cite{o2017introduction, park2020meta}. Meta-learning can reduce the training time. Using $K$ different channel realizations $h_1,...,h_K$, the frequentist meta-learning problem can be formulated as the minimization 
\begin{align}
	\label{eq:ch5:meta_obj_e2e_with_channel_empirical}
	\min_{\theta} \left\{ \mathcal{L}_{\mathcal{D}^\text{mtr}}(\theta) =  \frac{1}{K}\sum_{k=1}^K L_{\mathcal{D}_{h_k}} (\phi^\text{ma}(\mathcal{D}_{h_k}|\theta) ) \right\},
\end{align}
where the trained model $\phi^\text{ma}(\mathcal{D}_{h}|\theta)$ for each channel realization $h$, given the hyperparameter vector $\theta$, is taken here to be the MAML one-step-gradient update \eqref{eq:pertask_para_ma}, i.e.,
\begin{align}
	\label{eq:ch5:local_e2e_with_channel}
	\phi^\text{ma}(\mathcal{D}_{h_k}|\theta) = \theta - \alpha \nabla_\theta L_{\mathcal{D}_{h_k}}(\theta).
\end{align}
The empirical losses $L_{\mathcal{D}_{h_k}}(\cdot)$ in \eqref{eq:ch5:meta_obj_e2e_with_channel_empirical}--\eqref{eq:ch5:local_e2e_with_channel} are defined as in \eqref{eq:ch5:empirical_e2e_loss}, with $N^\text{tr}$ and $N^\text{va}$ used in lieu of $N$, respectively.

We next provide some numerical results for a frequency selective Rayleigh block fading channel model. More details can be found in \cite{park2020meta}. We assume transmission of $k=4$ bits through $n=4$ complex channel uses. The channel $h$ has three taps, each independently generated as a $\mathcal{CN}(0, 1/3)$ variable. The performance of the trained encoder-decoder pair is measured in terms of block error rate (BLER), i.e.,e\begin{align}
	\label{eq:ch5:bler_def}
	\text{BLER} = \mathbb{E} [  \hat{m} = m],
\end{align}
where the average is taken with respect to channel distribution $p(h)$, message probability distribution $p(m)$, channel distribution $p_h(y|x)$, and stochastic decoder $p_{\phi_\text{R}}(m|y)$. 

Fig.~\ref{fig4ch5} shows the BLER as a function of number of iterations used to train the encoder-decoder pair. Encoder and decoder are multi-layer neural networks \cite{simeone2022machine}. The figure also shows the performance obtained by adopting a more advanced decoder architecture that utilizes a radio transformer networks (RTN) \cite{o2017introduction}. The RTN applies a filter $w$ to the received signal $y$ to obtain the input $\bar{y} = y * w$ to the decoder as $p_{\theta_\text{R}}(\cdot|\bar{y})$. Aiming at explicitly designing a channel equalizer $w$ through additional neural network, RTN has been reported to generally accelerate the optimization procedure \cite{o2017introduction}. 

Similar to Section~\ref{sec:ch5_demodul}, meta-learning is compared with \emph{(i)} \emph{conventional learning}, which adopts a random initialization; and \emph{(ii)} \emph{joint learning}, which optimizes a single encoder-decoder pair from all the meta-training channels. After a sufficient number of adaptation steps for new channel realizations (around $10,000$), all the schemes achieve a BLER lower than $10^{-3}$, validating the power of data-driven encoding and decoding. However, among all the considered schemes, only meta-learning can reach a BLER near $10^{-3}$ with even a single iteration. This demonstrates that a successful transfer of knowledge from multiple channels via meta-learning can indeed reduce the iteration complexity of designing data-driven encoder-decoder pair. 


\begin{figure}
	\begin{center}
		\includegraphics[width=0.5\textwidth]{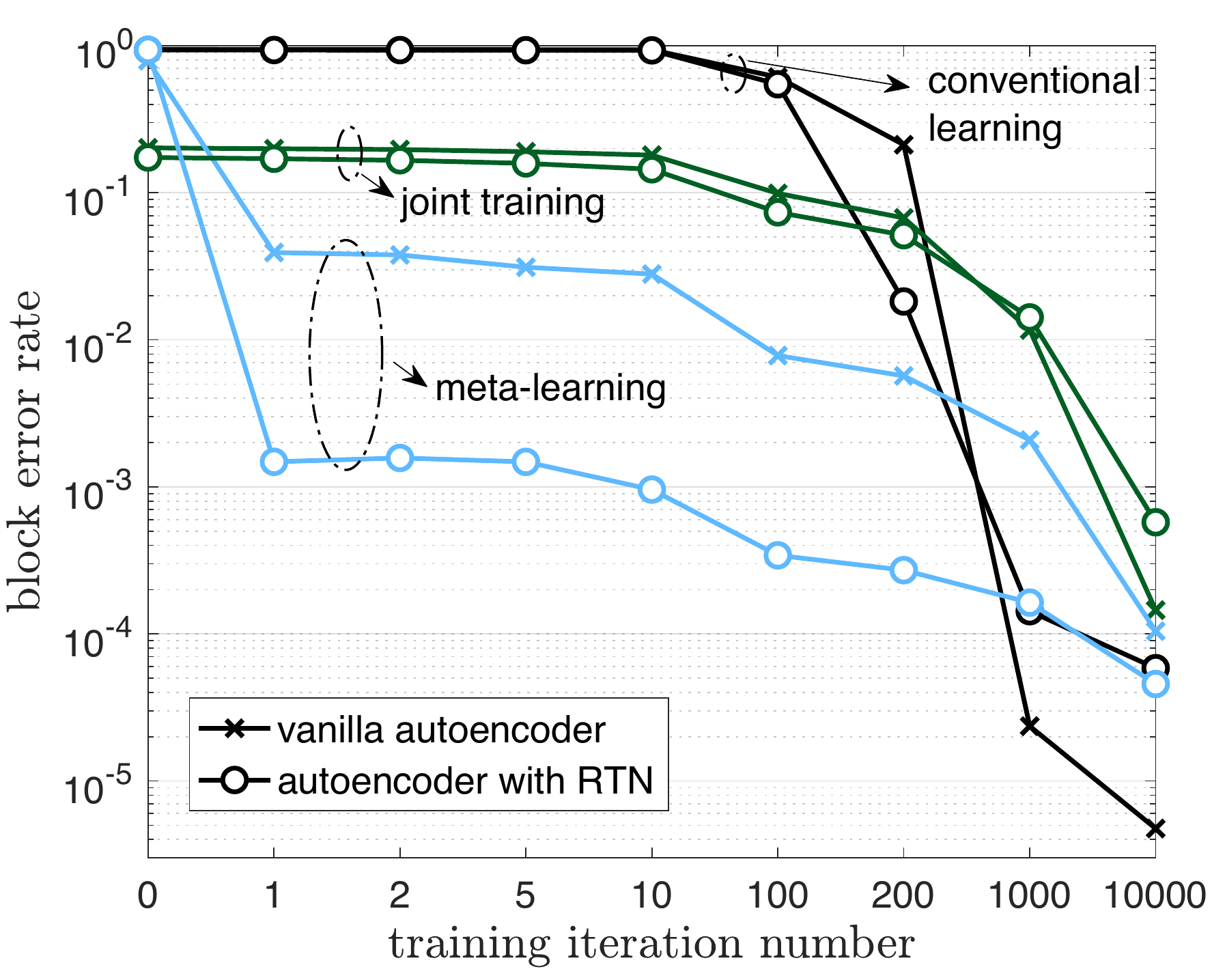}
	\end{center}
	\caption{Meta-learning for encoding and decoding with a (differentiable) channel model: BLER over iteration number for training on the new channel ($4$ bits, $4$ complex channel uses, Rayleigh block fading channel model with $3$ taps, and $16$ messages per iteration, under a $15\text{ dB}$ SNR, adapted from \cite{park2020meta}).\label{fig4ch5} }
\end{figure}


\section{Channel Prediction}
Channel prediction has many applications in modern communication systems, including proactive resource allocation \cite{nikoloska2022modular, agrawal2005iterative}. Deep learning based nonlinear channel predictors have been proposed through training of recurrent neural networks \cite{liu2006recurrent}, convolutional neural networks \cite{yuan2020machine}, and multi-layer perceptrons \cite{kim2020massive}. However, several studies, including \cite{jiang2019comparison, jiang2020long, kim2020massive}, have reported that deep learning based predictors tend to require large training data sets, while failing to outperform well-designed linear filters in the low-data regime. Following \cite{park2022predicting}, this subsection introduces linear \textbf{data-driven channel predictors} that effectively use the available training data via meta-learning. The key idea is to use the linear version of iMAML introduced in Section~\ref{subs:imaml_RR}, along with suitable dimensionality reduction methods via long-short term channel decomposition as proposed in \cite{simeone2004lower, cicerone2006channel, abdi2002space}.

\begin{figure}
	\begin{center}
		\includegraphics[width=0.7\textwidth]{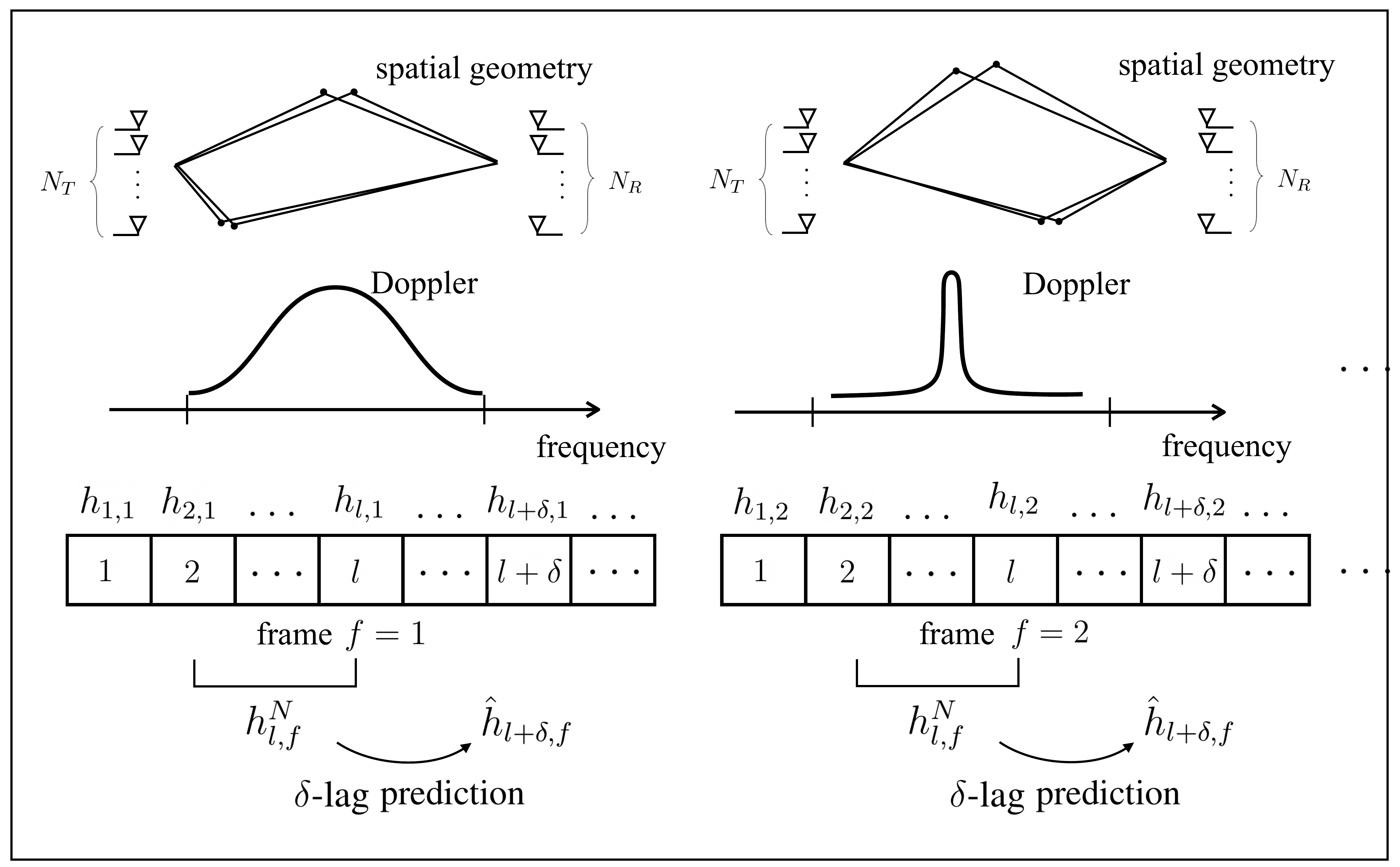}
	\end{center}
	\caption{Meta-learning for channel prediction: At any frame, characterized by generally different channel statistics, the problem of interest is to predict channel using previous consecutive channels.\label{fig6ch5} }
\end{figure}

\subsubsection{5.4.1 Problem Definition}
As shown in Fig.~\ref{fig6ch5}, we consider a wireless communication system in which both the spatial geometry and Doppler spectrum of the wireless channel may change at each frame. Each frame consists of multiple slots. Assuming $N_T$ transmit antennas, $N_R$ receive antennas, and $W$ taps, describing the delay spread of the channel, the complex channel vector at slot $i$ in frame $k$ can be written as $h_{i,k} \in \mathbb{C}^S$ with $S=N_R N_T W$. During any frame, the channel statistics are assumed to be static, while the channels vary across different slots within the same frame with the given frame statistics. 

Within each frame $k$, the channel predictor takes as input the $L$ previous channels 
\begin{align}
	H^L_{i,k} =  [h_{i,k}, ..., h_{i-L+1,k}] \in \mathbb{C}^{S \times L}
\end{align}
to predict the channel $h_{i+\delta,k}$ at a time lag of $\delta$ time steps via the linear predictor as
\begin{align}
	\label{eq:ch5:predicted_ch}
	\hat{h}_{i+\delta,k}(\phi_k) = \phi_k^\dagger \text{vec}(H^L_{i,k}),
\end{align}
where $\phi_k \in \mathbb{C}^{SL \times S}$ is the model parameter vector. In \eqref{eq:ch5:predicted_ch}, $\text{vec}(\cdot)$ is the vectorization operator that stacks the columns of the input matrix into a column vector.

\subsubsection{5.4.2 Conventional Learning}
Defining training data set $\mathcal{D}_k$ for the $k$-th frame with $N+L+\delta-1$ consecutive channel vectors, i.e., $\mathcal{D}_k = \{h_{1,k},..., h_{N+L+\delta-1,k}\}$, the corresponding loss function given the linear regressor $\phi$ is defined as the mean squared error (MSE)
\begin{align}
	\label{eq:ch5:LS_channel_prediction}
	L_{\mathcal{D}_k}(\phi) = \frac{1}{N}\sum_{i=1}^{N} \left\| \hat{h}_{i+\delta,k}(\phi) - h_{i+\delta,k} \right\|^2.
\end{align}
The linear channel predictor $\phi_k$ for the frame $k$ is optimized by addressing the minimization of the training loss \eqref{eq:ch5:LS_channel_prediction}.

\subsubsection{5.4.3 Meta-Learning}
To enable meta-learning, we introduce a bias vector $\theta$ that modifies the training objective in \eqref{eq:ch5:LS_channel_prediction} by adding an $l_2$ regularization term as discussed in Section~\ref{subs:imaml_RR}, i.e.,
\begin{align}
	\label{eq:ch5:ridge_channel_prediction}
	\min_{\phi} \left\{L_{\mathcal{D}_k}(\phi) + \frac{\lambda}{2}\left\|\phi - \theta\right\|^2\right\}.
\end{align}
Furthermore, we assume the availability of a meta-training data set obtained from $K$ previous frames. For each frame $k$, we have channels from $N^\text{tr}+ N^\text{va}+L+\delta-1$ slots, forming the training data set $\mathcal{D}_k^\text{tr} = \{h_{1,k},..., h_{N^\text{tr}+L+\delta-1,k}\}$ and the validation data set $\mathcal{D}_k^\text{va} = \{h_{N^\text{tr}+1,k},..., h_{N^\text{tr}+N^\text{va}+L+\delta-1,k}\}$. The bias vector is meta-learned using iMAML as described in Section~\ref{subs:imaml_RR}. This leads to
\begin{align}
	\label{eq:ch5:meta_obj_channel_pred}
	\min_\theta \left\{ \mathcal{L}_{\mathcal{D}^\text{mtr}}(\theta) = \frac{1}{K}\sum_{k=1}^K L_{\mathcal{D}_k^\text{va}}({\phi}^\text{im}(\mathcal{D}_k^\text{tr}|\theta)) \right\}, 
\end{align}
where the linear channel predictor ${\phi}^\text{im}(\mathcal{D}_k^\text{tr}|\theta)$ for frame $k$ is the solution of problem \eqref{eq:ch5:ridge_channel_prediction} using training set $\mathcal{D}_k^\text{tr}$, i.e., 
\begin{align}
	\label{eq:ch5:inner_channel_pred}
	{\phi}^\text{im}(\mathcal{D}_k^\text{tr}|\theta) = \argmin_{\phi} \left\{L_{\mathcal{D}_k^\text{tr}}(\phi) + \frac{\lambda}{2}\left\|\phi - \theta\right\|^2\right\}.
\end{align}
Both the linear channel predictor ${\phi}^\text{im}(\mathcal{D}_k^\text{tr}|\theta)$ and the solution of problem \eqref{eq:ch5:meta_obj_channel_pred} can be obtained in a closed form as described in Section~\ref{subs:imaml_RR} (by taking $\text{vec}(H_{n,k}^L)^\dagger$ in lieu of $x_{k,n}^\top$ and $h_{n+\delta,k}^\dagger$ instead of $y_{k,n}$).

When the dimension of the channel vector $S$ is large, the meta-learned bias vector $\theta$ obtained from \eqref{eq:ch5:meta_obj_channel_pred} is prone to meta-overfitting. Instead of using the channel vector directly, reference \cite{park2022predicting} proposes to decompose the channel vector into long-term space-time features $B_k \in \mathbb{C}^{S\times R}$ and short-term fading amplitude vector $d_{l,k} \in \mathbb{C}^{R \times 1}$ \cite{simeone2004lower, cicerone2006channel, abdi2002space}. This yields the decomposition
\begin{align}
	\label{eq:ch5:LSTD}
	h_{l,k} = B_k d_{l,k} = \sum_{r=1}^R b_k^r d_{l,k}^r ,
\end{align}
in which $R$ stands for the effective number of resolvable paths for the channel vector; $d_{l,k}^r \in \mathbb{C}$ for the $r$-th element of the vector $d_{l,k}$; and $b_k^r \in \mathbb{C}^{S \times 1}$ is the $r$-th column of the matrix $B_k$. 
The integer $R$ can be estimated by utilizing the previous channel vectors by using a standard method such as Akaike's information theoretic criterion (AIC) \cite{wax1985detection}, or by examining the meta-validation loss \cite{park2022predicting}. The long-term matrix $B_k$ is assumed to have negligible variations within a frame, while only the fading amplitudes change from slot to slot. The channel predictor $\phi_k$ is similarly decomposed in order to reduce the number of parameters to be trained \cite{park2022predicting}.

\begin{figure}
	\begin{center}
		\includegraphics[width=0.5\textwidth]{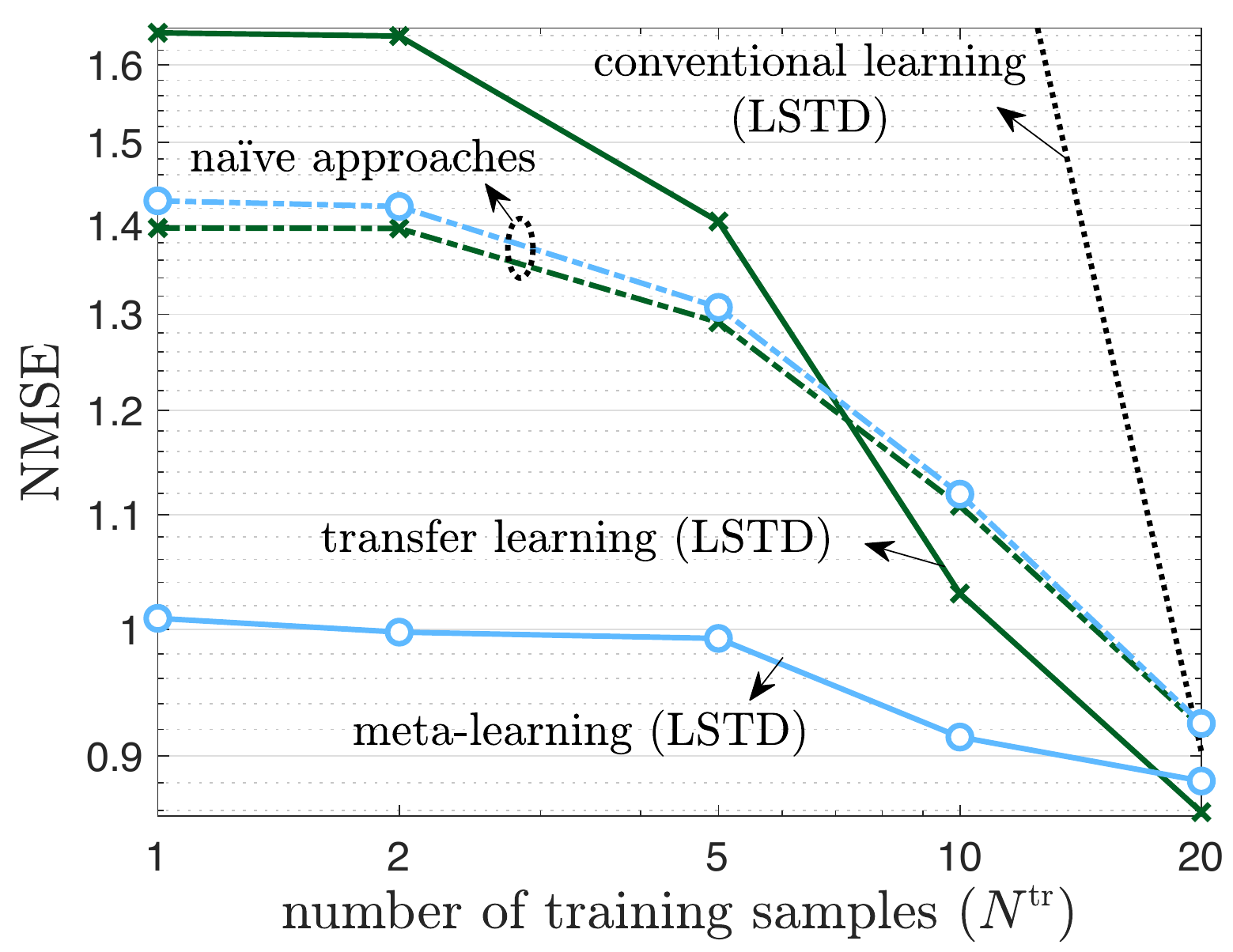}
	\end{center}
	\caption{Meta-learning for channel prediction: Multi-antenna frequency-selective channel prediction performance as a function of the number of training samples, under $19$-clustered, two-tap, and multi-antenna ($N_T=4$ transmit, $N_R=2$ receive antennas) 3GPP SCM channel model (adapted from \cite{park2022predicting}).\label{fig8ch5} }
\end{figure}

We now provide numerical results using the 3GPP spatial channel model (SCM) \cite{3gpp_tr_901} with $N_R=2$, $N_T=4$, and $W=2$. Fig.~\ref{fig8ch5} shows the normalized test MSE (NMSE) as a function of number of training samples $N^\text{tr}$. The NMSE is defined as the normalization with respect to the target channel vector $||\hat{h}_{l+\delta,k}(\phi) - {h}_{l+\delta,k}||^2 / ||h_{l+\delta,k}||^2 $. The performance of the meta-learned channel predictor using the decomposition \eqref{eq:ch5:LSTD} is compared with: \emph{(i)} meta-learning via \eqref{eq:ch5:meta_obj_channel_pred}; and \emph{(ii)} a joint learning solution that finds a bias vector $\theta$ by solving
\begin{align}
	\label{eq:ch5:transfer_obj_channel_pred}
	\min_\theta \frac{1}{K}\sum_{k=1}^K L_{\mathcal{D}_k}(\theta), 
\end{align}
with or without decomposition \eqref{eq:ch5:meta_obj_channel_pred}. In \eqref{eq:ch5:transfer_obj_channel_pred}, the data set $\mathcal{D}_k$ is union of the training data part $\mathcal{D}_k^\text{tr}$ and the validation part $\mathcal{D}_k^\text{va}$.  We refer in the figure to the schemes based on decomposition \eqref{eq:ch5:LSTD} as \emph{long-short-term decomposition (LSTD)}; while schemes without the decomposition are labelled as \emph{na\"ive} schemes. 

In Fig.~\ref{fig8ch5}, meta-learning based on the considered decomposition \eqref{eq:ch5:LSTD} outperforms all the other schemes by transferring useful knowledge for both long-term and short-term features based on channels obtained from multiple frames with different channel statistics.





\section{Power Control}
\label{sec:ch5.5:power_control}
Finally, in this subsection, we consider a fundamental radio-resource management problem in wireless networks -- power control. Power control refers to the optimization of the transmission power levels at distributed links that share the same spectral resources. Ideally, the communication engineer would derive an optimal power control solution that minimizes the level of interference in the network in the presence of time-varying channel conditions. Due to the complexity of modern wireless networks, that provide connectivity to devices ranging from sensors and cell phones to vehicles and robots, deriving an explicit optimal power control policy is infeasible. For such settings, data-driven power control method is promising candidates, which is the subject of this subsection. 

\subsubsection{5.5.1 Problem Definition}
\begin{figure*}[tbp]
\centering
\includegraphics[width=0.86\linewidth]{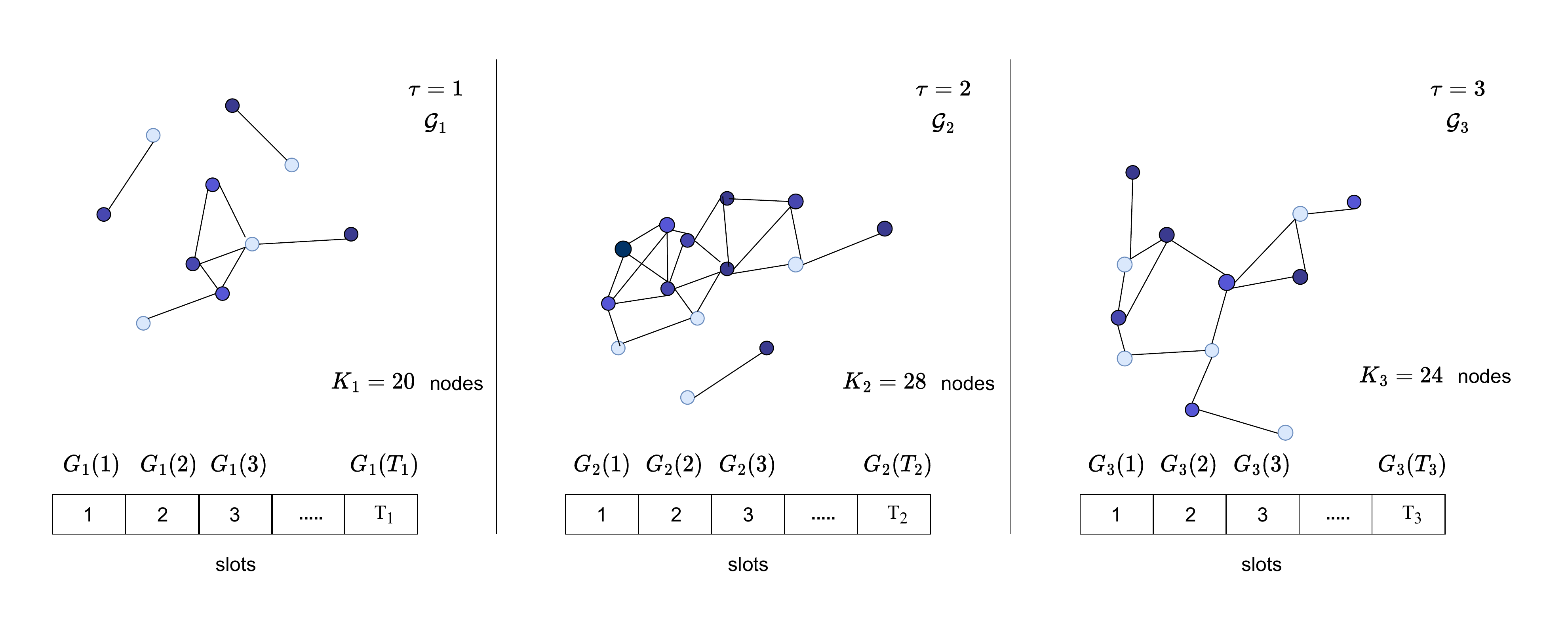}
\vspace*{-3mm}
\caption{Meta-learning for power control: In dynamic networks running over three periods $\tau = 1$, $\tau = 2$, and $\tau = 3$, the goal is to adapt the power control policy to each new topology using a few data-samples. }
\label{sys_mod_pc}
\end{figure*}
As shown in Fig.~\ref{sys_mod_pc}, we consider power control in complex networks with time-varying network topologies. In such dynamic networks, data-driven techniques based on fully connected deep-learning models entail training a different model whenever the number of devices changes, as such models commit to input and output layers of fixed sizes. In contrast, learning with inputs and outputs of variable size can be done using geometric models, such as \textbf{graph neural networks} (GNNs). 

GNNs have been introduced to address the problem of power control in  \cite{eisen2020optimal}. A GNN can encode information about the topology of a network through its underlying graph. Furthermore, the edge weights of the GNN \cite{eisen2020optimal}, are tied to the current channel realizations. As a result, the solution -- which is referred to as \textbf{random edge GNN} (REGNN) -- automatically adapts to time-varying channel conditions through the edge weights. The design problem consists of training the weights $\phi$ of the \textbf{graph filters}. 


We assume that the network is run over periods $k=1,...,K$, with topology possibly changing at each period $k$. During period $k$, the network is comprised of $V_k$ communication links. Transmissions on the $V_k$ links are assumed to occur at the same time using the same spectrum.  The resulting interference graph $\mathcal{G}_k = (\mathcal{V}_k, \mathcal{E}_k)$ includes an edge $(i, j) \in \mathcal{E}_k$ for any pair of links $i, j \in \mathcal{V}_k = \{ 1, ..., V_k \}$ with $i \neq j$ whose transmissions interfere with one another. We denote by $\mathcal{N}^i_k \subseteq \mathcal{V}_k$ the subset of links that interfere with link $i$ at period $k$. Both the number of links $V_k = |\mathcal{V}_k|$ and the topology defined by the edge set $\mathcal{E}_k$ generally vary across periods $k$. 

Each period contains $N$ time slots, indexed by $t = 1,...,N$. In time slot $t$ of period $k$, the channel between the transmitter of link $i$ and its intended receiver is denoted by $h^{i,i}_k (t)$, while $h^{j,i}_k (t)$ denotes the channel between transmitter of link $j$ and receiver of link $i$ with $j \in \mathcal{N}^i_k$. Channels account for both slow and fast fading effects, and, by definition of the interference graph $\mathcal{G}_k$, we have $h^{j,i}_k (t) = 0$ for $j \notin \mathcal{N}^i_k$. The channels for slot $t$ in period $k$ are arranged in the channel matrix $G_k (t) \in R^{V_k \times V_k}$, with the $(j,i)$ entry given by $\left[ G_k(t) \right]_{j,i} = g^{j,i}_k (t) = |h^{j,i}_k (t)|^2$.  
Channel states vary across time slots, and the designer is assumed to have access to channel realizations $\mathcal{D}_k = \{G_k (1), ..., G_k (N)\}$ over $N$ time slots in period $k$ comprising the per-task data set.

With this setup, given transmitted powers $p^i_k(t)$ in each $j$-th link, the achievable sum-rate in slot $t$ of frame $k$ is given by 
\begin{align}\label{eq:ch5:rate_k}
   c_k(p_k(t)) = \sum_{j=1}^{V_k} \log_2 \left(1 + \frac{g^{j,j}_k (t) p^j_k (t)}{\sigma^2 + \sum_{i \in \mathcal{N}_k^j} g^{i,j}_k (t) p^i_k(t)}\right),
\end{align}
where $\sigma^2$ denotes the per-symbol noise power. By \eqref{eq:ch5:rate_k}, interference is treated as worst-case additive Gaussian noise.
As per \cite{eisen2020optimal}, the power allocation vector in \eqref{eq:ch5:rate_k} is parametrized with a REGNN. Given a vector of filters $\phi_k$, this yields
\begin{align}\label{eq:ch5:power_param}
    p_k(t)  = \textrm{f} (G_k(t) \, | \, \phi_k),
\end{align}
where we can find the form of the REGNN function $\textrm{f} (G \, | \, \phi)$ in \cite{eisen2020optimal}.

\subsubsection{5.5.2 Conventional Learning}
Given a set of channel realizations, training of the REGNN parameters is done by tackling the unsupervised learning problem  \cite{eisen2020optimal} 
\begin{align}\label{eq:ch5:opt_phi}
    &\underset{\phi}{\text{min}} \left\{ L_{\mathcal{D}_k}(\phi) = - \frac{1}{N} \sum_{t=1}^{N} c_k(\textrm{f} (G_k(t) \, | \, \phi))\right\},
\end{align}
via SGD.  
Note that, the method in \cite{eisen2020optimal} adopts a joint learning strategy, whereby a single filter tap is optimized for all network configurations, i.e., the optimization in \eqref{eq:ch5:opt_phi} is carried out by summing the rates over all network topologies of interest.


\subsubsection{5.5.3 Black-Box Meta-Learning}
To apply conventional meta-learning, we first split the data set $\mathcal{D}_k$ into training part $\mathcal{D}_k^\text{tr}$ and validation part $\mathcal{D}_k^\text{va}$ as in the previous subsections. Using FOMAML and Reptile, as discussed in Section~\ref{sec:second_order_algorithms}, we aim to maximize the achievable rate in \eqref{eq:ch5:rate_k}, averaged across all tasks as
\begin{align}\label{eq:ch5:meta1}
   \min_{\theta} \left\{ \mathcal{L}_{\mathcal{D}^\text{mtr}}(\theta) = \frac{1}{K}  \sum_{k = 1}^{K} L_{\mathcal{D}_k^\text{va}}(\phi^\text{ma}(\mathcal{D}_k^\text{tr}|\theta)  )  \right\},
\end{align}
where the task-specific parameters $\phi^\text{ma}(\mathcal{D}_k^\text{tr}|\theta)$ are found by taking a single gradient step using the shared parameter $\theta$ as initialization:
\begin{align}\label{eq:ch5:meta}
   &\phi^\text{ma}(\mathcal{D}_k^\text{tr}|\theta) = \theta - \alpha \nabla_\theta L_{\mathcal{D}_k^\text{tr}}(\theta).
\end{align}
The second-order derivatives required to solve \eqref{eq:ch5:meta1} are ignored, and the initialization is computed as in \eqref{eq:outer_update_fo} and \eqref{eq:outer_update_re} for FOMAML and Reptile, respectively. We refer to such meta-learning schemes as ``black-box'', as they do not leverage the modular structure of GNN models.

\subsubsection{5.5.4 Modular Meta-Learning}
Power control has also been tackled in \cite{nikoloska2022modular} using the modular meta-learning method described in Section~\ref{sec:ch2:modular}. To do so, we define a set $\mathcal{M}$ of modules, each representing an instantiation of a REGNN filter. Representing the modules with indices $\mathcal{M} = \{1, ..., M\}$, and considering REGNNs with $L$ layers, each layer $l = 1,...,L$ is assigned one of the $M$ modules.
Accordingly, we introduce the discrete vector $S_k \in \{1,...,M\}^L$ to denote the module assignment which is a mapping between the layers $l = 1,...,L$ of the REGNN and the modules from the set $\mathcal{M}$.

The goal of modular meta-learning is to optimize the shared module set $\mathcal{M}$ so as to allow the system to find a combination of effective modules for any new topology during deployment. This is done by addressing problem
\begin{align}\label{eq:ch5:meta_mod}
    &\min_{\mathcal{M}} \left\{ \mathcal{L}_{\mathcal{D}^\text{mtr}}^\text{mod}(\mathcal{M}) =\frac{1}{K} \sum_{k=1}^K L_{\mathcal{D}_k^\text{va}}(\phi^\text{mod}(\mathcal{D}_k^\text{tr}|\mathcal{M})) \right\},
\end{align}
where the task-specific parameter $\phi^\text{mod}(\mathcal{D}_k^\text{tr}|\mathcal{M})$ is defined by the module set $\mathcal{M}$ and by the corresponding task-specific module assignment vector $S_k(\mathcal{M})$, i.e., $\phi^\text{mod}(\mathcal{D}_k^\text{tr}|\mathcal{M})=\phi^{(S_k(\mathcal{M}))}$ (cf. \eqref{eqn.mod}).
The module assignment vector is adapted per task as 
\begin{align}\label{eq:ch5:task_mod}
    & S_k(\mathcal{M}) = \underset{S \in \{1,...,M\}^L}{\text{argmin}} L_{\mathcal{D}_k^\text{tr}}(\phi^{(S(\mathcal{M}))}).
\end{align}


To tackle the mixed continuous-discrete problem over the module set and the assignment variables in \eqref{eq:ch5:meta_mod},  \cite{nikoloska2022modular} introduces a stochastic module assignment function given by a conditional distribution $\mathcal{P}_k(S_k |  \mathcal{M}, \mathcal{D}^{\text{tr}}_k)$, and reformulate the bi-level optimization problem as
\begin{align}\label{eq:ch5:opt_phi_meta_sto}
    &\underset{\mathcal{M}}{\text{min}}  \,\,\,\, \frac{1}{K}\sum_{k = 1}^{K} \underset{\mathcal{P}_k(S_k \,| \, \mathcal{M}, \mathcal{D}_k^{\text{tr}})}{\text{min}} \mathbb{E}_{S_k \sim \mathcal{P}_k(S_k |  \mathcal{M}, \mathcal{D}^{\text{tr}}_k)} \left[  L_{\mathcal{D}_k^\text{va}}(\phi^{(S_k(\mathcal{M}))}) \right].
\end{align}
In \eqref{eq:ch5:opt_phi_meta_sto}, the inner optimization is over the distributions $\{\mathcal{P}_k(\cdot \,| \, \mathcal{M}, \mathcal{D}_k^{\text{tr}})\}_{k = 1}^{K}$. We refer to \cite{nikoloska2022modular} for implementation details.

\begin{figure*}[tbp]
\centering
\includegraphics[width=0.7\linewidth]{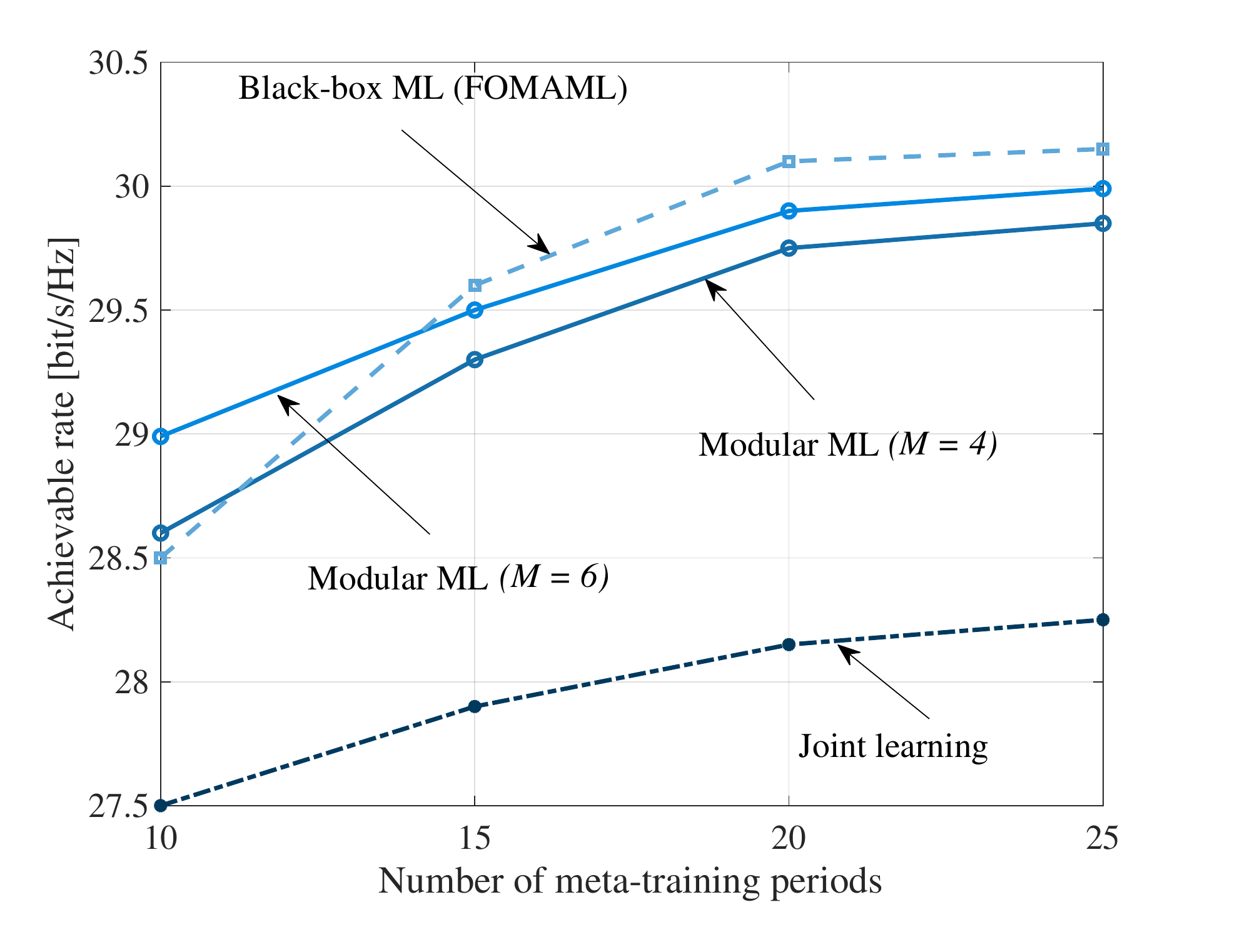}
\vspace*{-3mm}
\caption{Meta-learning for power control: Achievable rate in dynamic networks as a function of number of meta-training periods. The performance of the black-box and modular meta-learning is compared against joint learning (adapted from \cite{nikoloska2022modular}).}
\label{rate_tasks_BB_mod}
\end{figure*}

We now provide some numerical results under independent Rayleigh fading channels. Detailed settings can be found in \cite{nikoloska2022modular}. We compare the meta-learning methods to joint learning as proposed in \cite{eisen2020optimal}, which finds a single parameter vector by solving \eqref{eq:ch5:opt_phi} for the $K$ meta-training periods. We also consider both the black-box, i.e., standard, and modular meta-learning in Fig.~\ref{rate_tasks_BB_mod} by plotting the sum-rate for a network of dynamic size as a function of number of meta-training periods $K$.  

The results in Fig.~\ref{rate_tasks_BB_mod} demonstrate that modular meta-learning is advantageous over black-box methods when the number of meta-training tasks is smaller. However, as the number of meta-training tasks increases, due to the rigidity of modular methods, this gain is overcome by limitations due to bias, and black-box methods are able to achieve larger rates.



      

\section{Conclusions}
This section introduced several applications of meta-learning to wireless communication systems, ranging from demodulation to power control. For more references, we refer to \cite{jiang2019mind} for channel decoding; \cite{goutay2020deep, zhang2021embedding} for MIMO systems; and \cite{hu2021distributed, marini2022continual} for unmanned aerial vehicle (UAV) networks. We finally mention model-based meta-learning which may further reduce the resource overhead in communication systems \cite{raviv2022online, pratik2021neural}. Section~\ref{sec:ch7} contains some discussion on online and model-based meta-learning.

%% file: Chapter6/Chapter6.tex
\chapter{Integration with Emerging Computing Technologies}
\label{sec:ch6}

This section covers the integration of meta-learning with two emerging information processing methods: neuromorphic computing and quantum computing. Both computing technologies promise to improve the efficiency of specific, distinct, classes of processing tasks, while relying on dedicated hardware implementations that move beyond the current von Neumann digital computing architecture. Machine learning can potentially enable applications of both computing technologies to problems of practical interest. Data scarcity is, however, often an issue when training machine learning models implemented using neuromorphic or quantum computing platforms. In fact, both technologies are highly synergistic with specialized input data types that may be in short supply.  It is hence of interest to investigate settings in which meta-learning can enhance sample efficiency, while accounting for the unique properties and constraints of the two computing methods. This section provides a very brief introduction to this, with the main goals of highlighting main conceptual aspects and of providing suitable pointers to the literature.

\section{Neuromorphic Computing}
\label{sec:ch6.1:neuro_compute}
\textbf{Neuromorphic computing} is a \emph{brain-inspired} signal processing paradigm. It excels at tasks involving streaming, sparse,  time series, and/or targeting low-energy, always-on, operation with low-latency responses \cite{mehonic2022brain,davies2021advancing}. Neuromorphic processors implement spiking neural networks (SNNs), which replace the static neurons of classical machine learning with dynamic, spiking, neuronal models that process information in the timing of spikes. The focus on spike-based processing is well aligned with scientific consensus in neuroscience on the key role played by spikes to ensure low-energy, low-latency, and high-accuracy signalling \cite{humphries2021spike}.   With a design that ensures a very low idle energy consumption, the spiking neurons of an SNN can ensure an energy usage level that is proportional to the number of spikes processed.

SNNs  are particularly well suited to analyze data produced by neuromorphic sensors, such as event-driven cameras  and touch sensors \cite{hu2016dvs,lichtsteiner2006128,lee2019neuro}.  Such data consist of time series in which information is encoded in the timing of events recorded by the sensors. For example, event-driven cameras produce a spike at a pixel when the brightness recorded by the pixel crosses a given threshold.


\subsection{Neuromorphic Computing and Machine Learning}

Neuromorphic computing platforms implement SNNs, whose operation is determined by synaptic weights describing the links between spiking neurons as in a standard artificial neural networks. In some applications, the synaptic weights are fixed as a function of the computing task. This is the case, most notably, when the SNN is used to solve convex optimization problems \cite{davies2021advancing,mancoo2020understanding}. In most other applications, however, the synaptic weights are optimized using machine learning tools based on the availability of training data.

Denote as $s_{i,t} \in\{0,1\}$ the output of a neuron at discrete time $t$, with $s_{i,t}=1$ representing the transmission of a spike to all neurons connected to neuron $i$ by synapses stemming out of neuron $i$. Various models can be used to implement the spiking mechanism, with the most commonly adopted for SNNs being the \textbf{spike response model (SRM)}. Under the SRM, in order to decide whether to spike or not,  neuron $i$ at time $t$ applies a threshold function to an internal variable known as its \textbf{membrane potential}, i.e., 
\begin{align}\label{eq:ch6:membrane}
	s_{i,t} = \Theta\big( u_{i,t}-\vartheta \big) \in \{0, 1\},
\end{align}
where $\Theta(\cdot)$ is the Heaviside step function;  $u_{i,t}$ is the membrane potential of neuron $i$ at time $t$; and  $\vartheta$ is a fixed threshold. According to (\ref{eq:ch6:membrane}), 
a spike $s_{i,t} = 1$ is emitted when the membrane potential $u_{i,t}$ crosses a fixed threshold $\vartheta$. The membrane potential evolves over time as a function of the responses of the synapses ending at neuron $i$ to incoming spikes, as well as of the response of the neuron itself to its own spikes. The latter mechanism can implement refractoriness, whereby a neuron tends not to produce spikes too close in time.

\begin{figure}
	\begin{center}
		\includegraphics[clip,scale=0.50]{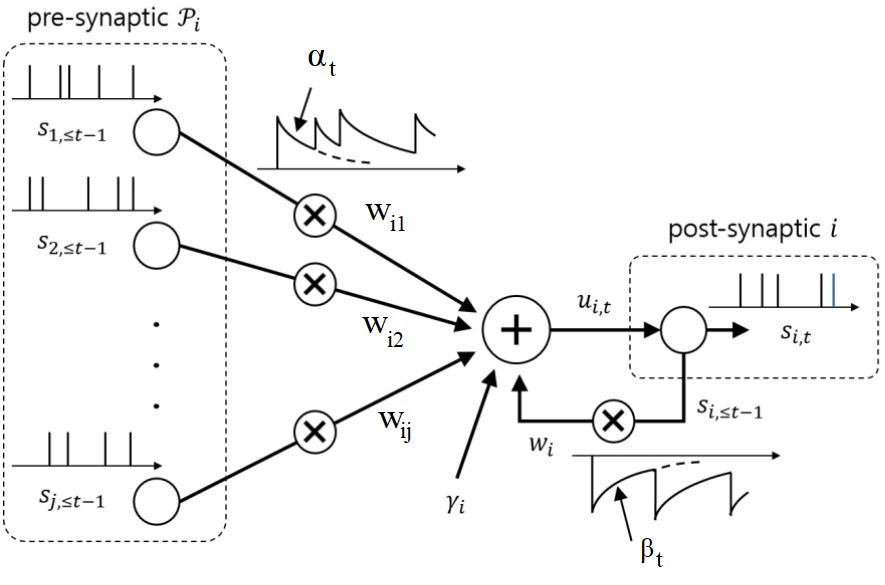}
		\par\end{center}\vspace{-0.4cm}\caption{Illustration of a spiking neuron. \label{fig:spiking}}
\end{figure}

Let us denote as $\mathcal{P}_i$ the set of neurons that have synapses ending at neuron $i$. The SRM stipulates that each such synapse will respond with a waveform $\alpha_t$ -- the impulse response of the synapse -- to each incoming spike.	Mathematically, as illustrated in Fig. \ref{fig:spiking}, the SRM prescribes the following update to the membrane potential of neuron $i$ at time $t$:
\begin{align}
	u_{i,t} &=  \underbrace{{\sum_{j \in \mathcal{P}_i} w_{ij} \big( \alpha_t \ast s_{j,t} \big)}}_{\text{pre-synaptic}} + \underbrace{\big( \beta_t  \ast s_{i,t}\big)}_{\text{post-synaptic}},
\end{align}where $\ast$ denotes the convolution operator. 
In this update, the contribution of pre-synaptic neurons  depends on the synaptic filter $\alpha_t$ through a learnable synaptic weights $w_{ij}$. Furthermore, the post-synaptic contribution of the spikes emitted by neuron $i$ is mediated through the feedback filter $\beta_t$. The duration of the synaptic filter $\alpha_t$ determines the memory of the synaptic response, while the duration of the feedback filter  $\beta_t$ dictates the effective length of refractory periods.

Focusing on \textbf{supervised learning}, we assume that the data set encompasses a target signal $x_{i,t}\in \{0,1\}$ for a subset $\mathcal{X}$ of neurons. In practice, the supervisory signals may be provided sequentially over time $t$, and hence training may take place \textit{online} as time index $t$ increases.  Accordingly, the training loss can be expressed as a sum of local losses $\ell(x_{i,t}, s_{i,t})$ evaluated on  each neuron $i \in \mathcal{X}$ over time $t=1,\ldots,T$, for some interval of time $T$, as 
\begin{align}\label{eq:ch6:nloss}
	\mathcal{L}(\theta) = \sum_{t=1}^T \sum_{i \in \mathcal{X}} \ell(x_{i,t}, s_{i,t}),
\end{align}
where each loss term $\ell(x_{i,t},s_{i,t})$ depends on the target output $x_{i,t}$ of neuron $i$ at time $t$ and on the actual outputs $s_{i,t}$. Since an SNN following the SRM neuronal model can be viewed as an \textit{recurrent neural network}, the training loss (\ref{eq:ch6:nloss}) can be, in principle, minimized via  gradient descent, with the gradient being computed via backpropagation over time. 

Denoting as  $\Theta'(\cdot)$ the first derivative of function  $\Theta(\cdot)$,	the general form of the partial derivative of the loss function (\ref{eq:ch6:nloss}) with respect to a synaptic weight $w_{ij}$ is given by  \begin{align}\label{eq:ch6:ngrad}
	\frac{\partial}{\partial w_{ij}} \mathcal{L}(\theta)=\sum_{t=1}^T \underbrace{e_{i,t}}_{\text{error signal}} \cdot \underbrace{\Theta'(u_{i,t}-\vartheta)}_{\text{post}_{i, t}} \cdot \big( \underbrace{\alpha_t \ast s_{j,t}}_{\text{pre}_{j, t}} \big),
\end{align}where: \begin{itemize}
	\item  $\text{pre}_{j,t} = \alpha_t \ast s_{j,t}$ is the \textbf{pre-synaptic trace}, which is 	large if the previous behavior of pre-synaptic neuron originating the synapse is consistent with synaptic receptive field of the synapses described by filter $\alpha_t$. For instance,  if  $\alpha_t$ decreases over time, the trace tends to large if the pre-synaptic neuron has spiked recently. 
	\item $\text{post}_{i,t} = \Theta'(u_{i,t} - \vartheta)$ is the  \textbf{post-synaptic term}, which measures 
	the ``sensitivity'' to changes in the membrane potential  of  post-synaptic  neuron $i$.
	\item  $e_{i,t} $ is \textbf{per-neuron error signal}, which is ideally evaluated via backpropagation through time as a function of the loss functions $\{\ell(x_{k,t}, s_{k,t})\}_{k \in \mathcal{X}}$ computed by the neurons $k\in \mathcal{X}$.
\end{itemize}

Using the partial derivative (\ref{eq:ch6:ngrad}), an online gradient descent rule can be implemented over discrete time $t$ as \begin{equation}\label{eq:ch6:n3fac}
	w_{ij} \leftarrow w_{ij} - \eta \underbrace{e_{i,t}}_{\text{error signal}} \cdot \underbrace{\Theta'(u_{i,t}-\vartheta)}_{\text{post}_{i, t}} \cdot \big( \underbrace{\alpha_t \ast s_{j,t}}_{\text{pre}_{j, t}} \big),
\end{equation} where $\eta>0$ is a learning rate. The synaptic update (\ref{eq:ch6:n3fac}) is an example of a \textbf{three-factor update rule}, whereby each synaptic weight is modified based on local information, in the form of the pre-synaptic and post-synaptic factors, as well as based on a per-neuron feedback signal. Accordingly, the update (\ref{eq:ch6:n3fac}) can be implemented at each synapse using \emph{locally} available information, in addition to the error signal, which requires feedback from the network, as discussed next.

Calculation of the gradient in (\ref{eq:ch6:ngrad}), and hence application of the three-factor rule (\ref{eq:ch6:n3fac}),  face two practical challenges:
\begin{itemize}
	\item \textbf{Credit assignment}: The impact of every synaptic weight propagates through neurons and time, and hence the calculation of the error signal $e_{i,t}$, generally requires backpropagating errors $\{\ell(x_{k,t}, s_{k,t})\}_{k \in \mathcal{X}}$ across the entire network and over all previous time instants $t'<t$. This problem is typically solved by approximating  backpropagation  through \textbf{truncated backprop} through time, possibly limited to a single time step,  and through \textbf{random feedback alignment}. Random feedback alignment computes the errors $e_{i,t}$ as a random function of the loss values  $\{\ell(x_{k,t}, s_{k,t})\}_{k \in \mathcal{X}}$.
	\item \textbf{Non-differentiability}: The activation function $\Theta(\cdot)$ is such that the derivative  $\Theta'(\cdot)$ is zero almost everywhere. To address this problem, the typical solution applies \textbf{surrogate gradient} methods, whereby the derivative $\Theta'(\cdot)$ is replaced with the derivative of a differentiable surrogate function, such as  sigmoid function.
\end{itemize}
We refer to \cite{neftci2019surrogate,jang2019introduction}  for additional discussion on gradient descent-based training of SNNs.

\subsection{Neuromorphic Computing and Meta-Learning}

Research in neuroscience has revealed learning mechanisms that operate at different time scales, with slower learning procedures targeting the acquisition of new skills and tasks \cite{lindsay2021models}. Through such outer, slower, learning loops, biological brains can acquire general concepts and methods, allowing a more efficient adaptation to specific activities or tasks \cite{karni1998acquisition, martin2000synaptic}. In this process, a variety of update techniques are at work to establish short-to-intermediate-term and long-term memory for the acquisition of new information over time, such as  \textit{long-term potentiation},  \textit{metaplasticity}, and \textit{heterosynaptic plasticity}. We refer to \cite{soures21tacos} for an overview. Meta-learning and continual learning for SNNs implement solutions that inspired by such mechanisms  \cite{soures21tacos, kudithipudi22lifelong}. In particular,   the three-factor rule (\ref{eq:ch6:n3fac}) can be directly built on to implement first-order meta-learning schemes such as FOMAML (see Section 2). We refer to \cite{rosenfeld2021fast} for details and results. 
%

\section{Quantum Computing}
\label{sec:ch6.2:quantum_compute}

Conceived in $1982$ by physicist Paul Benioff, and named after the subatomic physics it aims to harness, quantum computing is based on the concept of a \textbf{qubit}. A qubit is a quantum-mechanical system
that can represent the classical states,  $0$ and $1$ of a classical bit, as well as any superposition of both states \cite{benioff1982quantum}. The complex amplitudes defining a quantum state in superposition can mutually interfere, and they can define forms of correlation across multiple qubits, referred to as entanglement, with no classical counterpart. A quantum computer can be understood as a physical implementation of a number of interacting qubits with a precise control on the temporal evolution of the joint state of the qubits. Any quantum state evolution can be approximated by a sequence of a handful of elementary ``controls'', called \textbf{quantum gates}, which only act on one or two qubits at a time.  As a result, a \textit{universal} quantum computer only has to perform a small set of operations on qubits, much like classical computers are built on a limited number of logic gates. 

Examples of physical implementations of quantum computers involve the polarizations of photons, the discrete energy levels of an ion, the nuclear spins states of an atom, and the spin states of an electron.  Recent demonstrations of the potential of quantum computing based on such technologies have catalysed a booming activity in the field \cite{arute2019quantum}. At the time of writing, quantum computers have reached beyond the realm of a purely academic interest, and they appear to be at the critical point of becoming widely available for the commercial and scientific uses.

\subsection{Quantum Computing and Machine Learning}
A number of elementary quantum gates can be controlled via the selection of a vector $\phi$ of parameters. A quantum gate implements a linear, unitary, transformation of a quantum state. For a parameterized quantum gate, such unitary transformation is typically a function of rotation angles that make up vector $\phi$.  A sequence of parameterized and fixed quantum gates  gives rise to the workhorse of quantum machine learning -- the \textbf{parametrized quantum circuit (PQC)}. A PQC is often implemented using a so-called hardware-efficient \textit{ansatz} (i.e., model architecture), in which a layer of one-qubit unitary gates, parametrized by vector $\phi$, is followed by a layer of fixed, entangling, two-qubit gates. 

A PQC can be used to process and output classical or quantum data. Quantum data refers to quantum-mechanical systems encoding information in their quantum states. \textbf{Quantum data} may be produced by quantum sensors, which are emerging as important tools in various scientific fields \cite{degen2017quantum}. To extract \textbf{classical information} from a PQC, the state of the qubit register is measured, producing classical bits.

In quantum machine learning, for both cases of classical and quantum data, the parameters $\phi$ of a PQC are optimized in a data-dependent manner via a classical optimizer that keeps the PQC in the loop as shown in Fig. \ref{fig:ch1:qmlintro}. The classical optimizer receives measurement outputs from the PQC, and aims at updating the PQC parameters $\phi$ with the aim of optimizing a data-dependent cost function. Such optimization is typically done using standard methods like gradient descent.

The quantum machine learning architecture of Fig. \ref{fig:ch1:qmlintro} has a number of potential advantages over the traditional approach of handcrafting quantum algorithms assuming fault-tolerant quantum computers: \begin{itemize} 
	\item By keeping the quantum computer in the loop, the classical optimizer 
	can directly account for the non-idealities and limitations of quantum operations via measurements of the output of the quantum computer.
	\item If the PQC is sufficiently flexible and the classical optimizer sufficiently effective, the approach may automatically design well-performing quantum algorithms that would have been hard to optimize by hand via traditional formal methods.
\end{itemize}

\begin{figure}
	\begin{center}
		\includegraphics[clip,scale=0.30]{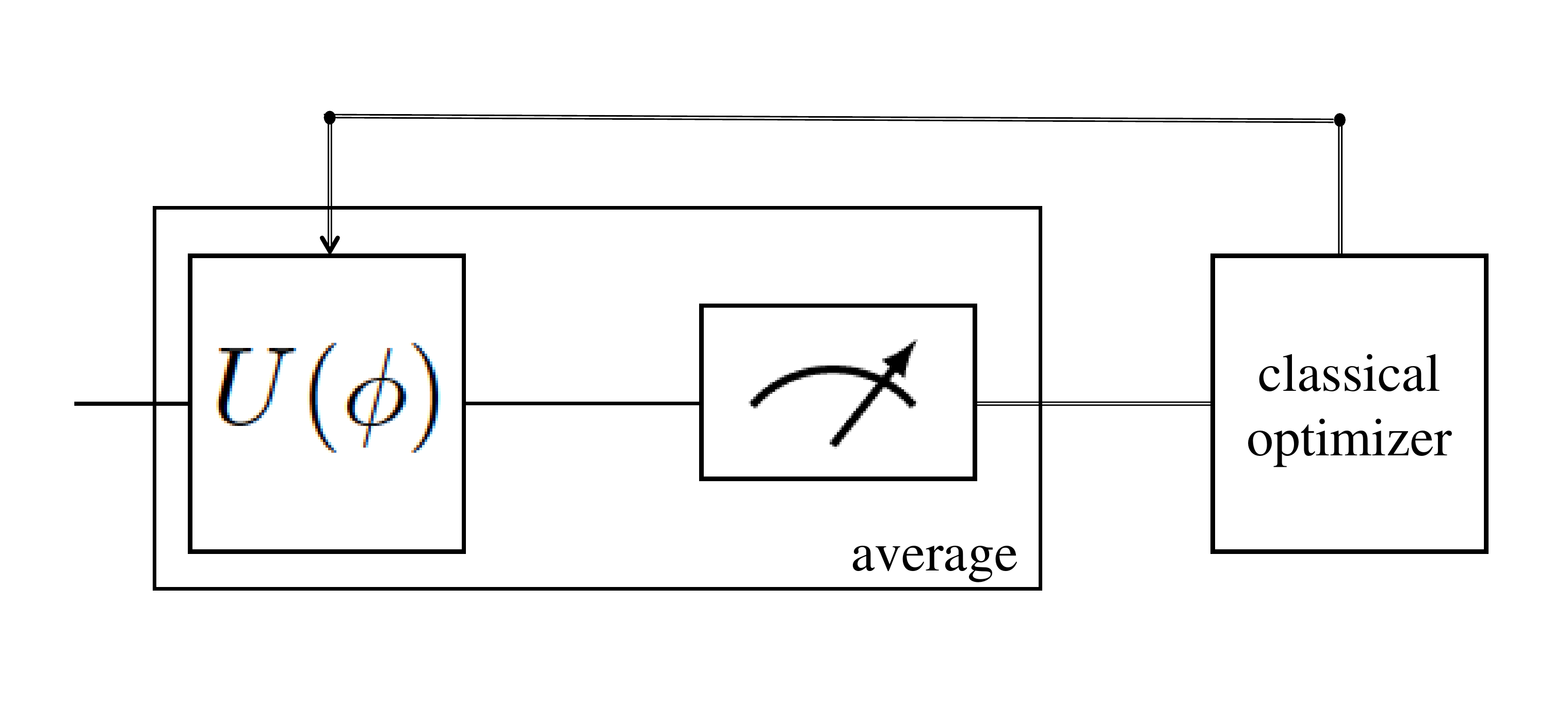}
		\par\end{center}\vspace{-0.4cm}\caption{Illustration of the quantum machine learning design methodology: A PQC with a pre-specified architecture is optimized via its vector of parameters, $\phi$, by a classical optimizer based on data and measurements of its outputs. The operation of a parametrized quantum circuit is defined by a unitary matrix $U(\phi)$ dependent on vector $\phi$. The block marked with a gauge sign represents quantum measurements, which convert quantum information produced by the quantum circuit into classical information. This conversion is inherently random, and measurement outputs are typically averaged before being fed to the classical optimizer. \label{fig:ch1:qmlintro}}
\end{figure}

\subsection{Quantum Machine Learning and Meta-Learning}
The integration between quantum machine learning and meta-learning can take two distinct forms, with the former supporting the latter or vice versa. 


\subsubsection{Classical Meta-Learning for Quantum Machine Learning}
Classical meta-learning algorithms as presented in this monograph can be leveraged to make the optimization of the PQC parameters $\phi$ more sample- or iteration-efficient. With this class of methods, the classical optimizer in Fig. \ref{fig:ch1:qmlintro} operates at two time scales, with the slower time scale processing data from multiple, related, meta-learning tasks. Classical neural network architectures, such as recurrent neural networks, can be meta-trained  to produce the PQC parameters $\phi$ in a more efficient manner than in the conventional case in which classical optimization applies separately to each learning task. We refer to \cite{wilson2021optimizing, verdon2019learning} for details and results.

\subsubsection{Quantum Machine Learning for Classical Meta-Learning}
Conversely, quantum machine learning models can be leveraged to enhance the performance of meta-learning for classical machine learning models. PQCs are particularly efficient as generative models that produce binary strings with complex joint distributions as the results of measurements at their outputs.  This suggests the use of PQCs to model variational distributions $q(\phi)$ in Bayesian meta-learning (see Section \ref{subs:bayes_maml}). 


To illustrate the idea of using quantum machine learning to aid classical meta-learning, consider the problem of training binary neural networks parameters' $\phi_k$ via Bayesian learning. 
The variational distribution $q(\phi_k)$ of the neural network's parameters $\phi_k$ is modelled implicitly via the output of the measurements of a PQC. Specifically, such measurements produce random binary strings $\phi_k \in\{0,1\}^{n}$, where $n = |\phi_k|$ denotes the total number of model parameters.  Importantly, such quantum models only provide samples, while the actual distribution of the measurements' outputs can only be estimated by averaging multiple measurements of the PQC's outputs. Therefore, PQCs model \textbf{implicit distributions}, and only define a stochastic procedure that directly generates samples for the model parameters $\phi_k$. 

Training from scratch for each task is thereby inefficient in terms of sample and iteration complexity and meta-learning alleviates these issues of optimizing the PQC. We refer to \cite{nikoloska2022quantum} for details and results.

\section{Conclusions}

This section has drown some connections between meta-learning and emerging computing technologies, which may play an important role in future machine learning systems. This is an active area of research, and more open problems will be reviewed in the next section.

%% file: Chapter7/Chapter7.tex
\chapter{Outlook}
\label{sec:ch7}

This monograph has provided an introduction to meta-learning by surveying methods, theory, and application. The topic of meta-learning is currently the subject of intense research in different disciplines, including information theory,  machine learning, hardware design, and neuroscience. In this final section, we provide an outlook of directions for research that have not been covered in the text and appear to be particularly promising and challenging at the time of writing. We specifically focus on aspects of interest for researchers in signal precessing.

\section{Methods}
\label{sec:ch7.1:methods}

In this subsection, we highlight research topics concerning the development of meta-learning methods.

\subsection{Continual (Online) Meta-Learning}
\label{sec:ch7:online_meta}
The conventional formulation of meta-learning studied in this monograph assumes the availability of meta-training data set collected \textbf{offline} from $K$ learning tasks, which is denoted as  $\mathcal{D}^\text{mtr} = \{(\mathcal{D}_{k}^{\text{tr}}, \mathcal{D}_k^\text{te})_{k=1}^{K} \}$. As we have seen in Section~\ref{sec:genbounds_metalearning}, the number of tasks $K$ plays an important role in ensuring successful generalization to new tasks, avoiding meta-overfitting. The meta-training data set may be, for instance, collected by acquiring data sets for similar tasks from existing repositories; or by storing data gathered during previous interactions with similar learning environments. In the latter case, it is natural to consider settings in which the meta-training dataset is built in an \emph{online} fashion by accumulating data observed over time, and updating accordingly the hyperparameter $\theta$. This formulation is known as \textbf{continual, or online  meta-learning} \cite{finn2019online} (see also \cite{simeoneCUP}). Online meta-learning plays an important role also in models for computational intelligence \cite{alberta}.

\begin{figure}
	\begin{center}
		\includegraphics[width=0.7\textwidth]{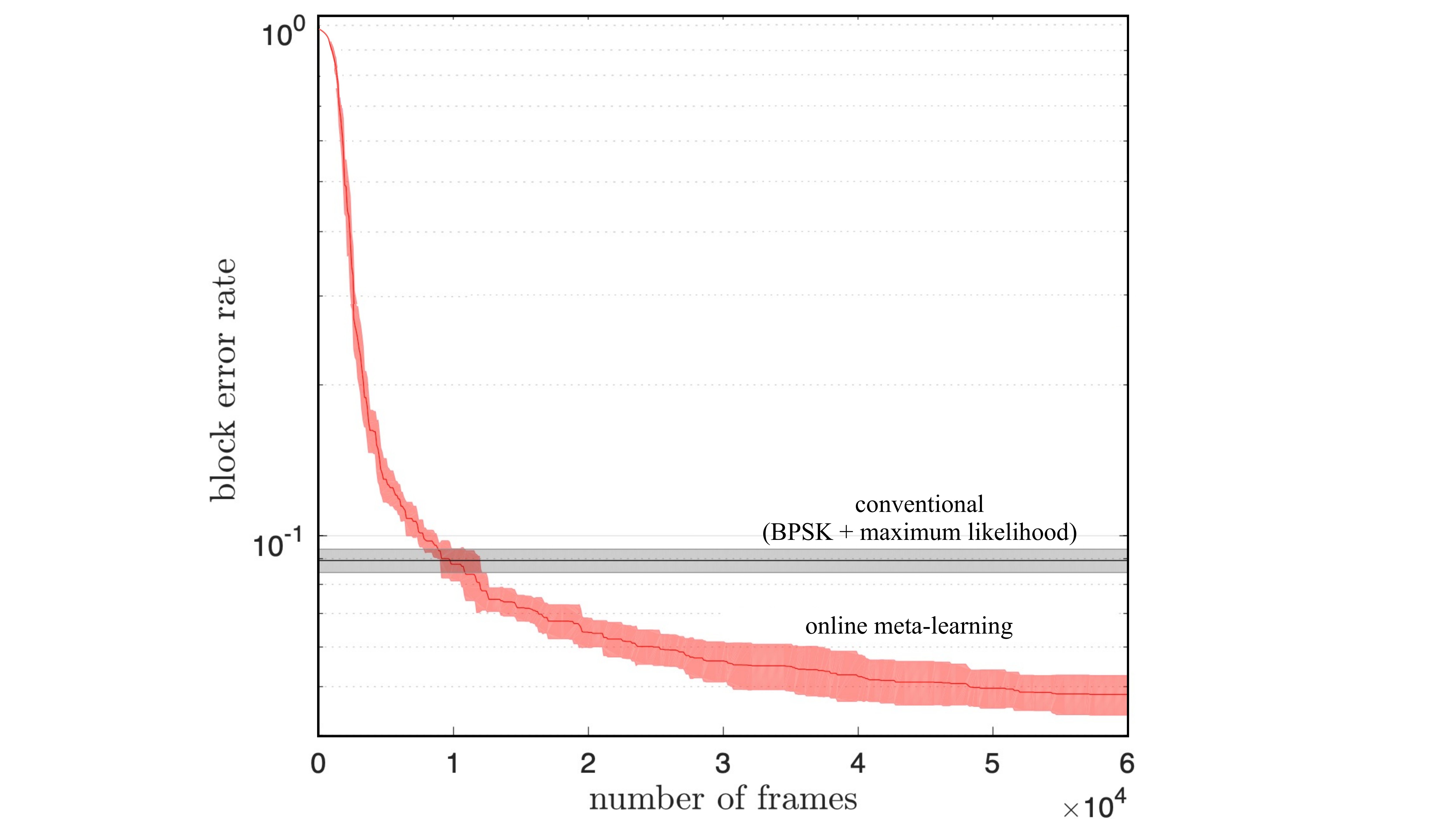}
	\end{center}
	\caption{Meta-learning for encoding and decoding without channel simulator: BLER as a function of the number of frames used during online meta-training phase ($8$ bits, $8$ complex channel uses; Rayleigh block fading channel with $3$ taps, $256$ messages per frame with $8$ pilot messages under a 10 \text{dB} SNR, adapted from \cite{park2020end}). \label{fig0ch7} }
\end{figure}

As an application of continual meta-learning, consider the problem of adapting a demodulator to changing channel conditions. While the setting studied in Section~\ref{sec:ch5_demodul} assumed the offline availability of a meta-training data set collected from a number of devices, a continual meta-learning formulation would operate in a streaming fashion. Accordingly, as data from more devices are collected, the hyperparameter $\theta$ is updated to better prepare the learning algorithm to adapt to new channel conditions. This particular application is studied in \cite{park2020learning}. 

When both encoder and decoder are updated in an online manner, revisiting the previous channel conditions is not feasible, and reference \cite{park2020end} proposed to continually update the meta-learned model at the receiver by applying the meta-gradient obtained from the current channel condition to the current hyperparameter vectors. Referring to \cite{park2020end} for details, Fig.~\ref{fig0ch7} and \ref{fig1ch7} illustrate the performance of the approach over channel conditions defined by an autoregressive Rayleigh fading process with temporal correlation factor $\rho$ \cite{park2020end}. Fig.~\ref{fig0ch7} gauges how many frames are needed for online meta-learning to successfully find a useful hyperparameter vector from the previous (meta-training) frames. In a manner similar to the discussions for offline meta-learning in Fig.~\ref{fig1_2ch5} and Fig.~\ref{fig2_1ch5}, Fig.~\ref{fig0ch7} shows that a sufficiently large number of frames are needed for a successful transfer of knowledge via meta-learning that ensures a performance gain with respect to a conventional per-frame solution. The impact of the channel correlation $\rho$ is analyzed in Fig.~\ref{fig1ch7}, which shows that meta-learning benefits from a smaller $\rho$. In fact, a large $\rho$ may cause meta-overfitting (see Section~\ref{sec:genbounds_metalearning}) due to the similarity of the channels observed during meta-training.

\begin{figure}
	\begin{center}
		\includegraphics[width=0.5\textwidth]{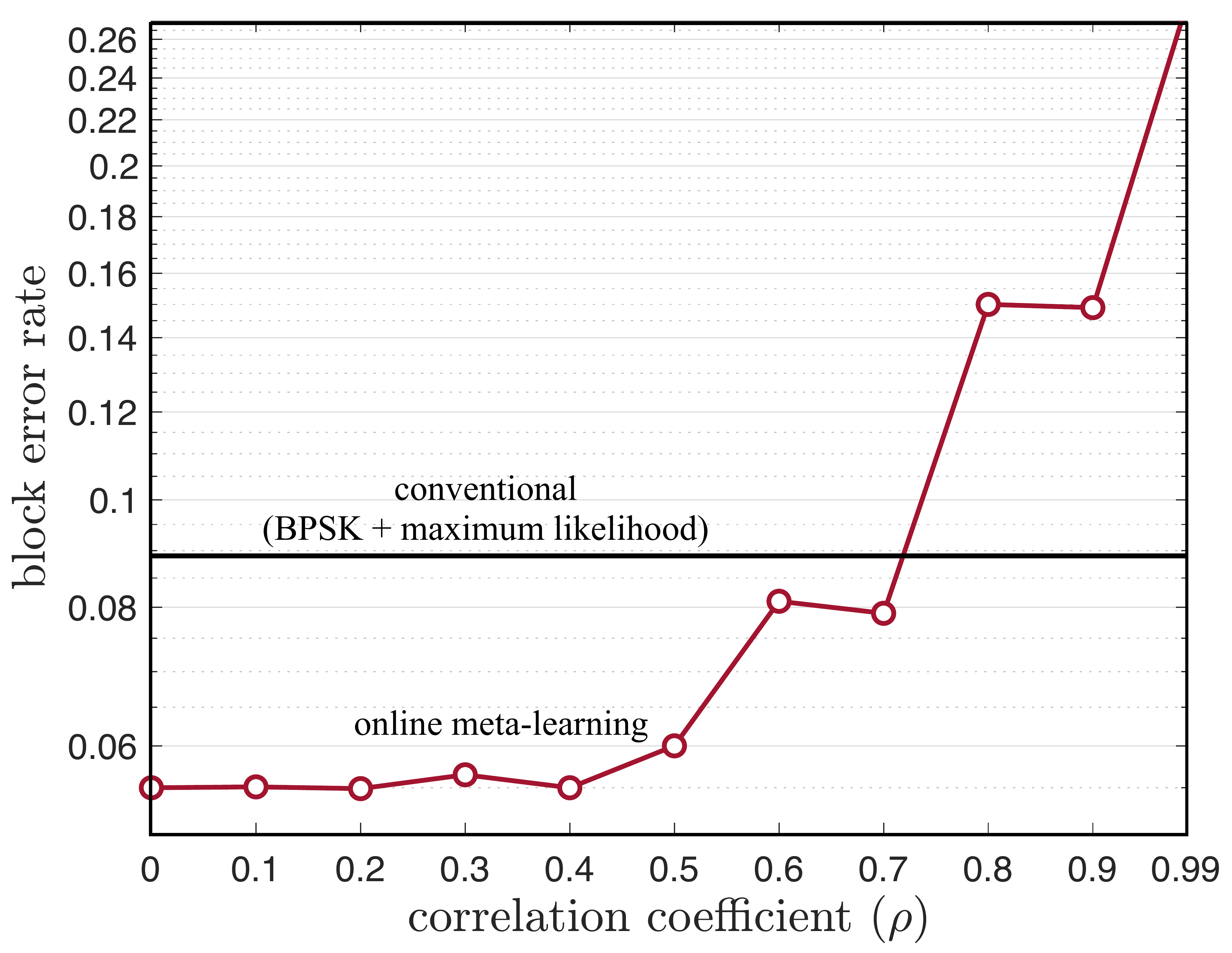}
	\end{center}
	\caption{Meta-learning for encoding and decoding without channel simulator: BLER as a function of the correlation coefficient $\rho$ of the time-varying channel model ($8$ bits, $8$ complex channel uses; Rayleigh block fading channel with $3$ taps, $256$ messages per frame with $8$ pilot messages under a 10 \text{dB} SNR, adapted from \cite{park2020end}). \label{fig1ch7} }
\end{figure}

\subsection{Meta-Learning for Reinforcement Learning}
This monograph has focused on  supervised and unsupervised learning problems. In such settings, the data sets are fixed. In contrast, in \textbf{reinforcement learning} (RL) data is collected through the interaction of the agent with the learning environment defining the given task. Meta-learning can be applied to RL problems with the goal of minimizing the duration of the interactions with new tasks that are required to obtain desirable performance levels \cite{Finn2017_maml, duan2016rl, nagabandi2018learning, rakelly2019efficient, berseth2021comps}. 

Continual meta-learning, as introduced in the previous subsection, can also be applied to RL. A key difference with respect to continual meta-learning for supervised or unsupervised learning is that it may be impossible to interact with previous tasks. This makes it impossible to evaluate the performance of new policies on previous tasks. For such practical scenarios, various techniques have been proposed, including model-based RL \cite{kaiser2019model, nagabandi2018learning, yu2020mopo}, off-policy RL \cite{degris2012off, rakelly2019efficient, yu2020mopo}, and behavior cloning \cite{mendonca2019guided, berseth2021comps}. 


As an example, unlike Section~\ref{sec:ch5:encoding_and_decoding}, which assumed knowledge of the channel model $p_h(y|x)$, RL-based solutions can optimize a transceiver through the direct interactions with the channel, assuming the presence of a feedback link from receiver to transmitter \cite{aoudia2019model}. 

As another application, consider the unmanned aerial base station (UABS) that provides radio coverage in vehicular networks \cite{3gpp2019uxnb}. Depending on a particular traffic pattern of the vehicles, an optimal trajectory of UABS can be found via RL \cite{deng2019joint}. However, such solutions may need retraining when the traffic pattern changes. In order to enable UABS to quickly adapt to new traffic patterns, the work \cite{marini2022continual} developed a meta-learning solution for RL that does not require revisiting the previous environments.

\subsection{Active Meta-Learning}
In the meta-learning formulations discussed so far, the meta-learning tasks are selected by ``nature''. This prevents the meta-learner from actively selecting tasks that are more informative about possible new tasks given what the meta-learner already knows. The active, sequential, selection of tasks is referred to as active meta-learning, and is currently an understudied area of research \cite{kaddour2020probabilistic,Nikoloska2022}. 

As an example, consider again the demodulation with few pilots studied in Section~\ref{sec:ch5_demodul}. Active meta-learning may help the designer reduce the number of required meta-training devices as in \cite{kfir2021towards}.


\subsection{Optimization for Overparameterized Meta-Learning}
When applied to deep learning models, meta-learning typically operates in the overparameterized regime, in which the number of the model parameters exceeds the amount of training data available. For example, ResNets-based MAML models have around 6 million parameters, but are trained on around 2 million meta-training samples  \cite{chen2018closer}.

When the meta-learning problem is overparameterized, the lower-level bilevel problem~\eqref{opt0-low} studied in Section~\ref{sec:ch3:bilevel-opt} may not be \emph{strongly} convex, and thus the lower-level problem has multiple solutions $\{\phi^*(\theta)\}$ given the hyperparameter vector $\theta$. This is problematic because the Hessian of the lower-level problem $\nabla_{\phi\phi}^2 g\big(\theta, \phi\big)$ may be not invertible, and thus the Hessian inverse used in the hyper-gradient \eqref{grad-deter-3} may not exist. Therefore, the alternating stochastic gradient-based ALSET method presented in Section \ref{sec:ch3:bilevel-opt} may not be theoretically justifiable in this case. 

To handle cases in which the lower-level problem has many solutions, two possible methods may be used. One is the optimistic solution that chooses a solution $\phi^*(\theta)$ by minimizing the upper-level objective (e.g., \cite{dempe2007new}), that is 
\begin{subequations}\label{oopt0}
	\begin{align}
		\min_{\theta\in \mathbb{R}^d, \phi^*(\theta)\in  \mathbb{R}^{\hat{d}}}&~~~{\cal L}(\theta):=\mathbb{E}_{\xi}\left[f\left(\theta, \phi^*(\theta);\xi\right)\right]~~~~~~~~~~~{\rm\sf (upper)} \\
		{\rm s. t.}~~~&\phi^*(\theta)\in \argmin_{\phi\in \mathbb{R}^{\hat{d}}}~\mathbb{E}_{\hat{\xi}}[g(\theta, \phi;\hat{\xi})]~~~~~~~~~~{\rm\sf (lower)}; \label{oopt0-low}
	\end{align} 
\end{subequations}
and the other is the pessimistic solution that chooses a solution $\phi^*(\theta)$ by maximizing the upper-level objective (e.g., \cite{dempe2014necessary}), that is   
\begin{subequations}\label{oopt1}
	\begin{align}
		&\min_{\theta\in \mathbb{R}^d}\max_{\phi^*(\theta)\in  \mathbb{R}^{\hat{d}}}~~~{\cal L}(\theta):=\mathbb{E}_{\xi}\left[f\left(\theta, \phi^*(\theta);\xi\right)\right]~~~~~~~~{\rm\sf (upper)} \\
		&~{\rm s. t.}~~~\phi^*(\theta)\in \argmin_{\phi\in \mathbb{R}^{\hat{d}}}~\mathbb{E}_{\hat{\xi}}[g(\theta, \phi;\hat{\xi})]~~~~~~~~~~~~~~~~{\rm\sf (lower)}. \label{oopt1-low}
	\end{align} 
\end{subequations}
The aforementioned bilevel optimization problems are much more challenging than those discussed in Section \ref{sec:ch3:bilevel-opt}, and their non-asymptotic analyses are relatively less explored \cite{vicol2022implicit, sinha2017review, liu2018pessimistic, liu2020methods,liu2020icml,sow2022constrained}.

\section{Theory}
\label{sec:ch7.2:theory}

We now turn to some open theoretical aspects of meta-learning.

\subsection{Benign Overfitting for Overparameterized Meta-Learning}
Statistical learning theory results derived using the standard techniques summarized in Section 4 suggest that overparameterized models tend to overfit \cite{hastie2009elements}.
Translating this insight into the meta-learning setting, one expects that, given the meta-training datasets $\{\mathcal{D}_{k}^{\text{tr}}\}_{k=1}^{K}$, if the model size grows large, the meta-generalization error $\Delta \mathcal{L}(\theta)$ defined in \eqref{eq:metageneralizationerror} also grows. However, empirical evidence reveals that overparameterized meta-learning methods still work well \cite{chen2018closer} -- a phenomenon often called ``benign overfitting.''

While generalization bounds for overparameterized models have been recently studied in the conventional learning setting \cite{Bartlett_benign_linear,Tsigler2020Benign_ridge,wang2021benign,frei2022benign}, their counterparts for meta-learning are  under-explored. The generalization performance under an overparameterized linear regression model has been studied in \cite{huang2022provable,chen2022understanding}, and it would be interesting to extend the analysis in \cite{huang2022provable,chen2022understanding} to nonlinear models by means of random features and neural tangent kernels. It is also interesting to investigate the implicit regularization effect \cite{neyshabur2014search, arora2019implicit} of meta-learning algorithms in overparameterized settings. 


\subsection{Epistemic Uncertainty of Bayesian Meta-Learning Under Model Misspecification}
The information-theoretic analysis of epistemic uncertainty for Bayesian meta-learning presented in Section~\ref{sec:ch4} relies on two crucial assumptions: $(a)$ the model is well-specified, and $(b)$ the exact meta-posterior distribution can be computed. However, neither of these assumptions seldom hold in practice. The true data distribution underlying the standard available data sets is not known in general, and Bayesian algorithms can only obtain approximate posterior distributions. Note that, in contrast, the PAC-Bayes bounds, presented in Section~\ref{sec:PACBayesianbounds_metalearning}, account for these practical considerations.

Characterizing the epistemic uncertainty when either of the above two assumptions is violated is an interesting open problem \cite{hullermeier2022quantifying}. For conventional learning, the recent work \cite{futami2022excess} explores this direction by combining the frequentist PAC-Bayesian generalization analysis with the Bayesian minimum excess risk analysis. Extensions to meta-learning offer an interesting line of future research.

\section{Applications}
\label{sec:ch7.3:apps}

We finally highlight an interesting research direction pertaining the application of meta-learning to communication systems. Also note that there are also many open problems at the intersection of meta-learning and emerging computing technologies as discussed in Section 6.

As discussed in Section~\ref{sec:ch5}, communication systems have been traditionally designed based on carefully designed models. Such models, even when inaccurate, may help define strong inductive biases that can be incorporated within data-driven approaches. For instance, the Viterbi algorithm \cite{viterbi1967error} is known to achieve the minimum BLER on known frequency-selective channels. When the channel is not known, the computation of branch metrics in the Viterbi algorithm can be designed in a data-driven fashion to mitigate the model deficit \cite{shlezinger2020viterbinet}. 

Model-based learning solutions have been reported to outperform both the conventional model-based algorithms and conventional black-box learning approaches \cite{shlezinger2020viterbinet, shlezinger2020deepsic}. Model-based meta-learning can further speed up model-based learning \cite{raviv2022online}. As an example, hypernetwork-based solutions (see Section~\ref{sec:ch2}) have been introduced for Kalman filter design \cite{pratik2021neural}, MIMO detection \cite{goutay2020deep}, and massive MIMO feedback \cite{liu2022learning} to aid model-based algorithms.
